\documentclass[twoside,11pt]{article}
\usepackage{jmlr}

\usepackage{enumerate}
\usepackage[OT1]{fontenc}
\usepackage{amsmath,amssymb}
\usepackage{natbib}
\usepackage[usenames]{color}
\usepackage{multirow}
            
\usepackage{mylatexstyle}
\usepackage{mathrsfs}
\usepackage[protrusion=false,expansion=true]{microtype}
\usepackage{subfigure}
\allowdisplaybreaks 

\newcommand{\JW}[1]{\textsf{\color{cyan} JW: #1}}

\jmlrheading{1}{2000}{1-56}{4/00}{10/00}{00-000}{author list}

\ShortHeadings{Benign Overfitting of Constant-Stepsize SGD for Linear Regression}{Difan Zou, Jingfeng Wu, Vladimir Braverman, Quanquan Gu, Sham M. Kakade}
\firstpageno{1}

\begin{document}
\title{Benign Overfitting of Constant-Stepsize SGD for Linear Regression}

\author{\name Difan Zou\begin{NoHyper}\thanks{Equal Contribution}\end{NoHyper} \email knowzou@cs.ucla.edu \\
       \addr Department of Computer Science\\ University of California, Los Angeles, Los Angeles, CA 90095, USA
       \AND
       \name Jingfeng Wu\begin{NoHyper}\footnotemark[1]\end{NoHyper} \email uuujf@jhu.edu \\
       \addr  Department of Computer Science\\
       Johns Hopkins University, Baltimore, MD 21218, USA
       \AND
       \name Vladimir Braverman \email vova@cs.jhu.edu \\
       \addr  Department of Computer Science\\
       Johns Hopkins University, Baltimore, MD 21218, USA
       \AND
       \name Quanquan Gu \email qgu@cs.ucla.edu \\
       \addr Department of Computer Science\\ University of California, Los Angeles, Los Angeles, CA 90095, USA
       \AND
       \name Sham M. Kakade \email sham@cs.washington.edu \\
       \addr Department of Computer Science\\ University of Washington, Seattle \& Microsoft Research, Seattle, WA 98195, USA
       }
\editor{}

\maketitle

\begin{abstract}
There is an increasing realization that algorithmic inductive biases are central in preventing overfitting; empirically, we often see a \textit{benign overfitting} phenomenon in overparameterized settings
for natural learning algorithms, such as stochastic gradient descent
(SGD), where little to no \emph{explicit}
regularization has been employed. This work considers this issue in arguably
the most basic setting: \textit{constant-stepsize SGD}
(with iterate averaging or tail averaging) for linear regression in the
overparameterized regime.
Our main result provides a sharp excess risk bound, stated in terms of the full eigenspectrum
of the data covariance matrix, that reveals a bias-variance
decomposition characterizing when generalization is possible:
(i) the variance bound is characterized in
terms of an \textit{effective dimension} (specific for SGD)  and (ii) the bias bound
provides a sharp geometric characterization in terms of the location of
the initial iterate (and how it aligns with the data covariance matrix).
More specifically, for SGD with iterate averaging, we demonstrate the sharpness of the established excess risk bound by proving a matching lower bound (up to constant factors). For SGD with tail averaging, we show its advantage over SGD with iterate averaging by proving a better excess risk bound together with a nearly matching lower bound. Moreover, we reflect on a number of notable differences between the algorithmic
regularization afforded by (unregularized) SGD in comparison to ordinary least squares
(minimum-norm interpolation) and ridge regression. Experimental results on synthetic data corroborate our theoretical findings
\footnote{A short version is accepted at the \emph{34th Annual Conference on Learning Theory} (COLT 2021).}.
\end{abstract}


\section{Introduction}

A widely observed and yet still striking phenomenon is that modern
machine learning models (e.g., deep neural networks) trained by stochastic
gradient descent often generalize while also
achieving near-zero training error (i.e., despite being
overparameterized and under-regularized~\footnote{By ``under-regularized'', we mean that the empirical training loss is near to $0$, such as with OLS when $N \gg d$.}. See~\cite{belkin2020two} for further discussion.).
There is reason to believe that characterizing these effects even in conceptually simpler (e.g. linear model) settings will also help our understanding of more complex settings, because many high dimensional effects are also observed even in simple linear models.
For example, this \emph{benign overfitting} effect is also observed
for the ordinary least square (OLS) estimator, where it is observed
that OLS generalizes in the overparameterized regime \citep{bartlett2020benign}.

For OLS in particular, the recent work of
\citet{bartlett2020benign} established non-asymptotic
generalization guarantees of the \emph{minimum-norm interpolator} for
overparameterized linear regression (the minimum-norm solution that
\emph{perfectly fits} the training samples
\citep{zhang2016understanding,bartlett2020benign}). 
More generally, there is a  growing body of work studying generalization in basic linear models in the overparameterized regime \citep{nakkiran2019deep,bartlett2020benign,belkin2020two,hastie2019surprises,tsigler2020benign,muthukumar2020classification,chatterji2020finite,nakkiran2020optimal}. In contrast, for \emph{stochastic gradient descent} (SGD) for least squares regression, the algorithmic aspects of generalization are far less well understood, where we lack a sharp characterization of it and when benign overfitting occurs (in other words, achieving diminishing generalization error). This is the focus of this work.

With regards to SGD in the classical underparameterized regime, the
seminal work of~\cite{polyak1992acceleration} showed that iterate
averaged SGD achieves, in the limit as the sample size goes to
infinity, the statistically optimal rate, even up to problem dependent
constant factors~\footnote{\cite{polyak1992acceleration} provided a
  stronger distributional limit theorem showing that the distribution
  of the averaged iterate (provided by SGD) precisely matches the
  distribution of the empirical risk minimizer.}; this optimality
crucially relies on the dimension being held finite, along with
regularity assumptions that make the problem locally strongly
quadratic. For the case of \emph{finite} dimensional, linear regression, there are a number of more
modern proofs which provide finite, non-asymptotic
rates~\citep{defossez2015averaged,bach2013non,dieuleveut2017harder,jain2017parallelizing,jain2017markov}. 
With regards to the overparameterized regime, there is far less work
~\citep{DieuleveutB15,berthier2020tight} being notable exceptions.
(See Section~\ref{sect:Related}  for further discussion on these related works.)

\paragraph{SGD for linear regression.}
The classical linear regression problem of interest is: 
\begin{equation}\label{eq:least_square}
\min_\wb L(\wb) ,\ \textrm{where} \,\, L(\wb) = \frac{1}{2}\EE_{(\xb,y)\sim\cD}\big[(y - \la\wb,\xb\ra)^2\big],
\end{equation}
where $\xb\in\mathcal{H}$, is the feature vector, where, $\mathcal{H}$ is some (finite $d$-dimensional or countably infinite dimensional) Hilbert space; $y\in\RR$ is the
response; $\cD$ is an unknown distribution over $\xb$ and $y$; and $\wb \in \mathcal{H}$ is the weight vector to be optimized.
We consider the stochastic approximation approach 
using constant
stepsize SGD, with iterate averaging: at each iteration $t$, an i.i.d.
example $(\xb_t, y_t) \sim \cD$ is observed, and the
weight is updated according to SGD as follows:
\begin{equation}\label{eq:sgd}
    \wb_t = \wb_{t-1} + \gamma \rbr{ y_t - \abr{\wb_{t-1}, \xb_t} } \xb_t,\qquad t=1,\dots, N,
\end{equation}
where $\gamma > 0$ is a constant stepsize, $N$ is the number of
samples observed, and the weights are initialized at $\wb_0 \in
\cH$.  The final output will be the average of the iterates: 
\begin{equation*}
    \overline{\wb}_{N} := \frac{1}{N} \sum_{t=0}^{N-1} \wb_t.
\end{equation*}
In the underparameterized setting with finite dimension $d$ ($d\ll
N$), as mentioned earlier (also see Section~\ref{sect:Related}), a rich body of work has  established that
$\overline{\wb}_{N}$ enjoys the optimal risk (up to constant factors)
of $\bigO{{d\sigma^2}/{N}}$, for sufficiently large $N$.
The focus of this work is on the overparameterized regime, where $d\gg
N$ (or possibly countably infinite).

\paragraph{Benign overfitting occurs in SGD for linear regression.}
Perhaps quite surprisingly, the benign overfitting phenomenon, i.e., a predictor that fits training data very well but still generalizes, happens for SGD (with constant stepsize and iterate averaging) even for the simple, overparameterized linear regression. This is empirically verified in Figure \ref{fig0}, where we see  in Figure \ref{fig0} (b) that, SGD overfits the training sample (achieving a training risk much lower than the Bayes risk) but still generalizes on the test sample (the test risk is vanishing). Understanding this phenomenon theoretically is one of the central goals of this work.

\begin{figure}
\vskip -0.1in
     \centering
     \subfigure[$\lambda_i=i^{-1}$]{\includegraphics[width=0.32\textwidth]{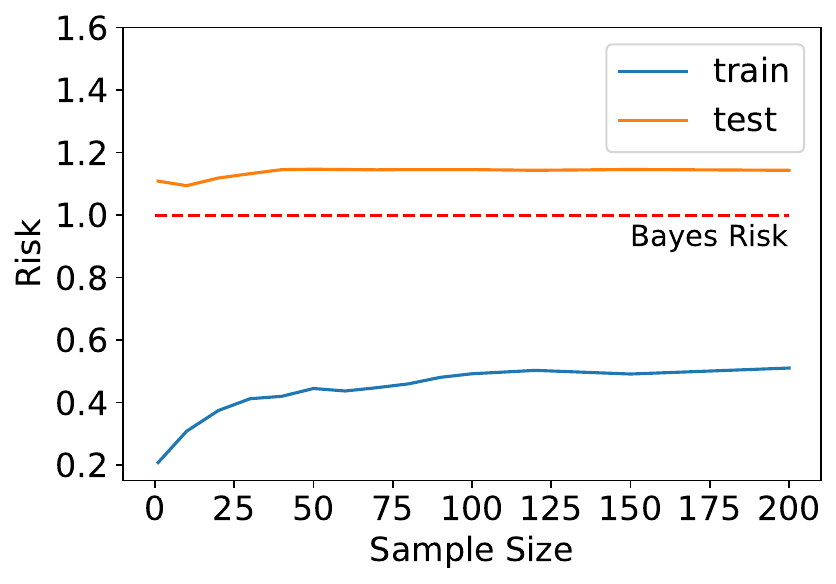}}
      \subfigure[$\lambda_i=i^{-1}\log^{-2}(i)$]{\includegraphics[width=0.32\textwidth]{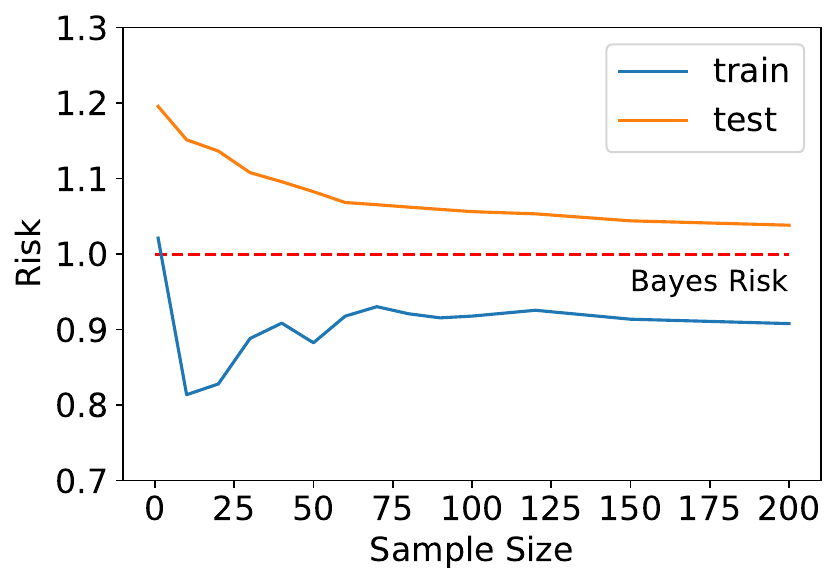}}
      \subfigure[$\lambda_i=i^{-2}$]{\includegraphics[width=0.32\textwidth]{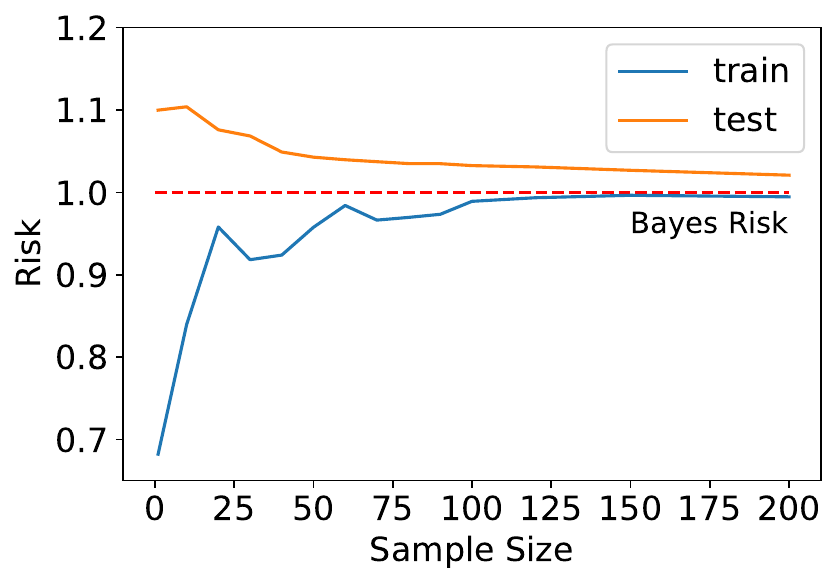}}
      \vskip -0.1in
    \caption{\small 
    Benign overfitting of SGD for linear regression. The plots show the training and test risks achieved by SGD (constant stepsize, iterate averaging) for least square problem instances (the spectrum of $\Hb$, i.e., $\{\lambda_i\}$ is specified under each subfigure). The problem dimension is $d=2000$ and the variance of model noise is $\sigma^2=1$ (hence the Bayes risk is $1$). The plots are averaged over $20$ independent runs. In (a), SGD overfits the training sample (achieving a training risk smaller than the Bayes risk)  but generalizes poorly. In (b), SGD  overfits the training sample and generalizes well, which exhibits the benign overfitting phenomenon. In (c), SGD generalizes on test samples and tends to forget training samples, which indicates a regularization effect of SGD. See Section \ref{sec:experiments} for more details.}
    \label{fig0}
\end{figure}

\paragraph{Our contributions.}
Our main result can be viewed as a counterpart to
the classical analysis of iterate averaged SGD to the
overparameterized regime for linear regression:
we provide a sharp excess risk bound showing how
(unregularized) SGD can generalize even in the infinite-dimensional setting.
Our bound is stated in a general manner, in terms of the full
eigenspectrum of the data covariance matrix along with a functional dependency on the
initial iterate; our lower bound shows our characterization is 
tight. As a corollary, we see how the
benign-overfitting phenomenon can be observed for SGD, provided certain spectrum decay
conditions on the data covariance are met. We also extend our results to SGD with tail-averaging \citep{jain2017markov,jain2017parallelizing}, where we run SGD for $s$ iterations and then take average over the subsequent $N$ iterates as the output. (see Section \ref{sec:tail_averaging} for more details.)

Some additional notable contributions are:
\begin{enumerate}
    \item 
    The sharpness of our bounds permits us to make comparisons to OLS
    (the minimum-norm interpolator) and ridge regression.
Notably, in a contrast to the variance of OLS \citep{bartlett2020benign},
the variance contribution to SGD is well behaved under substantially
weaker assumptions on the spectrum of the data covariance.
    This shows how inductive bias of SGD, in comparison to the
    minimum-norm interpolator, can lead to better generalization
     with no regularization. We also constrast our results to
    ridge regression based on the recent work by~\citet{tsigler2020benign}.     
    \item One notable aspect of our work is a sharp
      characterization of a ``bias process'' in SGD. In particular, consider the
      special case where $y=\wb^\star \cdot \xb$ (with probability
      one), for some $\wb^\star$. Here, SGD still differs from
      gradient descent on $L(\wb)$. Our characterization gives a novel
      characterization of how the variance in this process contributes
      to the final excess risk bound.
    \item From a technical standpoint, our work develops new proof techniques for
iterate averaged SGD.  
Our analysis tools are based on the operator view of averaged
SGD~\citep{DieuleveutB15,jain2017parallelizing,jain2017markov}.  
A core idea in the proof is in connecting the finite sample (infinite dimensional) covariance
matrices of the variance and bias stochastic processes to those of
their corresponding (asymptotic) stationary covariance matrices ---
an idea that was  introduced
in~\cite{jain2017markov} for the finite dimensional, variance analysis.
\end{enumerate}

\paragraph{Notation.}
We use lower case letters to denote scalars, and we use lower and upper
case bold face letters to denote vectors and matrices
respectively. For a vector $\xb\in \mathcal{H}$, 
$\|\xb\|_2$ denotes the norm in the Hilbert space $\cH$, and $\xb[i]$
denotes the $i$-th coordinate of $\xb$. For a matrix $\Mb$, its spectral
norm is denoted by $\|\Mb\|_2$. For a PSD matrix $\Ab$, define
$\|\xb\|_\Ab^2 := \xb^\top\Ab\xb$. 

\section{Main Results}\label{sec:main_theory}

We now provide matching (upto absolute constants) upper and lower excess risk bounds for iterate averaged SGD. We then compare these rates to those of OLS
and ridge regression, where we see striking similarities and notable differences.  

\subsection{Benign Overfitting of SGD}
We first introduce relevant notation and our assumptions.
Our first assumption is mild regularity conditions on the moments of the data distribution.
\begin{assumption}[Regularity conditions]\label{assump:second_moment}
Assume $\EE[\xb \xb^\top]$,
$\EE[ \xb \otimes \xb \otimes \xb \otimes \xb ]$,
and $\EE[y^2]$ exist and are all finite.
Furthermore, denote the second moment of $\xb$ by
\[\Hb := \EE_{\xb\sim\cD}[\xb\xb^\top],\] 
and suppose that $\tr(\Hb)$ is finite. For convenience,
we assume that $\Hb$ is strictly positive definite and that $L(\wb)$ admits a unique global
optimum, which we denote by $\wb^* := \argmin_{\wb} L(\wb)$.
\footnote{This is not necessary. In the case where $\Hb$ has
  eigenvalues which are $0$, we could instead choose $\wb^*$ to
  be the minimum norm vector in the set $\argmin_{\wb} L(\wb)$, and
  our results would hold for this choice of $\wb^\star$. 
For example, see~\cite{scholkopf2002learning} for a rigorous treatment
of working in a reproducing kernel Hilbert space.}
\end{assumption}

Our second assumption is on the behavior of the fourth moment, when viewed as a linear operator on PSD matrices:
\begin{assumption}[Fourth moment condition]\label{assump:bound_fourthmoment}
Assume there exists a positive constant $\alpha > 0$, such that for any PSD
matrix $\Ab$\footnote{This assumption can be relaxed into: for any PSD matrix $\Ab$ \emph{that commutes with $\Hb$}, it holds that $\EE_{\xb\sim\cD}[\xb\xb^\top\Ab\xb\xb^\top]\preceq \alpha \tr(\Hb\Ab)\Hb$. The presented analyzing technique is ready to be modified to cooperate with the relaxed assumption with the observation that the fourth moment operator is linear and self-adjoint. Similar relaxation applies to Assumption \ref{assumption:lowerbound_fourthmoment} as well.}, it holds that 
\[
\EE_{\xb\sim\cD}[\xb\xb^\top\Ab\xb\xb^\top]\preceq \alpha \tr(\Hb\Ab)\Hb.
\]
For Gaussian distributions, it suffices to take
$\alpha=3$. Furthermore, it is worth noting that 
this assumption is implied if the distribution over $\Hb^{-\half}\xb$ has sub-Gaussian
tails (see Lemma~\ref{lemma:sub-gaussian} in the Appendix for a
precise claim). 
Also, it is not difficult to verify that $\alpha\geq 1$.\footnote{This is due to that the square of the second moment is less than the fourth moment.} 
\end{assumption}
Assuming sub-Gaussian tails over $\Hb^{-\half}\xb$ is standard assumption
in regression
analysis (e.g. ~\citealt{HsuKZ14,bartlett2020benign,tsigler2020benign}),
and, as mentioned above, this assumption is substantially weaker. 
The assumption is somewhat stronger than what is often assumed for
iterate averaged SGD in the underparameterized regime
(e.g., ~\citealt{bach2013non,jain2017parallelizing}) (see 
Section~\ref{sect:Related} for further discussion). Additionally, we also remark that Assumption \ref{assump:bound_fourthmoment} can be further relaxed to that we only require $\Ab$ is PSD and commutable with $\Hb$, rather than all PSD matrix $\Ab$ (see Section \ref{sec:discussion} for more details). 

Our next assumption is a noise condition, where it is helpful to interpret $y -
\la\wb^*,\xb\ra$ as the additive noise.  Observe that the first order 
optimality conditions on $\wb^*$ imply 
$\EE_{(\xb,y)\sim\cD}[(y-\la\wb^*,\xb\ra)\xb] = \nabla L(\wb^*) = \boldsymbol{0}
$.
\begin{assumption}[Noise condition]\label{assump:noise}
Suppose that:
\[
\bSigma := \EE \sbr{ (y - \la\wb^*,\xb\ra)^2 \xb\xb^\top }, \quad
\sigma^2 := \norm{\Hb^{-\half} \bSigma \Hb^{-\half}}_2
\]
exist and are finite. Note that $\bSigma$ is the covariance matrix of the gradient noise at $\wb^\star$.
\end{assumption}
This assumption places a rather weak requirement on the
additive noise (due to that it permits model mis-specification) and is often made in the average SGD literature (e.g., 
\citealt{bach2013non,dieuleveut2017harder}). 
Observe that for \emph{well-specified models}, where 
\begin{equation}\label{eq:well}
y = \la\wb^\star, \xb\ra
+\epsilon, \quad \epsilon\sim\cN(0,\sigma^2_{\mathrm{noise}}),
\end{equation}
we have that $\bSigma = \sigma^2_{\mathrm{noise}}\Hb$ and so $\sigma^2 = \sigma^2_{\mathrm{noise}}$.


Before we present our main theorem, a few further definitions are in
order: denote the eigendecomposition of the Hessian as
$\Hb = \sum_{i}\lambda_i\vb_i\vb_i^\top$, where
$\{\lambda_i\}_{i=1}^\infty$ are the eigenvalues of $\Hb$ sorted in
non-increasing order and $\vb_i$'s are the corresponding
eigenvectors. We then denote:
\begin{gather*}
  {\Hb}_{0:k} := \textstyle{\sum_{i=1}^k}\lambda_i\vb_i\vb_i^\top,\quad\mbox{and}\quad
  {\Hb}_{k:\infty} := \textstyle{\sum_{i> k}}\lambda_i\vb_i\vb_i^\top.
\end{gather*}
Similarly we denote $\Ib_{0:k} := \sum_{i=1}^k \vb_i \vb_i^\top$ and $\Ib_{k:\infty} := \sum_{i > k} \vb_i \vb_i^\top$.
By the above definitions, we know
\[
\|\wb\|^2_{\Hb_{0:k}^{-1}}=
\sum_{ i\le
  k}\frac{(\vb_i^\top\wb)^2}{\lambda_i}, \quad
\|\wb\|^2_{{\Hb}_{k:\infty}}=\sum_{i> k}\lambda_i (\vb_i^\top\wb)^2,
\]
where we have slightly abused notation in that ${\Hb}_{0:k}^{-1}$
denotes a pseudo-inverse.

We now present our main theorem:

\begin{theorem}[Benign overfitting of SGD]\label{thm:generalization_error}
Suppose Assumptions \ref{assump:second_moment}-\ref{assump:noise} hold 
and that the stepsize is set so that $\gamma < 1/(\alpha\tr(\Hb))$.
Then the excess risk can be upper bounded as follows,
\begin{align*}
\EE [L(\overline{\wb}_{N})] - L(\wb^*)
&\le  2\cdot \mathrm{EffectiveBias}+2\cdot \mathrm{EffectiveVar},
\end{align*}
where
\begin{align*}
 \mathrm{EffectiveBias} & = \frac{1}{ \gamma^2N^2}\cdot\norm{\wb_0 - \wb^*}^2_{\Hb^{-1}_{0:k^*}} + \norm{\wb_0 - \wb^*}^2_{\Hb_{k^*:\infty}}, \\
 \mathrm{EffectiveVar} & =  \frac{2\alpha\big(\|\wb_0-\wb^*\|_{\Ib_{0:k^*}}^2 + N\gamma\|\wb_0-\wb^*\|_{\Hb_{k^*:\infty}}^2\big)}{N\gamma(1-\gamma \alpha\tr(\Hb))}\cdot\bigg(\frac{k^*}{N} + N\gamma^2 \sum_{i> k^*}\lambda_i^2\bigg)\notag\\
 &\qquad + \frac{ \sigma^2}{1-\gamma \alpha\tr(\Hb)}\cdot \rbr{\frac{k^*}{N} + N \gamma^2 \sum_{i>k^*}\lambda_i^2  }
\end{align*}
with  $k^* = \max \{k: \lambda_k \ge \frac{1}{\gamma N}\}$.
\end{theorem}
The interpretation is as follows:  the ``effective bias'' precisely corresponds to the rate of convergence had we run gradient descent directly on $L(\wb)$ (i.e., where the latter has no variance due to sampling).
The ``effective variance'' error stems from both the additive noise $y - \la\wb^*,\xb\ra$, i.e., the second term of the EffectiveVariance error, along with that even if there was no additive noise (i.e. 
$y - \la\wb^*,\xb\ra=0$ with probability one), i.e., the first term of the EffectiveVariance error, then SGD would still not be equivalent to GD.
The cut-off index $k^*$, which we refer to as the ``effective dimension'', plays a pivotal role in the excess risk bound, which separates the entire space into a $k^*$-dimensional ``head'' subspace where the bias error 
decays  more quickly than that of the bias error in the complement ``tail'' subspace.
To obtain a vanishing bound, the effective dimension $k^*$ must be $\smallO{N}$ and the tail summation $\sum_{i> k^*}\lambda_i^2$ must be  $\smallO{1/N}$. 

In terms of constant factors, the above bound can be improved by a factor of $2$ in the effective bias-variance decomposition (see \eqref{eq:bias_var_decomposition}). We now turn to lower bounds.

\paragraph{A lower bound.}
We first introduce the following assumption that states a lower bound on the fourth moment. 
\begin{assumption}[Fourth moment condition, lower bound]\label{assumption:lowerbound_fourthmoment}
Assume there exists a  constant $\beta\ge0$, such that for any PSD matrix $\Ab$, it holds that
\begin{align*}
\EE_{\xb\sim\cD}[\xb\xb^\top\Ab\xb\xb^\top]-\Hb\Ab\Hb\succeq \beta\tr(\Hb\Ab)\Hb.
\end{align*}
For Gaussian distributions, it suffices to take
$\beta=2$. 
\end{assumption}

The following lower bound shows that when the noise is well-specified our upper bound is not improvable except for absolute constants.
\begin{theorem}[Excess risk lower bound]\label{thm:generalization_error_lowerbound}
Suppose $N\geq 500$.
For any well-specified data distribution $\cD$ (see \eqref{eq:well}) that also satisfies Assumptions \ref{assump:second_moment} and \ref{assumption:lowerbound_fourthmoment}, for any stepsize such that $ \gamma < 1/\lambda_1$,  we have that:
\begin{align*}
    \EE [L(\overline{\wb}_{N})] - L(\wb^*) &\ge\frac{1}{100\gamma^2N^2}\cdot \|\wb_0-\wb^*\|^2_{\Hb_{0:k^*}^{-1}} + \frac{1}{100}\cdot\|\wb_0-\wb^*\|^2_{\Hb_{k^*:\infty}}\notag\\
    &\quad+ \frac{\beta \rbr{ \|\wb_0-\wb^*\|_{\Ib_{0:k^*}}^2 + N\gamma  \|\wb_0-\wb^*\|_{\Hb_{k^*:\infty}}^2 }}{1000  N\gamma }\cdot \rbr{ \frac{k^*}{N} + N \gamma^2  \sum_{i>k^*} \lambda_i^2 }\notag\\
    &\quad+\frac{ \sigma^2_{\mathrm{noise}} }{50} \cdot \rbr{\frac{k^*}{N} + N\gamma^2  \sum_{i>k^*}\lambda_i^2  }
\end{align*}
with $k^* = \max \{k: \lambda_k \ge \frac{1}{N \gamma}\}$.
\end{theorem}

Similar to the upper bound stated in Theorem \ref{thm:generalization_error}, the first two terms represent the EffectiveBias and the last two terms represent the EffectiveVariance, in which the third and last terms are contributed by the model noise and variance in SGD. Our upper bound matches our lower bound up to absolute constants, which indicates the obtained rates are tight, at least for Gaussian data distribution with well-specified noise.

\paragraph{Special cases.}

It is instructive to consider a few special cases of Theorem \ref{thm:generalization_error}. We first show the result for SGD with large stepsizes.

\begin{coro}[Benign overfitting with large stepsizes]\label{thm:large_stepsize}
 Suppose Assumptions \ref{assump:second_moment}-\ref{assump:noise} hold and 
that the stepsize is set to $\gamma =  1/(2\alpha\sum_i\lambda_i)$.
Then 
\begin{align*}
 \mathrm{EffectiveBias} & = \frac{4 \alpha^2 (\sum_i\lambda_i)^2}{ N^2}\cdot\norm{\wb_0 - \wb^*}^2_{\Hb^{-1}_{0:k^*}} + \norm{\wb_0 - \wb^*}^2_{\Hb_{k^*:\infty}} \\
 \mathrm{EffectiveVar} & = \big(2\sigma^2+4\alpha^2 \|\wb_0-\wb^*\|_\Hb^2\big) \cdot \rbr{\frac{k^*}{N} 
+  \frac{N \sum_{i>k^*}\lambda_i^2}{4\alpha^2 (\sum_i\lambda_i)^2} },
\end{align*}
where $k^* = \max \{k: \lambda_k \ge \frac{2\alpha \sum_i\lambda_i}{ N}\}$.
\end{coro}

Note that the bias error decays at different rates in different subspaces. Crudely, in the ``head'' eigenspace (spanned by the eigenvectors corresponding to large eigenvalues) the bias error decays in a faster $\bigO{1/N^2}$ rate (though there is weighting of $\lambda_i$ in the head), while in the remaining ``tail'' eigenspace, the bias error decays  at a slower $\bigO{1/N}$ rate (due to that all the eigenvalues in the tail are less than $\bigO{1/N}$).
The following corollary provides a crude bias bound, showing that bias never decays more slowly than $\bigO{1/N}$.
\begin{coro}[Crude bias-bound]\label{thm:simplied_theory}
Suppose Assumptions \ref{assump:second_moment}-\ref{assump:noise} hold and 
that the stepsize is set to $\gamma =  1/(2\alpha\sum_i\lambda_i)$.
Then 
\begin{align*}
\EE [L(\overline{\wb}_{N})] - L(\wb^*)\le   \frac{8\alpha\|\wb_0-\wb^*\|_2^2\cdot \sum_{i}\lambda_i}{N}
+4 \sigma^2 \cdot \rbr{\frac{k^*}{N} + \frac{N\sum_{i>k^*}\lambda_i^2}{4\alpha^2 (\sum_{i}\lambda_i )^2}   },
\end{align*}
where  $k^* = \max \{k: \lambda_k \ge \frac{2\alpha \sum_i\lambda_i}{ N}\}$.
\end{coro}



Theorems \ref{thm:generalization_error} and \ref{thm:generalization_error_lowerbound} suggests that the excess risk achieved by SGD depends on the spectrum of the covariance matrix. The following corollary gives  examples of data spectrum such that the excess risk is diminishing.
\begin{coro}[Example data distributions]\label{thm:example_spectrum}
Under the same conditions as Theorem~\ref{thm:generalization_error}, suppose $\norm{\wb_0 - \wb^*}_2$ is bounded.
\begin{enumerate}
    \item For $\Hb \in \RR^{d\times d}$, let $s=N^r$ and $d=N^q$ for some positive constants $0<r\le 1$ and $q\ge 1$. If the spectrum of $\Hb$ satisfies 
    \begin{align*}
    \lambda_k = 
    \begin{cases}
        1/s,  &  k\le s,\\
        1/(d-s),     & s+1\le k\le d,
    \end{cases}
    \end{align*}
    then $\EE[L(\overline{\wb}_N)]-L(\wb^*) = \bigO{N^{r-1}+N^{1-q}}$. 
    \item If the spectrum of $\Hb$ satisfies $\lambda_k = k^{-(1+r)}$ for some $r> 0$,  then $\EE[L(\overline{\wb}_N)]-L(\wb^*) = \bigO{N^{-r/(1+r)}}$.
    \item If the spectrum of $\Hb$ satisfies $\lambda_k = k^{-1}\log^{-\beta}(k+1)$ for some $\beta>1$, then $\EE[L(\overline{\wb}_N)]-L(\wb^*) =\bigO{\log^{-\beta}(N)}$.
    \item If the spectrum of $\Hb$ satisfies $\lambda_k = e^{-k}$, then $\EE[L(\overline{\wb}_N)]-L(\wb^*) =  \bigO{\log(N)/N}$.
\end{enumerate}
\end{coro}


\subsection{Comparisons to OLS and Ridge Regression}
\label{sect:Compare}

We now compare these rates to those obtained by OLS or
ridge regression. 

\paragraph{SGD vs. minimum-norm solution of OLS.}
In a somewhat more restrictive setting, \citet{bartlett2020benign} prove that the minimum $\ell_2$ norm interpolator for the linear regression problem on $N$ training examples, denoted by $\hat{\wb}_N$, gives the following excess risk lower bound:
\begin{align*}
\EE[L(\hat{\wb}_N)]-L(\wb^*) &\ge  c\sigma^2\bigg(\frac{k^\star}{N}+\frac{N\sum_{i>k^\star}\lambda_i^2}{(\sum_{i>k^*}\lambda_i)^2}\bigg),
\end{align*}
where $c$ is an absolute constant, $\sigma^2$ is the variance of model noise, and $k^\star = \min\{k\ge 0: \sum_{i>k}\lambda_i/\lambda_{k+1}\ge bN\}$ for some constant $b>0$.
It is clear that in order to achieve benign overfitting, one needs to ensure that $k^\star=\smallO{N}$ and $\sum_{i>k^\star}\lambda_i^2/(\sum_{i>k^\star}\lambda_i)^2=\smallO{1/N}$. The first requirement prefers slow decaying rate of the data spectrum since one hopes to get a large $\sum_{i>k}\lambda_i/\lambda_{k+1}$ for small $k$. On the contrary,  the second requirement suggests that the spectrum should decay fast enough since we need to ensure that the tail summation $\sum_{i>k^\star}\lambda_i^2$ is small. Consequently, as shown in Theorem 6 in \citet{bartlett2020benign}, if the data spectrum decays in a rate $\lambda_k = k^{-\alpha}\log^{-\beta}(k+1)$, the minimum $\ell_2$-norm interpolator can achieve vanishing excess risk only when $\alpha = 1$ and $\beta\ge 1$. In contrast, our results show that SGD can achieve vanishing excess risk for any $\alpha>1$ and $\beta\geq 0$ (as well as the case of $\alpha=1$ and $\beta> 1$, see Corollary \ref{thm:example_spectrum} for details) since a fast decaying spectrum can ensure both small $k^*$ (the effective dimension) and small tail summation  $\sum_{i>k^\star}\lambda_i^2$. 

\paragraph{SGD vs. ridge regression.}
\citet{tsigler2020benign} show that the  ridge regression estimator, denoted by $\hat\wb_N^\lambda$, has the following lower bound on the excess risk:
\begin{align*}
\EE[L(\hat{\wb}_N^\lambda)]-L(\wb^*) &\ge \max_k\Bigg\{c_1\sum_i\frac{\lambda_i\wb^*[i]^2}{(1+\lambda_i/(\lambda_{k+1}\rho_k))^2} +\frac{c_2}{n}\sum_{i} \min\bigg(1,\frac{\lambda_i^2}{\lambda_{k+1}^2(\rho_k+2)^2}\bigg)\Bigg\},
\end{align*}
where $\lambda$ is the regularization parameter, $c_1$ and $c_2$ are absolute constants and $\rho_k=\big(\lambda+\sum_{i>k}\lambda_i\big)/(N\lambda_{k+1})$. \citet{tsigler2020benign} further show that the lower bound nearly matches the following upper bound of the excess risk:
\begin{align*}
\EE[L(\hat{\wb}_N^\lambda)]-L(\wb^*) &\le c_1'\bigg(\|\wb^*\|_{\Hb^{-1}_{0:k^\star}}^2\cdot\bigg(\frac{\lambda+\sum_{i>k}\lambda_i}{N}\bigg)^2+\|\wb^*\|_{\Hb_{k^\star:\infty}}^2\bigg)\notag\\
&\qquad + c_2'\sigma^2\bigg(\frac{k^\star}{N}+\frac{N\sum_{i>k^\star}\lambda_i^2}{(\lambda+\sum_{i>k^*}\lambda_i)^2}\bigg),
\end{align*}
where $c_1'$ and $c_2'$ are absolute constants, and $k^\star = \min\{k\ge 0: (\sum_{i>k}\lambda_i+\lambda)/\lambda_{k+1}\ge bN\}$ for some constant $b>0$. 
Comparing this to Corollary \ref{thm:large_stepsize} suggests that SGD (using a constant stepsize with iterate averaging) may exhibit an implicit regularization effect that performs comparably to ridge regression with a constant regularization parameter (here we assume that $\tr(\Hb)$ is of a constant order).  A more direct problem-dependent comparison (e.g., consider the optimal learning rate for SGD and optimal $\lambda$ for ridge regression) is a fruitful direction of further study, to more accurately gauge the differences between the implicit regularization afforded by SGD and the explicit regularization of ridge regression. 

\section{Further Related Work}
\label{sect:Related}

We first discuss the work on iterate averaging in the finite
dimensional case before turning to the overparameterized regime. In the underparameterized regime, where $d$ is assumed to
be finite, the behavior of
constant stepsize SGD with iterate average or tail average has been
well investigated from the perspective of the \emph{bias-variance
  decomposition}
\citep{defossez2015averaged,dieuleveut2017harder,lakshminarayanan2018linear,jain2017markov,jain2017parallelizing}.
For iterate averaging from the beginning,
\citet{defossez2015averaged,dieuleveut2017harder} show a
$\bigO{{1}/{N^2}}$ convergence rate for the bias error and a
$\bigO{{d}/{N}}$ convergence rate for the variance error, where $N$ is
the number of observed samples and $d$ is the number of parameters.
The bias error rate can be further improved by considering averaging
only the tail iterates \citep{jain2017markov,jain2017parallelizing,pmlr-v75-jain18a},
provided that the minimal eigenvalue of $\Hb$ is bounded away from
$0$. We note that the work in~\citet{jain2017markov,jain2017parallelizing,pmlr-v75-jain18a} also
give the optimal rates with model misspecification.
These results all have dimension factors $d$ and do not apply to
the overparameterized regime, though our results recover the finite
dimensional case (and the results for delayed tail averaging from~\citet{jain2017markov,jain2017parallelizing} can
be applied here for the bias term).  We further develop on the proof techniques
in~\citet{jain2017markov}, where we use properties of asymptotic
stationary distributions for the purposes of finite sample size analysis.

Another notable difference in our work is that
Assumption~\ref{assump:bound_fourthmoment} (which is implied by
sub-Gaussianity, see Lemma~\ref{lemma:sub-gaussian}) is somewhat stronger than what is often assumed for
iterate average SGD analysis, where $\EE[\xb\xb^\top\xb\xb^\top]
\preceq R^2\Hb$, as adopted in
\citet{bach2013non,defossez2015averaged,dieuleveut2017harder,jain2017markov,jain2017parallelizing}. Our assumption implies
an $R^2$ bound with $R^2 = \alpha \tr(\Hb)$. In terms of analysis, we
note that our variance analysis only relies on an $R^2$ condition,
while our bias analysis relies on our stronger sub-Gaussianity-like assumption.

We now discuss related works in the overparameterized regime~\citep{DieuleveutB15,berthier2020tight}. 
Compared with \citep{DieuleveutB15}, our bounds apply to least square instances with \emph{any} data covaraince spectrum (under Assumption \ref{assump:bound_fourthmoment}), while~\citet{DieuleveutB15} only covered least square instances that have specific data covaraince spectrum (see A3 in~\citep{DieuleveutB15}).
In comparison with \citet{berthier2020tight}, their bounds rely on a weaker fourth moment assumption, but rely on a stronger true parameter
assumption in that $\norm{\Hb^{-\alpha} \wb^* }_2$ must be finite, where
$\alpha > 0$ is a constant (see Theorem 1 condition (a) in \citet{berthier2020tight}).
Our fourth moment assumption (Assumption \ref{assump:bound_fourthmoment}) is a natural starting
point for analyzing the over-parameterized regime because it also allows for direct comparisons to OLS and ridge regression, as discussed above.


Concurrent to this work, \citet{chen2020dimension} provide dimension
independent bounds for averaged SGD; their excess risk bounds for linear regression are not as sharp as those provided here.

\section{Proof Outline}\label{sec:general}

We now provide the high level ideas in the proof. A key idea is relating the finite sample (infinite dimensional) covariance
matrices of the variance and bias stochastic processes to those of
their corresponding (asymptotic) stationary covariance matrices ---
an idea developed
in~\cite{jain2017markov} for the finite dimensional, variance analysis.

This section is organized as follows:
Section \ref{sec:proof-preliminary} introduces additional notation and relevant linear operators; Section \ref{sec:proof-decomp} presents a refined bound on a now standard bias-variance decomposition; Section \ref{sec:proof-variance} outlines the  variance error analysis, followed by 
Section \ref{sec:proof-bias} outlining the bias error analysis. Complete proofs of the upper and lower bounds are provided in the Appendix ~\ref{append-sec:proof-upper-bound} and Appendix~\ref{sec:lower_bound}, respectively.

\subsection{Preliminaries}\label{sec:proof-preliminary}
For two matrices $\Ab$ and $\Bb$, their inner product is defined as $\la \Ab, \Bb \ra := \tr\rbr{\Ab^\top \Bb}$.
The following properties will be used frequently: if $\Ab$ is PSD, and $\Bb \succeq \Bb'$, then
\(
\la \Ab, \Bb \ra \ge \la \Ab, \Bb' \ra.
\)
We use $\otimes$ to denote the kronecker/tensor product.
We define the following linear operators:
\begin{gather*}
    \cI = \Ib \otimes \Ib,\quad
    \cM = \EE [ \xb \otimes \xb \otimes \xb \otimes \xb ],\quad
    \tilde{\cM} = \Hb \otimes \Hb, \\
    \cT = \Hb \otimes \Ib + \Ib \otimes \Hb - \gamma\cM, \quad
    \tilde \cT = \Hb \otimes \Ib + \Ib \otimes \Hb - \gamma\Hb\otimes\Hb.
\end{gather*}
We use the notation $\mathcal{O}\circ \Ab$ to denotes the operator
$\mathcal{O}$ acting on a symmetric matrix $\Ab$.
For example, with these definitions, we have that for a symmetric matrix $\Ab$, 
\begin{gather}
\cI \circ \Ab = \Ab, \ \ \ \cM\circ \Ab = \EE [ (\xb^\top \Ab \xb) \xb \xb^\top ], \ \ \  \tilde{\cM} \circ \Ab = \Hb \Ab \Hb, \notag \\     
(\cI - \gamma \cT) \circ \Ab = \EE [ (\Ib - \gamma \xb \xb^\top)\Ab (\Ib - \gamma \xb \xb^\top) ], \ \  (\cI-\gamma\tilde\cT)\circ\Ab = (\Ib-\gamma\Hb)\Ab(\Ib-\gamma\Hb). \label{eq:0005}
\end{gather}
We conclude by summarizing a few technical properties of these operators (see Lemma \ref{lemma:operators2} in Appendix).
\begin{lemma}\label{lemma:operators}
An operator $\cO$ defined on symmetric matrices is called PSD mapping, if $\Ab \succeq 0$ implies $\cO\circ \Ab \succeq 0$.
Then we have
\begin{enumerate}
    \item $\cM$ and $\tilde\cM$ are both PSD mappings.
    \item $\cI-\gamma\cT$ and $\cI-\gamma\tilde\cT$ are both PSD mappings.
    \item $\cM - \tilde\cM$ and $\tilde \cT - \cT$ are both PSD mappings.
\item  If $0 < \gamma \le 1/\lambda_1$, then $\tilde{\cT}^{-1}$ exists, and is a PSD mapping.
\item  If $0 < \gamma \le 1/(\alpha\tr(\Hb))$, then $\cT^{-1}\circ \Ab$ exists for PSD matrix $\Ab$, and $\cT^{-1}$ is a PSD mapping. 
\end{enumerate}
\end{lemma}

\subsection{The Bias-Variance Decomposition}\label{sec:proof-decomp}

It is helpful to consider the bias-variance decomposition for averaged SGD, which has been extensively studied before in the underparameterized regime ($N\gg d$) \citep{DieuleveutB15,jain2017parallelizing,jain2017markov}.
For convenience, we define the centered SGD iterate as $\betab_t := \wb_t - \wb^*$. Similarly we define $\bar{\betab}_{N} := \frac{1}{N}\sum_{t=0}^{N-1} \betab_t$.


\noindent (1) If the sampled data contains no label noise, i.e., $y_{t} = \la\wb^*, \xb_t\ra$, then the obtained SGD iterates $\{\betab^{\bias}_t\}$ reveal the \emph{bias error}, 
\begin{equation}\label{eq:bias_iterates}
    \betab^{\bias}_t = \rbr{\Ib-\gamma\xb_t\xb_t^\top} \betab^{\bias}_{t-1}, \qquad \betab^{\bias}_0 = \betab_0.
\end{equation}
\noindent (2) If the iterates are initialized from the optimal $\wb^*$, i.e., $\wb_0 = \wb^*$, then the obtained SGD iterates $\{\betab^{\var}_t\}$ reveal the \emph{variance error},
\begin{equation}\label{eq:variance_iterates}
    \betab^{\var}_t = \rbr{\Ib-\gamma\xb_t\xb_t^\top} \betab^{\var}_{t-1} + \gamma \xi_t \xb_t, \qquad \betab^{\var}_0 = \boldsymbol{0},
\end{equation}
where $\xi_t := y_t - \la\wb^*, \xb_t\ra$ is the inherent noise.
Note  the ``bias iterates'' can be viewed as a stochastic process of SGD on a consistent linear system;
similarly, the ``variance iterates'' should be treated as a stochastic process of SGD initialized from the optimum.

Using the defined operators, the update rule of the iterates \eqref{eq:bias_iterates} imply the following recursive form of $\Bb_t : = \EE [\betab_t^\bias \otimes \betab_t^\bias ]$:
\begin{equation}\label{eq:update_Bt}
    \Bb_t = (\cI - \gamma\cT)\circ \Bb_{t-1} \qquad  \text{and}  \qquad \Bb_0 = \betab_0\otimes \betab_0,
\end{equation}
and the update rule \eqref{eq:variance_iterates} imply the following recursive form of $\Cb_t := \EE [\betab_t^\var \otimes \betab_t^\var]$:
\begin{equation}\label{eq:update_Ct}
\Cb_t = (\cI-\gamma\cT) \circ \Cb_{t-1} + \gamma^2\bSigma,\qquad \Cb_0 = \boldsymbol{0}.
\end{equation}
We define the averaged version of $\betab^{\bias}_t$ and $\betab^\var_t$ in the same way as $\overline{\wb}_N$, i.e., $\bar{\betab}_{N}^{\bias} := \frac{1}{N}\sum_{t=0}^{N-1} \betab_t^{\bias}$ and $\bar{\betab}_{N}^{\var} := \frac{1}{N}\sum_{t=0}^{N-1} \betab_t^{\var}$. 
With a little abuse of probability space, from \eqref{eq:sgd}, \eqref{eq:bias_iterates} and \eqref{eq:variance_iterates} we have that
\[
\betab_t = \betab_t^\bias + \betab_t^\var,
\]
then an application of Cauchy–Schwarz inequality leads to the following \emph{bias-variance decomposition} on the excess risk (see \citet{jain2017parallelizing}, also Lemma \ref{lemma:bias_var_decomposition} in the appendix):
\begin{gather}
\EE [L(\overline{\wb}_{N})] - L(\wb^*) = \frac{1}{2}\la\Hb,\EE[\bar\betab_{N}\otimes \bar\betab_{N}]\ra\le \rbr{ \sqrt{\bias} + \sqrt{\var} }^2, \label{eq:bias_var_decomposition} \\
\text{where }\ \
\bias := \frac{1}{2} \langle \Hb, \EE[{\bar\betab}^{\bias}_{N} \otimes {\bar\betab}^{\bias}_{N}] \rangle, \ \  
\var := \frac{1}{2} \langle \Hb, \EE[{\bar\betab}^{\var}_{N} \otimes {\bar\betab}^{\var}_{N}] \rangle. \notag
\end{gather}
In the above bound, the two terms are usually referred to as the \emph{bias error} and the \emph{variance error} respectively. 
Furthermore, expanding the kronecker product between the two averaged iterates, and doubling the squared terms, we have the following upper bounds on the bias error and the variance error (see Lemma \ref{lemma:bias_var_decomposition_bound} in the appendix for the proof):
\begin{gather}
    \bias :=  \half \langle \Hb, \EE[{\bar\betab}^{\bias}_{N} \otimes {\bar\betab}^{\bias}_{N}] \rangle 
    \le \frac{1}{N^2}\sum_{t=0}^{N-1}\sum_{k=t}^{N-1}\big\la (\Ib-\gamma\Hb)^{k-t}\Hb,\Bb_t\big\ra, \label{eq:formula_bias} \\
    \var :=  \half \langle \Hb, \EE[{\bar\betab}^{\var}_{N} \otimes {\bar\betab}^{\var}_{N}] 
    \le \frac{1}{N^2}\sum_{t=0}^{N-1}\sum_{k=t}^{N-1}\big\la (\Ib-\gamma\Hb)^{k-t}\Hb,\Cb_t\big\ra. \label{eq:formula_var}
\end{gather}
Note that in the above bounds, we keep both summations in finite steps, and this makes our analysis sharp as $N \ll d$. In comparison, \citet{jain2017markov,jain2017parallelizing} take the inner summation to infinity, which yields looser upper bounds for further analysis in the overparameterized setting.
Next we bound the two error terms \eqref{eq:formula_bias} and \eqref{eq:formula_var} separately.
\subsection{Bounding the Variance Error}\label{sec:proof-variance}
We would like to point out that in the analysis of the variance error \eqref{eq:formula_var}, Assumption \ref{assump:bound_fourthmoment} can be replaced by a weaker assumption: $\EE[\xb\xb^\top\xb\xb^\top]\preceq R^2\Hb$, where $R$ is a positive constant \citep{jain2017parallelizing,jain2017markov,dieuleveut2017harder}. A proof under the weaker assumption can be found in Appendix~\ref{append-sec:proof-variance}.
Here, for consistency, we sketch the proof under Assumption \ref{assump:bound_fourthmoment}.

To upper bound \eqref{eq:formula_var}, noticing that $(\Ib-\gamma\Hb)^{k-t}\Hb$ is PSD, it suffices to upper bound $\Cb_t$ in PSD sense. 
In particular, by Lemma 5 in \citet{jain2017markov} (restated in Lemma \ref{lemma:monotonicity_phit} in the appendix), the sequence $\{\Cb_t\}_{t=0,\dots}$ has the following property,
\begin{equation}\label{eq:Ct_crude_bound}
0 = \Cb_0 \preceq \Cb_1\preceq \cdots\preceq 
\Cb_\infty\preceq \frac{\gamma \sigma^2}{1-\gamma \alpha\tr(\Hb)}\Ib.
\end{equation}
This gives a uniform but crude upper bound on $\Cb_t$ for all $t\ge 0$. 
However, a direct application of this crude bound to \eqref{eq:formula_var} cannot give a sharp rate in the overparameterized setting.
Instead, we seek to refine the bound of $\Cb_t$ based on its update rule in  \eqref{eq:update_Ct} (see the proof of Lemma~\ref{lemma:upper_bound_phit} for details):
\begin{align}
\Cb_t &= (\cI- \gamma\cT )\circ \Cb_{t-1} + \gamma^2 \bSigma \notag\\
& = (\cI - \gamma \tilde\cT) \circ \Cb_{t-1} + \gamma^2(\cM - \tilde\cM)\circ \Cb_{t-1}+\gamma^2\bSigma\notag\\
&\preceq (\cI - \gamma \tilde\cT) \circ \Cb_{t-1} + \gamma^2\cM\circ \Cb_{t-1}+\gamma^2\bSigma \qquad (\text{since $\tilde\cM$ is a PSD mapping})  \notag\\
&\preceq (\cI - \gamma \tilde\cT) \circ \Cb_{t-1} + \frac{\gamma^3\sigma^2}{1-\gamma  \alpha\tr(\Hb)}\cM\circ \Ib + \gamma^2\bSigma, \quad (\text{by \eqref{eq:Ct_crude_bound} and $\cM$ is a PSD mapping}) \notag \\
& \preceq (\cI - \gamma \tilde\cT) \circ \Cb_{t-1} + \frac{\gamma^3\sigma^2\alpha\tr(\Hb) }{1-\gamma  \alpha\tr(\Hb)}\Hb + {\gamma^2\sigma^2}\Hb,  \qquad (\text{by Assumptions \ref{assump:bound_fourthmoment} and \ref{assump:noise}}) \notag \\
&= (\cI - \gamma \tilde\cT) \circ \Cb_{t-1} + \frac{\gamma^2\sigma^2 }{1-\gamma  \alpha\tr(\Hb)}\Hb. \notag
\end{align}
Solving the above recursion, we obtain the following refined upper bound for $\Cb_t$:
\begin{align}
\Cb_t
&\preceq \frac{\gamma^2\sigma^2 }{1-\gamma  \alpha\tr(\Hb)} \sum_{k=0}^{t-1} (\cI-\gamma\tilde\cT)^k\circ\Hb \notag\\
&= \frac{\gamma^2\sigma^2 }{1-\gamma  \alpha\tr(\Hb)} \sum_{k=0}^{t-1} (\Ib-\gamma\Hb)^{k}\Hb (\Ib-\gamma\Hb)^{k} \qquad (\text{by the property of $\cI-\gamma\tilde\cT$ in \eqref{eq:0005}})\notag \\
&\preceq \frac{\gamma^2\sigma^2 }{1-\gamma  \alpha\tr(\Hb)} \sum_{k=0}^{t-1} (\Ib-\gamma\Hb)^{k} \Hb 
= \frac{\gamma\sigma^2}{1-\gamma\alpha\tr(\Hb)}\cdot\big(\Ib-(\Ib-\gamma\Hb)^t\big). \label{eq:upperbuond_ct_sketch}
\end{align}
Now we can plug the above refined upper bound \eqref{eq:upperbuond_ct_sketch} into \eqref{eq:formula_var}, and obtain
\begin{align}
   \var 
    &\le \frac{ \sigma^2}{ N^2 (1-\gamma\alpha\tr(\Hb))} \sum_{t=0}^{N-1} \big\la \Ib - (\Ib - \gamma\Hb)^{N-t} , \Ib - (\Ib - \gamma \Hb)^{t}  \big \ra\notag \\
    &= \frac{ \sigma^2}{N^2 (1-\gamma\alpha\tr(\Hb))} \sum_{t=0}^{N-1} \sum_{i}\rbr{ 1 - (1 - \gamma\lambda_i)^{N-t}}\rbr{ 1 - (1 - \gamma\lambda_i)^{t}} \notag\\
    &\le \frac{ \sigma^2}{N^2 (1-\gamma\alpha\tr(\Hb))} \cdot N \cdot \sum_{i}\rbr{ 1 - (1 - \gamma\lambda_i)^{N}}^2. \label{eq:finalbound_variance_sketch}
\end{align}
The remaining effort is to precisely control the summations in \eqref{eq:finalbound_variance_sketch} according to the scale of the eigenvalues: for large eigenvalues $\lambda_i\ge \frac{1}{N\gamma}$, which appear at most $k^*$ times, we use $1-(1-\gamma\lambda_i)^N\le 1$; and for the remaining small eigenvalues $\lambda_i< \frac{1}{N\gamma}$, we use $1-(1-\gamma\lambda_i)^N \le \bigO{N\gamma\lambda_i}$.
Plugging these into \eqref{eq:finalbound_variance_sketch} gives us the final full spectrum upper bound on the variance error (see the proof of Lemma \ref{lemma:upperbound_var} for more details). This bound contributes to part of $\mathrm{EffectiveVar}$ in Theorem \ref{thm:generalization_error}.

\subsection{Bounding the Bias Error}\label{sec:proof-bias}
Next we discuss how to bound the bias error \eqref{eq:formula_bias}.
A natural idea is to follow the same way in analyzing the variance error, and derive a similar bound on $\Bb_t$.
Yet a fundamental difference between the variance sequence \eqref{eq:update_Ct} and the bias sequence \eqref{eq:update_Bt} is that: $\Cb_t$ is increasing, while $\Bb_t$ is ``contracting'', 
hence applying the same procedure in the variance error analysis cannot lead to a tight bound on $\Bb_t$.
Instead, observing that  $\Sbb_t := \sum_{k=0}^{t-1} \Bb_k$, the summation of a contracting sequence, is increasing in the PSD sense. Particularly, we can rewrite $\Sbb_t$ in the following recursive form 
\begin{align}\label{eq:update_St}
\Sbb_t &= (\cI-\gamma\cT)\circ\Sbb_{t-1}+\Bb_0,
\end{align}
which resembles that of $\Cb_t$ in \eqref{eq:update_Ct}.
This motivates us to: 
(i) express the obtained bias error bound \eqref{eq:formula_bias} by $\Sbb_t$, 
and (ii) derive a tight upper bound on $\Sbb_t$ using similar analysis for the variance error.

For (i), by some linear algebra manipulation (see the derivation of  \eqref{eq:bias-upperbound-0}), we can bound \eqref{eq:formula_bias} 
 as follows:
\begin{align}\label{eq:upperbuond_bias_sketch}
\bias 
\le \frac{1}{\gamma N^2}\big\la\Ib-(\Ib-\gamma\Hb)^{N},\sum_{t=0}^{N-1}\Bb_t\big\ra
= \frac{1}{\gamma N^2}\big\la\Ib-(\Ib-\gamma\Hb)^{N},\Sbb_{N}\big\ra.
\end{align}

For (ii), 
we first show that $\{\Sbb_t\}_{t=1,\dots,N}$ is increasing and has a crude upper bound (see Lemmas \ref{lemma:properties_St} and \ref{lemma:T_inv}):
\begin{align}\label{eq:St_crude_bound}
\Bb_0 = \Sbb_1\preceq \Sbb_2 \preceq\cdots\preceq \Sbb_{N}, \quad \text{and}\quad \cM\circ\Sbb_{N}\preceq \frac{\alpha \cdot \tr\Big( \big(\Ib - (\Ib-\gamma\Hb)^{2N}\big) \Bb_0 \Big)}{\gamma(1-\gamma\alpha\tr(\Hb))}\cdot\Hb.
\end{align}
Then similar to our previous procedure in bounding $\Cb_t$, we can tighten the upper bound on $\Sbb_t$ by its recursive form
\eqref{eq:update_St} and the crude bound ($\cM\circ\Sbb_{N-1}$ in \eqref{eq:St_crude_bound}), and obtain the following refined bound (see Lemma \ref{lemma:upperbound_St}) for $\Sbb_N$:
\begin{align}\label{eq:upperbuond_st_sketch}
\Sbb_N \preceq
 \sum_{k=0}^{N-1}(\Ib-\gamma\Hb)^{k}\Bb_0(\Ib-\gamma\Hb)^k+\frac{\gamma\alpha\cdot \tr\Big(\big(\Ib - (\Ib-\gamma\Hb)^{2N}\big)\Bb_0\Big)}{1-\gamma \alpha \tr(\Hb)}\sum_{k=0}^{N-1}(\Ib-\gamma\Hb)^{2k}\Hb.
\end{align}
The remaining proof will be similar to what we have done for the variance error bound:
substituting \eqref{eq:upperbuond_st_sketch} into \eqref{eq:upperbuond_bias_sketch} gives an upper bound on the bias error with respect to the summations over functions of eigenvalues. Then by carefully controlling each summation according to the scale of the corresponding eigenvalues, we will obtain a tight full spectrum upper bound on the bias error (see the proof of Lemma \ref{lemma:bound_bias_final} for more details).

As a final remark, noticing that different from the upper bound of $\Cb_t$ in \eqref{eq:upperbuond_ct_sketch}, the upper bound for $\Sbb_t$ in \eqref{eq:upperbuond_st_sketch} consists of two terms. 
The first term will contribute to the $\mathrm{EffectiveBias}$ term in Theorem \ref{thm:generalization_error}, while the second term will be merged to the bound of the variance error and contribute to the $\mathrm{EffectiveVar}$ term in Theorem \ref{thm:generalization_error}.

\section{The Effect of Tail-Averaging}\label{sec:tail_averaging}
We further consider benign overfitting of SGD when \emph{tail-averaging}  \citep{jain2017parallelizing} is applied, i.e.,
\begin{align*}
\overline \wb_{s:s+N} = \frac{1}{N}\sum_{t=s}^{s+N-1}\wb_t.
\end{align*}
We present the following theorem as a counterpart of Theorem \ref{thm:generalization_error}. The proof is deferred to Appendix \ref{append-sec:tail-average}.

\begin{theorem}[Benign overfitting of SGD with tail-averaging]\label{thm:generalization_error_tail}
Consider SGD with tail-averaging.
Suppose Assumptions \ref{assump:second_moment}-\ref{assump:noise} hold 
and that the stepsize is set so that $\gamma < 1/(\alpha\tr(\Hb))$.
Then the excess risk can be upper bounded as follows,
\begin{align*}
\EE [L(\overline{\wb}_{s:s+N})] - L(\wb^*)
&\le  2\cdot \mathrm{EffectiveBias}+2\cdot \mathrm{EffectiveVar},
\end{align*}
where
\begin{align*}
 \mathrm{EffectiveBias} & = \frac{1}{\gamma^2N^2}\cdot\big\|(\Ib-\gamma\Hb)^s(\wb_0-\wb^*)\big\|_{\Hb_{0:k^*}^{-1}}^2 + \big\|(\Ib-\gamma\Hb)^s(\wb_0-\wb^*)\big\|_{\Hb_{k^*:\infty}}^2 \\
 \mathrm{EffectiveVar} &= \frac{4\alpha \big(\|\wb_0-\wb^*\|^2_{\Ib_{0:k^\dagger}}+(s+N)\gamma\|\wb_0-\wb^*\|_{\Hb_{k^\dagger:\infty}}^2\big)}{N\gamma(1-\gamma \alpha\tr(\Hb))}\cdot\bigg(\frac{k^*}{N} + N\gamma^2 \sum_{i> k^*}\lambda_i^2\bigg)\notag\\
 &\qquad + \frac{ \sigma^2}{ 1-\gamma \alpha \tr (\Hb)} \cdot\bigg(\frac{k^*}{N} + \gamma\cdot \sum_{k^*< i\le k^\dagger}\lambda_i + (s+N)\gamma^2\cdot\sum_{i>k^\dagger}\lambda_i^2\bigg),
\end{align*}
where  $k^* = \max \{k: \lambda_k \ge \frac{1}{\gamma N}\}$ and $k^\dagger = \max\{k:\lambda_k\ge \frac{1}{\gamma(s+N)}\}$.
\end{theorem}

Theorem \ref{thm:generalization_error_tail} shows that tail-averaging has improvements over iterate-averaging.
This agrees with the results shown in \citet{jain2017parallelizing}: in the underparameterized regime ($N\gg d$) and for the strongly convex case ($\lambda_d > 0$), one can obtain substantially improved convergence rates on the bias term.

We also provide a lower bound on the excess risk for SGD with tail-averaging as a counterpart of Theorem \ref{thm:generalization_error_lowerbound}, which shows that our upper bound is nearly tight. The proof is again deferred to Appendix \ref{append-sec:tail-average}.

\begin{theorem}[Excess risk lower bound, tail-averaging]\label{thm:generalization_error_tail_lowerbound}
Consider SGD with tail-averaging.
Suppose $N\ge500$.
For any well-specified data distribution $\cD$ (see \eqref{eq:well}) that also satisfies Assumptions \ref{assump:second_moment}, \ref{assump:bound_fourthmoment} and \ref{assumption:lowerbound_fourthmoment}, for any stepsize such that $ \gamma < 1/\lambda_1$, we have that:
\begin{align*}
    \EE [L(\overline{\wb}_{s:s+N})] - L(\wb^*) &\ge \frac{1}{100\gamma^2N^2}\cdot \|(\Ib-\gamma\Hb)^s(\wb_0-\wb^*)\|^2_{\Hb_{0:k^*}^{-1}} + \frac{\|(\Ib-\gamma\Hb)^s(\wb_0-\wb^*)\|^2_{\Hb_{k^*:\infty}}}{100}\notag\\
    &\qquad + \frac{\beta \|\wb_0-\wb^*\|_{\Hb_{k^\dagger:\infty}}^2}{10^4 } \rbr{ \frac{k^*}{N} + N \gamma^2  \sum_{i>k^*}\lambda_i^2 } \\
    &\qquad + \frac{ \sigma_{\mathrm{noise}}^2 }{600} \rbr{\frac{k^*}{N} + \gamma \sum_{k^* < i \le k^{\dagger}}\lambda_i  + (s+N)\gamma^2  \sum_{i>k^{\dagger}}\lambda_i^2  },
    \end{align*}
    where $k^* = \max \{k: \lambda_k \ge \frac{1}{ N \gamma}\}$ and $k^{\dagger} = \max \{k: \lambda_k \ge \frac{1}{(s+N)\gamma }\}$.
\end{theorem}

Comparing our upper and lower bounds, they are matching (upto absolute constants) for most of the terms, except for the first effective variance term, where a $\norm{\wb_0 - \wb^*}^2_{\Ib_{0:k^\dagger}}$ is lost (suppose that $s=\Theta(N)$).
Our conjecture is that the upper bound is improvable in this regard.
Obtaining matching upper and lower bounds for SGD with tail-averaging is left as a direction for future work.


\section{Experiments}\label{sec:experiments}
In this section, we seek to empirically observe the benign overfitting phenomenon for SGD in Gaussian least square problems and verify our theorems on the generalization performance of SGD. 

We first consider three over-parameterized linear regression problem instances with $d=2000$ and the spectrum of $\Hb$ as $\lambda_i=i^{-1}$, $\lambda_i=i^{-1}\log(i)^{-2}$, and $\lambda_i=i^{-2}$, respectively. Besides, the ground truth is fixed to be $\wb^*[i]=i^{-1}$. The training and test risks for these three problems are displayed in Figure \ref{fig0}. We observe that when $\lambda_i=i^{-1}$, the SGD algorithm overfits the training data and fails to generalize; when  $\lambda_i=i^{-1}\log(i)^{-2}$, SGD overfits the training data (achieving a training risk much smaller than the Bayes risk) while generalizes well (achieving a vanishing test risk), which exhibits a benign overfitting phenomenon of SGD; when $\lambda_i=i^{-2}$, SGD gives vanishing test risk and tends to un-fit the training data, which indicates a regularization effect of SGD.
In sum, the experiments suggest that benign overfitting of SGD can happen when the spectrum of $\Hb$ decays neither fast nor slow. This is consistent with the benign overfitting of least square (minimum-norm solution) \citep{bartlett2020benign}, where for $\Hb$ with spectrum in form of $\lambda_i=i^{-\alpha}\log^{-\beta}(i)$, the benign overfitting phenomenon can only happen for $\alpha=1$ and $\beta >1$.

Then we consider $6$ problem instances, which are the combinations of two covariance matrices $\Hb$ with eigenvalues $\lambda_i=i^{-1}$ and $\lambda_i=i^{-2}$ respectively; 
and three true model parameter $\wb^*$ with components $\wb^*[i]=1$, $\wb^*[i]=i^{-1}$, and $\wb^*[i] = i^{-10}$, respectively. 
We investigate four algorithms: (1) SGD with iterate averaging (from the beginning), (2) SGD with tail averaging ($\bar \wb_{N/2:N-1}$), (3) ordinary least square (minimum-norm interpolator), and (4) ridge regression (regularized least square), where the hyperparameters (i.e., $\gamma$ for SGD and $\lambda$ for ridge regression) are fine-tuned to achieve the best performance. 
Results are shown in Figure \ref{fig1}. 
We see that (1) SGD, with either iterate averaging or tail averaging, is comparable to ridge regression, and significantly outperforms ordinary least square in some problem instances, and (2) SGD with tail averaging performs better than SGD with iterate averaging. These observations are consistent with our theoretical findings and demonstrate the benefit of the implicit regularization from SGD. 

\begin{figure}[!t]
\vskip -0.1in
     \centering
     \subfigure[$\lambda_i=i^{-1}, \wb^{*}{[i]}=1$]{\includegraphics[width=0.32\textwidth]{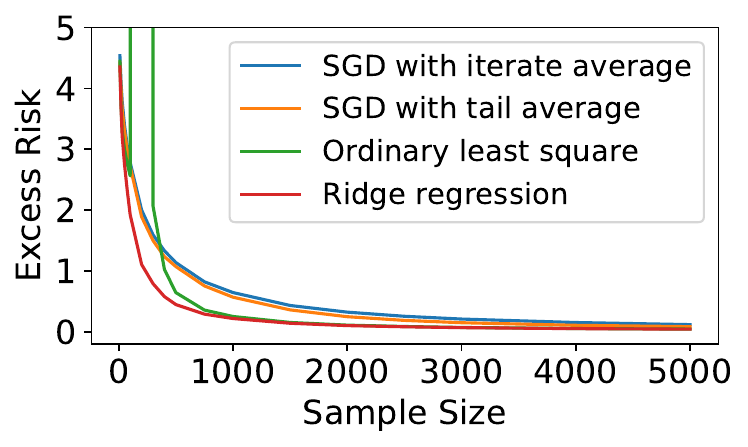}}
      \subfigure[$\lambda_i=i^{-1}, \wb^{*}{[i]}=i^{-1}$]{\includegraphics[width=0.32\textwidth]{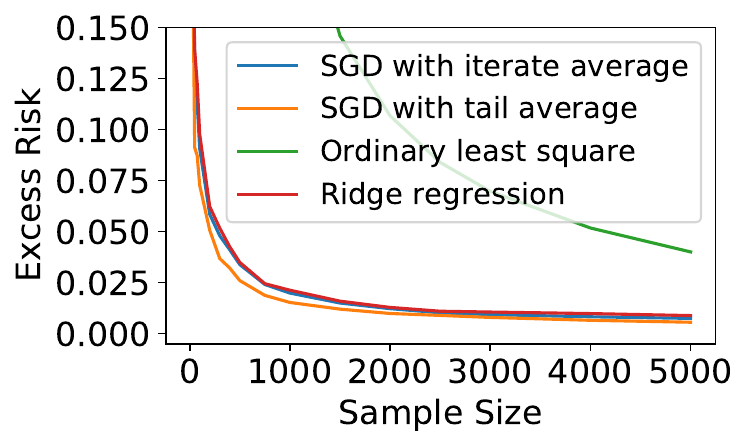}}
      \subfigure[$\lambda_i=i^{-1}, \wb^{*}{[i]}=i^{-10}$]{\includegraphics[width=0.32\textwidth]{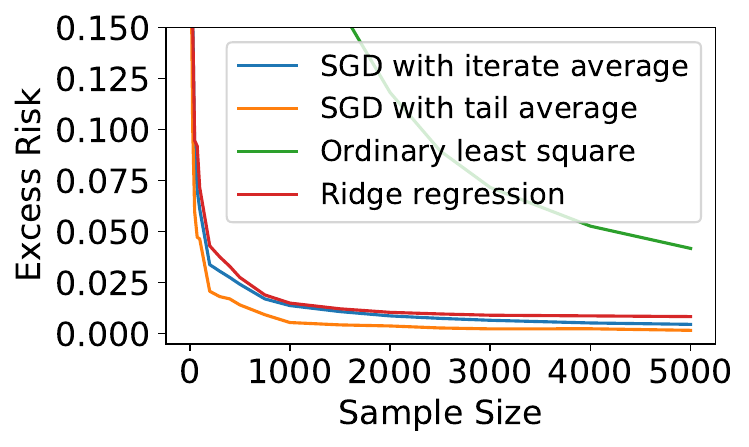}}\\
      \subfigure[$\lambda_i=i^{-2}, \wb^{*}{[i]}=1$ ]{\includegraphics[width=0.32\textwidth]{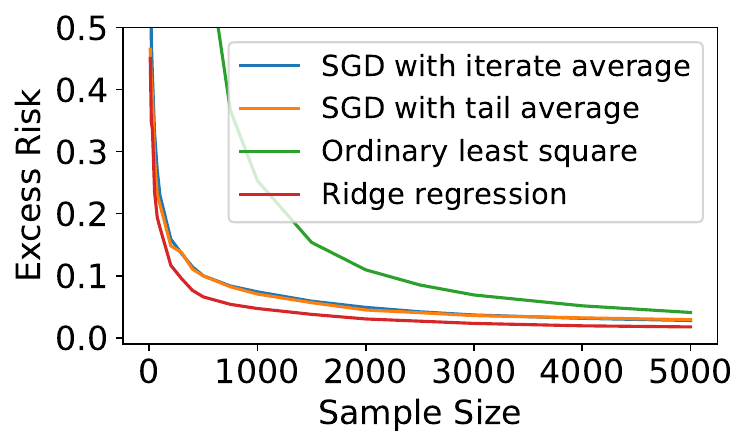}}
      \subfigure[$\lambda_i=i^{-2}, \wb^{*}{[i]}=i^{-1}$]{\includegraphics[width=0.32\textwidth]{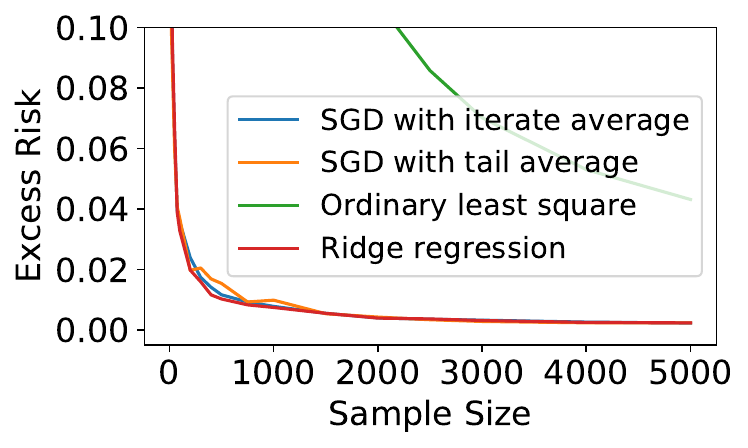}}
      \subfigure[$\lambda_i=i^{-2}, \wb^{*}{[i]}=i^{-10}$]{\includegraphics[width=0.32\textwidth]{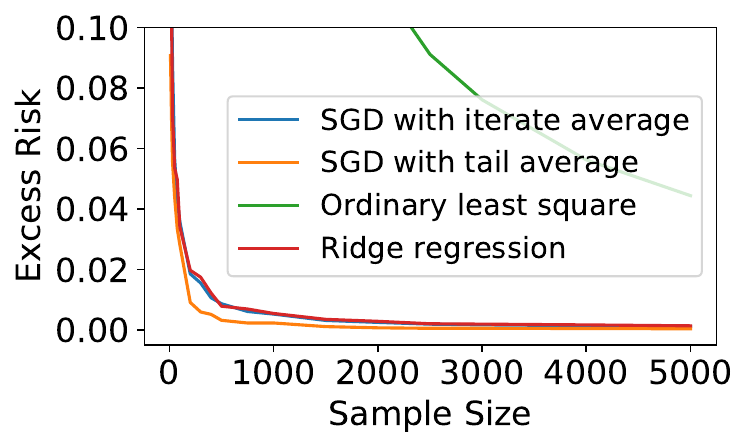}}
      \vskip -0.1in
    \caption{Excess risk comparison between SGD with iterate average, SGD with tail average, ordinary least square, and ridge regression, where the stepsize $\gamma$ and regularization parameter $\lambda$ are fine-tuned to achieve the best performance. The problem dimension is $d=200$ and the variance of model noise is $\sigma^2=1$. We consider $6$ combinations of $2$ different covariance matrices and $3$ different ground truth model vectors. The plots are averaged over $20$ independent runs.}
    \label{fig1}
  
\end{figure}

\section{Discussion}\label{sec:discussion}
This work considers the question of how well constant-stepsize SGD (with iterate average or tail average)
generalizes for the linear regression problem in the overparameterized
regime. Our main result provides a sharp excess risk bound, stated in
terms of the full eigenspectrum of the data covariance matrix. Our
results reveal how a benign-overfitting phenomenon can occur under
certain spectrum decay conditions on the data covariance.

There are number of more subtle points worth reflecting on:
\paragraph{Moving beyond the square loss.} Focusing on linear
regression is a means to understand phenomena that are exhibited more
broadly. One natural next step here would be understand the analogues
of the classical iterate averaging
results~\citep{polyak1992acceleration} for locally quadratic models,
where decaying stepsizes are necessary for vanishing risk.

\paragraph{Relaxing the data distribution assumption.}
While our data distribution assumption (Assumption
\ref{assump:bound_fourthmoment}) can be satisfied if the whitened data
is sub-Gaussian, it still cannot cover the simple one-hot case (i.e.,
$\xb= \eb_i$ with probability $p_i$, where
$\sum_{i}p_i=1$). Here, we conjecture that 
modifications of our proof can be used to establish
the theoretical guarantees of SGD under the following relaxed assumption on the
data distribution: assume that
$\EE[\xb\xb^\top\Ab\xb\xb^\top]\le a \tr(\Hb\Ab)\cdot \Hb +
b\|\Hb\|_2\cdot \Hb^{1/2}\Ab\Hb^{1/2}$ for all PSD matrix $\Ab$ and
some nonnegative constants $a$ and $b$, which is weaker than
Assumption~\ref{assump:bound_fourthmoment} in the sense that we can
allow $a=0$; this assumption captures the case where $\xb$ are
standard basis vectors, with $a=0$ and $b=1$. 



\section*{Acknowledgement}

DZ is supported by the Bloomberg Data Science Ph.D. Fellowship. JW is supported in part by NSF CAREER grant 1652257. VB is supported in part by NSF CAREER grant 1652257, ONR Award N00014-18-1-2364 and and the Lifelong Learning Machines program from DARPA/MTO. QG is partially supported by the National Science Foundation IIS-2008981. SK acknowledges funding from the National Science Foundation Award CCF-1703574.

\bibliography{refs}


\appendix

\section{Discussions on Assumption \ref{assump:bound_fourthmoment}}\label{sec:assump_discuss}

\citet{HsuKZ14,bartlett2020benign,tsigler2020benign} assume that $\zb := \Hb^{-\half} \xb $ is sub-Gaussian.
The following lemma shows that our Assumption \ref{assump:bound_fourthmoment} is implied by assuming sub-Gaussianity.



\begin{lemma}\label{lemma:sub-gaussian}
Suppose $\EE[\xb\xb^\top] = \Hb$, and $\zb := \Hb^{-\half} \xb $ is $\sigma_z^2$-sub-Gaussian random vector, then for any PSD matrix $\Ab$, we have 
\[\EE [ (\xb^\top \Ab \xb)\xb \xb^\top]  \preceq 16 \sigma_z^4 \tr(\Ab \Hb) \Hb. \]
\end{lemma}
\begin{proof}
Note that $\zb$ is a $\sigma_z^2$-sub-Gaussian random vector with identity covariance matrix, implying that for any fixed unit vector $\ub$ that $\ub^\top\zb$ is a $\sigma_z^2$-sub-Gaussian random variable. Then we have the following inequality for any unit vectors $\ub$ and $\vb$
\begin{align*}
    \EE [ (\ub^\top \zb)^2 (\vb^\top \zb)^2 ]&\le \sqrt{\EE [ (\ub^\top \zb)^4]}\cdot \sqrt{\EE[(\vb^\top \zb)^4]}
    \le \max\big\{\EE [(\ub^\top\zb)^4 ], \EE [(\vb^\top\zb)^4 ]\big\}
    \le 16\cdot \sigma_z^4,
\end{align*}
where  the first inequality follows from the Cauchy–Schwarz inequality; and  the last inequality uses the fact that $\ub^\top\zb$ is $\sigma_z^2$ sub-Gaussian. Here, the factor $16$ is due to the sub-Gaussian property (Proposition 2.5.2, \citet{vershynin2018high}).
Next, for any PSD matrix $\Ab$, suppose its eigenvalue decomposition is $\Ab = \sum_{i} \mu_i \ub_i \ub_i^\top$, where $\mu_i\ge 0$ is the eigenvalue and $\ub_i$ is the corresponding eigenvector, we have 
\begin{align}\label{eq:0013}
    \EE [ (\zb^\top \Ab \zb) \zb \zb^\top] = \sum_{i} \mu_i \EE [ (\ub_i^\top \zb)^2 \zb \zb^\top ].
\end{align}
For any unit vector $\vb$, we have:
\begin{align*}
\vb^\top\EE [ (\zb^\top \Ab \zb) \zb \zb^\top]\vb = \sum_{i}\mu_i\EE[(\ub_i^\top\zb)^2(\vb^\top\zb)^2]\le 16\cdot\sigma_z^4\cdot\sum_{i}\mu_i = 16\cdot\sigma_z^4\tr(\Ab).
\end{align*}
This implies that for any PSD matrix $\Ab$ we have
\begin{align}\label{eq:bound_MA_subgaussian}
\EE [ (\zb^\top \Ab \zb) \zb \zb^\top]\le 16\cdot\sigma_z^4\tr(\Ab)\Ib.
\end{align}
Finally considering $\xb = \Hb^{\half}\zb$, we have for any PSD matrix $\Ab$:
\begin{align*}
    \EE [ (\xb^\top \Ab \xb) \xb \xb^\top] 
    &=  \EE  [ (\zb^\top \Hb^\half \Ab \Hb^\half \zb) \Hb^\half\zb \zb^\top\Hb^\half]  \\
    & = \Hb^\half\EE  [ (\zb^\top \Hb^\half \Ab \Hb^\half \zb) \zb \zb^\top] \Hb^\half\\
    &\preceq \Hb^\half \cdot 16 \sigma_z^4 \tr(\Hb^\half \Ab \Hb^\half)\cdot \Ib\cdot \Hb^\half \\
    &= 16\sigma_z^4 \tr(\Ab \Hb) \Hb,
\end{align*}
where the second line holds since $\zb^\top \Hb^\half \Ab \Hb^\half \zb$ is a scalar and the third line of the above equation is due to \eqref{eq:bound_MA_subgaussian}.
This concludes the proof.
\end{proof}

\section{Proofs of the Upper Bounds}\label{append-sec:proof-upper-bound}

\subsection{Technical Lemma}

\begin{lemma}[Restatement of Lemma \ref{lemma:operators}]\label{lemma:operators2}
An operator $\cO$ defined on symmetric matrices is called PSD mapping, if $\Ab \succeq 0$ implies $\cO\circ \Ab \succeq 0$.
Then we have
\begin{enumerate}
    \item $\cM$ and $\tilde\cM$ are both PSD mappings.
    \item $\cI-\gamma\cT$ and $\cI-\gamma\tilde\cT$ are both PSD mappings.
    \item $\cM - \tilde\cM$ and $\tilde \cT - \cT$ are both PSD mappings.
\item  If $0 < \gamma < 1/\lambda_1$, then $\tilde{\cT}^{-1}$ exists, and is a PSD mapping.
\item  If $0 < \gamma < 1/(\alpha\tr(\Hb))$, then $\cT^{-1}\circ \Ab$ exists for PSD matrix $\Ab$, and $\cT^{-1}$ is a PSD mapping. 
\end{enumerate}
\end{lemma}
\begin{proof}
The following proofs are summarized from \citet{jain2017markov,jain2017parallelizing}, and we include them here for completeness.

    \begin{enumerate}
        \item 
     For any PSD matrix $\Ab \succeq 0$, by definition, we have
\begin{align*}
    \cM \circ \Ab &= \EE [\xb\xb^\top \Ab \xb\xb^\top] \succeq 0,\\
    \tilde\cM \circ\Ab &= \Hb \Ab \Hb \succeq 0.
\end{align*}
Therefore, both $\cM$ and $\tilde\cM$ are PSD mappings.

\item  For any PSD matrix $\Ab \succeq 0$, we have
\begin{align*}
    (\cI-\gamma\cT) \circ \Ab &= \EE [(\Ib - \gamma \xb \xb^\top) \Ab  (\Ib - \gamma \xb \xb^\top)] \succeq 0,\\
    (\cI-\gamma\tilde\cT) \circ \Ab &= (\Ib - \gamma\Hb) \Ab  (\Ib - \gamma\Hb) \succeq 0.
\end{align*}
Hence, $\cI-\gamma\cT$ and $\cI-\gamma\tilde\cT$ are both PSD mapping.

\item  For any  PSD matrix $\Ab \succeq 0$,
\[
(\cM - \tilde\cM) \circ \Ab = \EE[ \xb \xb^\top \Ab \xb \xb^\top] - \Hb \Ab \Hb
= \EE [ (\xb \xb^\top - \Hb) \Ab (\xb \xb^\top - \Hb) ] \succeq 0.
\]
Thus, $\tilde{\cT} - \cT = \cM - \tilde\cM$ is PSD.

\item  According to \eqref{eq:0005}, if $0< \gamma < 1/\lambda_1$, $\Ib - \gamma \Hb$ is a contraction map, thus for any symmetric matrix $\Ab$, the following exists:
\begin{equation*}
     \sum_{t=0}^\infty (\cI - \gamma \tilde\cT)^t \circ \Ab = \sum_{t=0}^\infty (\Ib - \gamma \Hb)^t \Ab (\Ib - \gamma \Hb)^t.
\end{equation*}
Therefore, $\sum_{t=0}^\infty (\cI - \gamma \tilde\cT)^t$ exists and 
\( \tilde \cT^{-1}  = \gamma \sum_{t=0}^\infty (\cI - \gamma \tilde \cT)^t\)
exists.
Furthermore, for any PSD matrix $\Ab\succeq 0$, we have 
\[ \tilde \cT^{-1}\circ \Ab =  \gamma \sum_{t=0}^\infty (\cI - \gamma \tilde \cT)^t\circ \Ab = \gamma \sum_{t=0}^\infty (\Ib - \gamma \Hb)^t \Ab (\Ib - \gamma \Hb)^t \succeq 0, \]
which implies $ \tilde \cT^{-1}$ is a PSD mapping. 

\item  For any finite PSD matrix $\Ab$, consider the following identity
\begin{align*}
\cT^{-1}\circ\Ab = \gamma\sum_{t=0}^\infty (\cI-\gamma\cT)^t\circ\Ab.
\end{align*}
Clearly, if the right hand side exists, it must be PSD since $\cI-\gamma\cT$ is a PSD mapping. It remains to show that $\sum_{t=0}^\infty (\cI-\gamma\cT)^t\circ\Ab$ is finite, and it suffices to show that 
\[\tr\rbr{\sum_{t=0}^\infty (\cI-\gamma\cT)^t\circ\Ab} = \sum_{t=0}^\infty \tr\rbr{ (\cI-\gamma\cT)^t\circ\Ab} < \infty .\]
Based on the definition of $\cT$, let $\Ab_t = (\cI-\gamma\cT)^t\circ\Ab$, we have
\begin{align}\label{eq:upperbound_At1}
\tr(\Ab_t) &= \tr(\Ab_{t-1}) - \gamma \tr(\Hb\Ab_{t-1}) - \gamma\tr(\Ab_{t-1}\Hb)+\gamma^2\tr\big(\EE[\xb\xb^\top\Ab\xb\xb^\top]\big)\notag\\
&=\tr(\Ab_{t-1}) - 2\gamma \tr(\Hb\Ab_{t-1}) + \gamma^2\tr\big(\Ab_{t-1}\EE[\xb\xb^\top\xb\xb^\top]\big).
\end{align}
By Assumption \ref{assump:bound_fourthmoment}, we have $\EE[\xb\xb^\top\xb\xb^\top] \preceq \alpha\tr(\Hb)\Hb$. Therefore, it follows that
\begin{align}\label{eq:upperbound_At2}
\tr(\Ab_t) &\le \tr(\Ab_{t-1}) - (2\gamma-\gamma^2 \alpha\tr(\Hb))\tr(\Hb\Ab_{t-1})\notag\\
&\le \tr\big((\Ib-\gamma\Hb)\Ab_{t-1}\big)\notag\\
&\le (1-\gamma\lambda_d)\tr(\Ab_{t-1}),
\end{align}
where we use the assumption $\gamma < 1/(\alpha\tr(\Hb))$ in the first inequality. This further implies that 
\begin{align*}
\sum_{t=0}^\infty \tr\rbr{ (\cI-\gamma\cT)^t\circ\Ab} = \sum_{t=0}^\infty \tr(\Ab_t)\le \frac{\tr(\Ab)}{\gamma\lambda_d} < \infty.
\end{align*}
Therefore, $\cT^{-1}\circ\Ab$ exists, and is PSD. So $\cT^{-1}$ is a PSD mapping.
\end{enumerate}
\end{proof}

\subsection{Bias-Variance Decomposition}\label{append-sec:proof-decomp}

\begin{lemma}[Bias-variance decomposition]\label{lemma:bias_var_decomposition}
\begin{align*}
\EE [L(\overline{\wb}_{N})] - L(\wb^*) = \frac{1}{2}\la\Hb,\EE[\bar\betab_{N}\otimes \bar\betab_{N}]\ra\le \rbr{ \sqrt{\bias} + \sqrt{\var} }^2,
\end{align*}
where 
\[
\bias := \frac{1}{2} \langle \Hb, \EE[{\bar\betab}^{\bias}_{N} \otimes {\bar\betab}^{\bias}_{N}] \rangle, \qquad 
\var := \frac{1}{2} \langle \Hb, \EE[{\bar\betab}^{\var}_{N} \otimes {\bar\betab}^{\var}_{N}] \rangle.
\]
\end{lemma}
\begin{proof}
This proof comes from \citep{jain2017markov}. For completeness we included it here.

With a slight abuse of notations (or probability spaces), we have \({\betab}_t = \betab^\bias_t + \betab^\var_t\), where the randomness of $\betab^\bias_t$ and $\betab^\var_t$ is understood as coming from the same probability space as ${\betab}_t$.
This implies 
\(\bar{\betab}_t = \bar{\betab}^\bias_t + \bar{\betab}^\var_t\).
Then we have
\begin{align*}
  & \EE [L(\overline{\wb}_{N})] - L(\wb^*) \\
  &= \frac{1}{2}\la\Hb,\EE[\bar\betab_{N}\otimes \bar\betab_{N}]\ra  \\
  &=  \EE \bigg[\frac{1}{\sqrt{2}}\bar\betab_{N}^\top\cdot \Hb\cdot \frac{1}{\sqrt{2}}\bar\betab_{N}\bigg] \\
  &\le \rbr{\sqrt{\EE \bigg[\bigg(\frac{1}{\sqrt{2}}\bar\betab_{N}^\bias\bigg)^\top\cdot \Hb\cdot \frac{1}{\sqrt{2}}\bar\betab_{N}^\bias\bigg]} + \sqrt{\EE \bigg[\bigg(\frac{1}{\sqrt{2}}\bar\betab_{N}^\var\bigg)^\top\cdot \Hb\cdot \frac{1}{\sqrt{2}}\bar\betab_{N}^\var\bigg]}}^2  \\
  &= \rbr{\sqrt{ \frac{1}{2} \langle \Hb, \EE[{\bar\betab}^{\bias}_{N} \otimes {\bar\betab}^{\bias}_{N}] \rangle }  + \sqrt{ \frac{1}{2} \langle \Hb, \EE[{\bar\betab}^{\var}_{N} \otimes {\bar\betab}^{\var}_{N}] \rangle } }^2,
\end{align*}
where we use Cauchy–Schwarz inequality in the inequality  such that for any vector $\ub$ and $\vb$, \( \EE \norm{\ub + \vb}^2_{\Hb} \le \rbr{ \sqrt{\EE \norm{\ub}_{\Hb}^2} + { \sqrt{\EE\norm{\vb}_{\Hb}^2}} }^2 \).
\end{proof}


\begin{lemma}\label{lemma:bias_var_decomposition_bound}
Recall iterates \eqref{eq:update_Bt} and \eqref{eq:update_Ct}. 
If the stepsize satisfies $\gamma\le1/\lambda_1$, the bias error and variance error are upper bounded respectively as follows:
\begin{gather*}
    \bias :=  \half \langle \Hb, \EE[{\bar\betab}^{\bias}_{N} \otimes {\bar\betab}^{\bias}_{N}] \rangle 
    \le \frac{1}{N^2}\sum_{t=0}^{N-1}\sum_{k=t}^{N-1}\big\la (\Ib-\gamma\Hb)^{k-t}\Hb,\Bb_t\big\ra, \\
    \var := \half \langle \Hb, \EE[{\bar\betab}^{\var}_{N} \otimes {\bar\betab}^{\var}_{N}] 
    \le \frac{1}{N^2}\sum_{t=0}^{N-1}\sum_{k=t}^{N-1}\big\la (\Ib-\gamma\Hb)^{k-t}\Hb,\Cb_t\big\ra.
\end{gather*}
\end{lemma}
\begin{proof}
The proof will largely rely on the calculation in \citet{jain2017parallelizing}. 
Firstly, based on the definitions of $ \betab_t^{\bias}$ and $\betab_t^{\bias}$ provided in \eqref{eq:bias_iterates} and \eqref{eq:variance_iterates}, we have
\begin{align}
\EE[\betab_t^{\bias}|\betab_{t-1}^{\bias}] &= \EE[\Pb_t\betab_{t-1}^{\bias}|\betab_{t-1}^{\bias}] = (\Ib-\gamma\Hb)\betab_{t-1}^{\bias}.\label{eq:bias_iterate_expectation}\\
\EE[\betab_t^{\var}|\betab_{t-1}^\var] &= \EE[\Pb_t\betab_{t-1}^{\var}+\gamma\xi_t\xb_t|\betab_{t-1}^{\var}] = (\Ib-\gamma\Hb)\betab_{t-1}^{\var}\label{eq:variance_iterate_expectation}.
\end{align}
Then regarding the quantity $\EE[\bar\betab_{N}^{\bias}\otimes \bar\betab_{N}^{\bias}]$, we have
\begin{align}
&\EE[\bar\betab_{N}^{\bias}\otimes \bar\betab_{N}^{\bias}]\notag\\
& = \frac{1}{N^2}\cdot\bigg(\sum_{0\le k\le t\le N-1}\EE[\betab_t^{\bias}\otimes \betab_k^{\bias}] + \sum_{0\le t<k\le N-1}\EE[\betab_t^{\bias}\otimes \betab_k^{\bias}]\bigg)\notag\\
& \preceq \frac{1}{N^2}\cdot\bigg(\sum_{0\le k\le t\le N-1}\EE[\betab_t^{\bias}\otimes \betab_k^{\bias}] + \sum_{0\le t\le k\le N-1}\EE[\betab_t^{\bias}\otimes \betab_k^{\bias}] \bigg)\notag\\
& = \frac{1}{N^2}\cdot\bigg(\sum_{0\le k\le t\le N-1}(\Ib-\gamma\Hb)^{t-k}\EE[\betab_k^{\bias}\otimes \betab_k^{\bias}] + \sum_{0\le t\le k\le N-1}\EE[\betab_t^{\bias}\otimes \betab_t^{\bias}](\Ib-\gamma\Hb)^{k-t} \bigg)\notag\\
& = \frac{1}{N^2}\cdot\sum_{t=0}^{N-1}\sum_{k=t}^{N-1}\Big((\Ib-\gamma\Hb)^{k-t}\EE[\betab_t^{\bias}\otimes \betab_t^{\bias}]+\EE[\betab_t^{\bias}\otimes \betab_t^{\bias}](\Ib-\gamma\Hb)^{k-t}\Big),\label{eq:expansion_average_outproduct}
\end{align}
where we use \eqref{eq:bias_iterate_expectation} for $k-t$ (or $t-k$) times in the second equality. 
Therefore, plugging \eqref{eq:expansion_average_outproduct} into
the inner product $\langle \Hb, \EE[{\bar\betab}^{\bias}_{N} \otimes {\bar\betab}^{\bias}_{N}] \rangle$ and noticing $\Hb$ is PSD, we have
\begin{align*}
&\half \langle \Hb, \EE[{\bar\betab}^{\bias}_{N} \otimes {\bar\betab}^{\bias}_{N}] \rangle\notag\\
&\le \frac{1}{2N^2}\cdot\sum_{t=0}^{N-1}\sum_{k=t}^{N-1}\Big\la\Hb,(\Ib-\gamma\Hb)^{k-t}\EE[\betab_t^{\bias}\otimes \betab_t^{\bias}]+\EE[\betab_t^{\bias}\otimes \betab_t^{\bias}](\Ib-\gamma\Hb)^{k-t}\Big\ra\notag\\
& =\frac{1}{N^2}\cdot\sum_{t=0}^{N-1}\sum_{k=t}^{N-1}\Big\la(\Ib-\gamma\Hb)^{k-t}\Hb,\EE[\betab_t^{\bias}\otimes \betab_t^{\bias}]\Big\ra
\end{align*}
where the last equality holds since $\Hb$ and $(\Ib-\gamma\Hb)^{k-t}$ commute.

By \eqref{eq:variance_iterate_expectation}, we can similarly obtain the following for $\EE[{\bar\betab}^{\var}_{N} \otimes {\bar\betab}^{\var}_{N}]$,
\begin{align*}
&\EE[\bar\betab_{N}^{\var}\otimes \bar\betab_{N}^{\var}]\notag\\
&\preceq\frac{1}{N^2}\cdot\sum_{t=0}^{N-1}\sum_{k=t}^{N-1}\Big((\Ib-\gamma\Hb)^{k-t}\EE[\betab_t^{\var}\otimes \betab_t^{\var}]+\EE[\betab_t^{\var}\otimes \betab_t^{\var}](\Ib-\gamma\Hb)^{k-t}\Big),
\end{align*}
which further leads to
\begin{align*}
\half \langle \Hb, \EE[{\bar\betab}^{\var}_{N} \otimes {\bar\betab}^{\var}_{N}] \rangle\le\frac{1}{N^2}\cdot\sum_{t=0}^{N-1}\sum_{k=t}^{N-1}\Big\la(\Ib-\gamma\Hb)^{k-t}\Hb,\EE[\betab_t^{\var}\otimes \betab_t^{\var}]\Big\ra.
\end{align*}
This completes the proof.

\end{proof}

\subsection{Bounding the Variance Error}\label{append-sec:proof-variance}
We first introduce a weaker assumption (compared with Assumption \ref{assump:bound_fourthmoment}) on the data distribution, which is sufficient to get our desired results on the variance error.

\begin{assumption}\label{assump:R2}
There exists a constant $R>0$ such that $\EE[\xb\xb^\top\xb\xb^\top]\preceq R^2\Hb$.
\end{assumption}
We make this assumption to emphasize that our variance analysis does not rely on stronger assumptions than those in a number of prior works for iterate averaged SGD \citep{bach2013non,jain2017parallelizing,berthier2020tight}. Moreover, note that this assumption is implied by Assumption \ref{assump:bound_fourthmoment} by setting $\Ab = \Ib$, which gives $R^2=\alpha\tr(\Hb)$.

Recall the variance error upper bound in Lemma \ref{lemma:bias_var_decomposition_bound}:
\begin{align*}
\var \le\frac{1}{N^2}\sum_{t=0}^{N-1}\sum_{k=t}^{N-1}\big\la(\Ib-\gamma\Hb)^{k-t}\Hb,\Cb_t\big\ra.
\end{align*}

We first have the following crude bound on $\Cb_t$.
\begin{lemma}\label{lemma:monotonicity_phit} (\citep{jain2017markov} Lemma 5)
Under Assumptions \ref{assump:second_moment}, \ref{assump:noise} and \ref{assump:R2}, if the stepsize satisfies $\gamma < 1/R^2$, it holds that
\begin{align*}
0 = \Cb_0 \preceq \Cb_1\preceq \cdots\preceq 
\Cb_\infty \preceq \frac{\gamma \sigma^2}{1-\gamma R^2}\Ib. 
\end{align*}
\end{lemma}
\begin{proof}
This lemma directly comes from Lemmas 3 and 5 in \citet{jain2017markov}. 
For completeness, a proof is included as follows.

We first show that $\Cb_t$ is increasing: 
\begin{align*}
    \Cb_t &= (\cI- \gamma\cT )\circ \Cb_{t-1} + \gamma^2 \bSigma \\
    &= \gamma^2 \sum_{k=0}^{t-1} (\cI - \gamma \cT)^k \circ \bSigma \qquad (\text{solving the recursion}) \\
    &= \Cb_{t-1} + \gamma^2 (\cI - \gamma \cT)^{t-1} \circ \bSigma \\
    &\succeq \Cb_{t-1}. \qquad (\text{since $\cI - \gamma \cT$ is a PSD mapping by Lemma \ref{lemma:operators} })
\end{align*}
Next we show that $\Cb_\infty$ exists. Since $\Cb_t$ is PSD and increasing, it suffices to show that $\tr(\Cb_t)$ can be bounded uniformly. For any $t\ge 1$, we have
\begin{align}\label{eq:0014}
    \Cb_t = \gamma^2 \sum_{k=0}^{t-1} (\cI - \gamma \cT)^k \circ \bSigma 
    \preceq \gamma^2 \sum_{t=0}^{\infty} (\cI - \gamma \cT)^t \circ \bSigma.
\end{align}
Let $\Ab_t := (\cI-\gamma\cT)^t\circ\bSigma$, then $ \Ab_t = (\cI-\gamma\cT)\circ \Ab_{t-1}$.
By Assumption \ref{assump:R2} we have $\EE[\xb\xb^\top\xb\xb^\top] \preceq R^2 \Hb$. Then, by \eqref{eq:upperbound_At1}, we can get
\begin{align}
\tr(\Ab_t) 
&=\tr(\Ab_{t-1}) - 2\gamma \tr(\Hb\Ab_{t-1}) + \gamma^2\tr\big(\Ab_{t-1}\EE[\xb\xb^\top\xb\xb^\top]\big)\notag\\
&\le \tr(\Ab_{t-1}) - (2\gamma-\gamma^2 R^2)\tr(\Hb\Ab_{t-1})\notag\\
&\le \tr\big((\Ib-\gamma\Hb)\Ab_{t-1}\big)\notag\notag\\
&\le (1-\gamma\lambda_d)\tr(\Ab_{t-1}),\label{eq:0015}
\end{align}
where we use the assumption $\gamma\le 1/R^2$ in the second inequality. Combining \eqref{eq:0014} and \eqref{eq:0015}, we have for any $t\geq 1$ that
\begin{align*}
\tr(\Cb_t) \le \gamma^2 \sum_{t=0}^\infty \tr\rbr{ (\cI-\gamma\cT)^t\circ\bSigma} = \gamma^2 \sum_{t=0}^\infty \tr(\Ab_t)\le \frac{\gamma \tr(\bSigma)}{\lambda_d} < \infty.
\end{align*}
Therefore, $\tr(\Cb_t)$ is uniformly upper bounded, hence $\Cb_\infty$ exists.

Finally we upper bound $\Cb_\infty$.
Taking limits in \eqref{eq:update_Bt}, we have 
\[ \Cb_\infty = (\cI - \gamma \cT)\circ \Cb_\infty + \gamma^2 \bSigma, \]
which immediately implies
\[\Cb_\infty = \gamma \cT^{-1}\circ \bSigma.\]
Recalling $\tilde{\cT} = \cT  + \gamma \cM  - \gamma \tilde\cM$ and the definitions and properties of the operators, we have
\begin{align*}
    \tilde\cT \circ \Cb_\infty 
    &= \cT \circ \Cb_\infty +\gamma \cM \circ \Cb_\infty - \gamma  \tilde\cM \circ \Cb_\infty \\
    &= \gamma \bSigma +\gamma \cM \circ \Cb_\infty -\gamma \tilde\cM \circ \Cb_\infty \qquad (\text{since $\Cb_\infty = \gamma \cT^{-1}\circ \bSigma$})\\
    &\preceq \gamma \bSigma +\gamma \cM \circ \Cb_\infty \qquad (\text{since $\tilde\cM$ is a PSD mapping by Lemma \ref{lemma:operators}}) \\
    &\preceq \gamma \sigma^2 \Hb +\gamma \cM \circ \Cb_\infty. \qquad (\text{since $\bSigma \preceq \sigma^2 \Hb$ by Assumption \ref{assump:noise}}) 
\end{align*}
Recall that $\tilde\cT^{-1}$ exists and is a PSD mapping by Lemma \ref{lemma:operators}, we then have
\begin{align}
     \Cb_\infty 
     &\preceq \gamma \sigma^2\cdot \tilde\cT^{-1}\circ \Hb + \gamma \tilde\cT^{-1}\circ \cM \circ \Cb_\infty \notag\\
     &\preceq \gamma \sigma^2\cdot \sum_{t=0}^\infty ( \gamma \tilde\cT^{-1}\circ \cM )^t \circ \tilde\cT^{-1}\circ \Hb. \qquad (\text{solving the recursion})\label{eq:0016}
\end{align}
In addition, we have  
\begin{align}\label{eq:0017}
    \tilde\cT^{-1} \circ \Hb 
    &=\gamma \sum_{t=0}^\infty (\cI - \gamma \tilde\cT)^t\circ \Hb\notag \\
    &= \gamma \sum_{t=0}^\infty (\Ib - \gamma \Hb)^t\Hb (\Ib - \gamma \Hb)^t \qquad (\text{by the property of $\cI-\tilde\cT$ in \eqref{eq:0005}})\notag\\
    &\preceq \gamma \sum_{t=0}^\infty (\Ib - \gamma \Hb)^t\Hb\notag \\
    &= \Ib. 
\end{align}
Substituting \eqref{eq:0017} into \eqref{eq:0016}, we obtain
\begin{align*}
     \Cb_\infty 
     &\preceq \gamma \sigma^2\cdot \sum_{t=0}^\infty ( \gamma \tilde\cT^{-1}\circ \cM )^t \circ \Ib \\
     &= \gamma \sigma^2\cdot \sum_{t=0}^\infty ( \gamma \tilde\cT^{-1}\circ \cM )^{t-1}\circ \gamma \tilde\cT^{-1}\circ \cM \circ \Ib \\
     &\preceq \gamma \sigma^2\cdot \sum_{t=0}^\infty ( \gamma \tilde\cT^{-1}\circ \cM )^{t-1}\circ \gamma R^2 \Hb \\
     & \preceq \gamma \sigma^2 \cdot  \sum_{t=0}^\infty (\gamma R^2)^t \Ib \\
     &= \frac{\gamma\sigma^2}{1-\gamma R^2}\Ib,
\end{align*}
where the second inequality is due to \(\cM \circ \Ib \preceq R^2 \Hb \) by Assumption \ref{assump:R2} and $\tilde\cT^{-1} \circ \Hb \preceq \Ib$ in \eqref{eq:0017}, and the third inequality is by recursion. This completes the proof.
\end{proof}

The following lemma refines the bound on $\Cb_t$ by its update rule and its crude bound shown in previous lemma.
\begin{lemma}\label{lemma:upper_bound_phit}
Under Assumptions \ref{assump:second_moment}, \ref{assump:noise} and \ref{assump:R2}, if the stepsize satisfies $\gamma < 1/R^2$, it holds that
\begin{equation*}
\Cb_t \preceq \frac{\gamma \sigma^2}{1-\gamma R^2}\cdot \rbr{\Ib - (\Ib - \gamma\Hb)^{t} } . 
\end{equation*}
\end{lemma}
\begin{proof}
By \eqref{eq:update_Ct} and the definitions of $\cT$ and $\tilde \cT$, we have
\begin{align}\label{eq:recursive_upperbound_Ct}
\Cb_t &= (\cI- \gamma\cT )\circ \Cb_{t-1} + \gamma^2 \bSigma \notag\\
& = (\cI - \gamma \tilde\cT) \circ \Cb_{t-1} + \gamma^2(\cM - \tilde\cM)\circ \Cb_{t-1}+\gamma^2\bSigma\notag\\
&\preceq (\cI - \gamma \tilde\cT) \circ \Cb_{t-1} + \gamma^2\cM\circ \Cb_{t-1}+\gamma^2\bSigma,
\end{align}
where the last inequality is due to the fact that $\tilde\cM$ is a PSD mapping.
Then by Lemma \ref{lemma:monotonicity_phit}, we have for all $t\ge 0$,
\begin{align}\label{eq:0010}
    \cM \circ \Cb_t  \preceq \cM \circ \Cb_\infty \preceq \cM  \circ \frac{\gamma \sigma^2}{1-\gamma R^2}\Ib
    = \frac{\gamma \sigma^2}{1-\gamma R^2} \cdot \EE [\nbr{\xb}_2^2 \xb\xb^\top]
    \preceq \frac{\gamma R^2 \sigma^2}{1-\gamma R^2} \cdot \Hb.
\end{align}
Substituting \eqref{eq:0010} and $\bSigma \preceq \nbr{\Hb^{-1/2}\bSigma\Hb^{-1/2}}_2 \cdot \Hb$ into \eqref{eq:recursive_upperbound_Ct}, we obtain 
\begin{align*}
\Cb_t
&\preceq (\cI-\gamma \tilde\cT) \circ \Cb_{t-1} + \gamma^2 \cdot  \frac{\gamma R^2 \sigma^2}{1-\gamma R^2} \cdot \Hb + \gamma^2\cdot  \|\Hb^{-1/2}\bSigma\Hb^{-1/2}\|_2 \cdot \Hb \notag \\
&= (\cI-\gamma \tilde\cT) \circ \Cb_{t-1} + \gamma^2 \cdot  \frac{\gamma R^2 \sigma^2}{1-\gamma R^2} \cdot \Hb + \gamma^2\sigma^2 \cdot \Hb \notag \\
&= (\cI-\gamma \tilde \cT) \circ \Cb_{t-1} + \frac{\gamma^2 \sigma^2}{1-\gamma R^2} \cdot \Hb \notag\\
&\preceq \frac{\gamma^2 \sigma^2}{1-\gamma R^2}\cdot \sum_{k=0}^{t-1}(\cI-\gamma\tilde\cT)^k \circ \Hb. \qquad (\text{solving the recursion})\notag\\
&=\frac{\gamma^2 \sigma^2}{1-\gamma R^2} \cdot \sum_{k=0}^{t-1} (\Ib - \gamma\Hb)^{k} \Hb (\Ib - \gamma\Hb)^{k} \qquad (\text{by the property of $\cI-\gamma\tilde\cT$ in \eqref{eq:0005}}) \notag\\
&\preceq \frac{\gamma^2 \sigma^2}{1-\gamma R^2} \cdot \sum_{k=0}^{t-1} (\Ib - \gamma\Hb)^{k} \Hb \notag\\
&= \frac{\gamma \sigma^2}{1-\gamma R^2}\cdot \rbr{\Ib - (\Ib - \gamma\Hb)^{t} }, 
\end{align*}
where in the last inequality we use $\gamma \le 1/R^2 \le 1/\tr(\Hb) \le 1/\lambda_1$.
This completes the proof.
\end{proof}

We are ready to provide the variance error upper bound.
\begin{lemma}\label{lemma:upperbound_var}
Under Assumptions \ref{assump:second_moment}, \ref{assump:noise} and \ref{assump:R2}, if the stepsize satisfies $\gamma < 1/R^2$, then it holds that
\begin{equation*}
    \var \le \frac{ \sigma^2}{1-\gamma R^2} \rbr{\frac{k^*}{N} + \gamma^2 N \cdot \sum_{i>k^*}\lambda_i^2  },
\end{equation*}
where $k^* = \max \{k: \lambda_k \ge \frac{1}{N \gamma}\}$.
\end{lemma}
\begin{proof}
By Lemma \ref{lemma:bias_var_decomposition},
we can bound the variance error as follows
\begin{align*}
   \var 
    &\le \frac{1}{N^2}\sum_{t=0}^{N-1}\sum_{k=t}^{N-1}\big\la(\Ib-\gamma\Hb)^{k-t}\Hb,\Cb_t\big\ra \\
    &= \frac{1}{\gamma N^2} \sum_{t=0}^{N-1} \big\la \Ib - (\Ib - \gamma\Hb)^{N-t} ,\Cb_t \big\ra \\
    &\le \frac{ \sigma^2}{ N^2 (1-\gamma R^2)} \sum_{t=0}^{N-1} \big\la \Ib - (\Ib - \gamma\Hb)^{N-t} , \rbr{\Ib - (\Ib - \gamma \Hb)^{t}}  \big \ra \\
    &= \frac{ \sigma^2}{N^2 (1-\gamma R^2)} \sum_{i} \sum_{t=0}^{N-1} \rbr{ 1 - (1 - \gamma\lambda_i)^{N-t} }\rbr{ 1 - (1 - \gamma \lambda_i)^{t} } \notag\\
    & \le \frac{ \sigma^2}{N^2 (1-\gamma R^2)} \sum_{i} \sum_{t=0}^{N-1} \rbr{ 1 - (1 - \gamma\lambda_i)^N }\rbr{ 1 - (1 - \gamma \lambda_i)^N }\notag\\
    & = \frac{ \sigma^2}{N (1-\gamma R^2)}\rbr{ 1 - (1 - \gamma \lambda_i)^N }^2,
\end{align*}
where the second inequality is due to Lemma \ref{lemma:upper_bound_phit}, $\{\lambda_i\}_{i\geq 1}$ are the eigenvalues of $\Hb$ and are sorted in decreasing order.
Since $\gamma\le 1/\lambda_1$, we have for all $i\ge 1$ that
\begin{align}\label{eq:upperbound_1-gamma_exponent}
1-(1-\gamma\lambda_i)^{N}\le \min\big\{1, \gamma N\lambda_i\big\}.
\end{align}
Set $k^* = \max \{k: \lambda_k \ge \frac{1}{ \gamma N}\}$, then
\begin{align*}
    \var 
    &\le \frac{ \sigma^2}{ N (1-\gamma R^2)} \sum_i \min\big\{1, \gamma^2 N^2\lambda_i^2\big\} \\
    &\le \frac{ \sigma^2}{ N (1-\gamma R^2)} \rbr{k^* + {N^2}\gamma^2 \cdot \sum_{i>k^*}\lambda_i^2  } \\
    &= \frac{ \sigma^2}{1-\gamma R^2} \rbr{\frac{k^*}{N} + \gamma^2 N \cdot \sum_{i>k^*}\lambda_i^2  }.
\end{align*}

\end{proof}

\subsection{Bounding the Bias Error}\label{append-sec:proof-bias}
In this part we will focus on bounding the bias error. Recall the bias error bound in Lemma~\ref{lemma:bias_var_decomposition_bound}:
\begin{align}
\bias &\le\frac{1}{N^2}\sum_{t=0}^{N-1}\sum_{k=t}^{N-1}\big\la(\Ib-\gamma\Hb)^{k-t}\Hb,\Bb_t\big\ra \notag\\
&= \frac{1}{\gamma N^2}\sum_{t=0}^{N-1} \big\la\Ib - (\Ib-\gamma\Hb)^{N-t}, \Bb_t\big\ra  \notag \\
&\le \frac{1}{\gamma N^2} \big\la\Ib - (\Ib-\gamma\Hb)^{N}, \sum_{t=0}^{N-1}\Bb_t\big\ra. \label{eq:bias-upperbound-0}
\end{align}
Let $\Sbb_n =\sum_{t=0}^{n-1}\Bb_t$, then we only need to bound $\Sbb_{N}$.

\begin{lemma}\label{lemma:properties_St}
Let $\Sbb_t =\sum_{k=0}^{t-1}\Bb_k$, if $\gamma< 1/(\alpha\tr(\Ab))$, we have
\begin{align*}
\Sbb_t &= (\cI-\gamma\cT)\circ\Sbb_{t-1}+\Bb_0.
\end{align*}
Moreover, it holds that
\begin{align*}
\Bb_0 = \Sbb_0\preceq \Sbb_1\preceq\cdots\preceq \Sbb_\infty. 
\end{align*}
\end{lemma}
\begin{proof}
 By \eqref{eq:update_Bt}, we have
\begin{align}\label{eq:0003}
\Bb_t = (\cI-\gamma\cT)\circ\Bb_{t-1}=(\cI-\gamma\cT)^{t}\circ\Bb_0,
\end{align}
where we used recursion. 
Then we have 
\begin{align*}
\Sbb_t = \sum_{k=0}^{t-1}(\cI-\gamma\cT)^k\circ\Bb_0 = (\cI-\gamma\cT)\circ\bigg(\sum_{k=0}^{t-1}(\cI-\gamma\cT)^k\circ\Bb_0\bigg) + \Bb_0 =(\cI-\gamma\cT)\circ\Sbb_{t-1} + \Bb_0.
\end{align*}
Moreover, since $\Bb_t$ is PSD for all $t\ge 0$, it is clear that $\Sbb_t = \Sbb_{t-1} + \Bb_t\succeq \Sbb_{t-1}$. Besides, by Lemma \ref{lemma:operators}, we know that 
\begin{align*}
\Sbb_\infty := \sum_{k=0}^\infty(\cI-\gamma\cT)^t\circ\Bb_0 = \gamma^{-1}\cT^{-1}\circ\Bb_0
\end{align*}
exists. Thus it can be readily shown that 
\begin{align*}
\Bb_0 =  \Sbb_1\preceq\cdots\preceq\Sbb_t\preceq\Sbb_{t+1}\preceq\cdots\preceq \Sbb_\infty,
\end{align*}
which completes the proof.
\end{proof}

\begin{lemma}\label{lemma:bound_M_Tinv_A}
Under Assumptions \ref{assump:bound_fourthmoment}, for any symmetric matrix $\Ab$, if $\gamma< 1/(\alpha\tr(\Hb))$, it holds that
\begin{align*}
  \cM\circ\cT^{-1}\circ \Ab\preceq\frac{\alpha\tr(\Ab)}{1-\gamma \alpha\tr(\Hb)}  \cdot\Hb.
\end{align*}
\end{lemma}
\begin{proof}
We first tackle $\cT^{-1}\circ\Ab$. In particular, by Lemma \ref{lemma:operators} we have the operator $\cT^{-1}$ exists and thus $\cT^{-1}\circ\Ab$ also exists, which can be obtained by solving for the PSD matrix $\Db$ satisfying the following equation,
\begin{align*}
 \cT \circ\Db = \Ab.
\end{align*}
Using the definition of $\tilde \cT$, we have:
\begin{align}\label{eq:station_equation_T}
\tilde \cT \circ\Db =\gamma \cM\circ\Db + \Ab - \gamma\Hb\Db\Hb,
\end{align}
where $\cM\circ\Db = \EE[\xb\xb^\top\Db\xb\xb^\top]$.
Further by Lemma \ref{lemma:operators} we know that $\tilde\cT^{-1}$ and $\cM$ are both PSD mapping. This implies that for any PSD matrices $\Ub$ and $\Ub'$ satisfying $\boldsymbol{0}\preceq \Ub\preceq\Ub'$, it holds that
\begin{align*}
\boldsymbol{0}\preceq \cM\circ\Ub\preceq\cM\circ\Ub',\qquad \boldsymbol{0}\preceq \tilde\cT^{-1}\circ\Ub\preceq\tilde\cT^{-1}\circ\Ub'.
\end{align*}
Combining the above two results we also have
\begin{align}\label{eq:property_operators}
\boldsymbol{0}\preceq \cM\circ\tilde\cT^{-1}\circ\Ub\preceq\cM\circ\tilde\cT^{-1}\circ\Ub'.
\end{align}
Therefore, applying the operator $\cT^{-1}$ to both sides of \eqref{eq:station_equation_T} yields
\begin{align}\label{eq:recursive_bound_D_1}
\Db &= \gamma\tilde\cT^{-1}\circ\cM\circ\Db+\tilde\cT^{-1}\circ\Ab - \gamma\tilde\cT^{-1}\circ(\Hb\Db\Hb)\notag\\
&\preceq\gamma\tilde\cT^{-1}\circ\cM\circ\Db+\tilde\cT^{-1}\circ\Ab.
\end{align}

Then we can apply the operator $\cM$ to both sides of \eqref{eq:recursive_bound_D_1}, by the monotonicity property in \eqref{eq:property_operators}, we have 
\begin{align}\label{eq:recursive_bound_MD}
\cM\circ\Db&\preceq \gamma\cM\circ\tilde\cT^{-1}\circ\cM\circ\Db+\cM\circ\tilde\cT^{-1}\circ\Ab\notag\\
&\preceq \sum_{t=0}^\infty (\gamma\cM\circ\tilde \cT^{-1})^t \circ(\cM\circ\tilde\cT^{-1}\circ\Ab).
\end{align}
By Assumption \ref{assump:bound_fourthmoment} we have
\begin{align}\label{eq:0007}
\cM\circ\tilde\cT^{-1}\circ\Ab\preceq \alpha\tr(\Hb \tilde\cT^{-1}\circ\Ab)\Hb.
\end{align}
Additionally, based on the definition of $\tilde\cT$, we have
\begin{align*}
\tilde\cT^{-1}\circ\Ab  = \gamma \sum_{t=0}^\infty (\cI-\gamma \tilde \cT)^t \circ \Ab= \gamma\sum_{t=0}^\infty(\Ib-\gamma\Hb)^{t}\Ab(\Ib-\gamma\Hb)^{t}.
\end{align*}
Therefore, it follows that
\begin{align}\label{eq:0008}
\tr(\Hb \tilde\cT^{-1}\circ\Ab) &= \gamma\tr\bigg(\sum_{t=0}^\infty \Hb(\Ib-\gamma\Hb)^t\Ab(\Ib-\gamma\Hb)^t\bigg)\notag\\
&=\gamma\tr\bigg(\sum_{t=0}^\infty \Hb(\Ib-\gamma\Hb)^{2t}\Ab\bigg)\notag\\
&= \tr\big( \Hb(2\Hb-\gamma\Hb^2)^{-1}\Ab\big)\notag\\
&\le \tr(\Ab),
\end{align}
where the last inequality is because we have $\gamma\le1/\lambda_1$ and thus $\Hb(2\Hb-\gamma\Hb^2)^{-1}\preceq\Ib$.
Substituting \eqref{eq:0008} into \eqref{eq:0007} yields 
\begin{align*}
\cM\circ\tilde\cT^{-1}\circ\Ab\preceq \alpha\tr(\Ab)\Hb.
\end{align*}
Note that we have $\tilde \cT^{-1}\Hb\preceq \Ib$ and $\cM\circ\Ib\preceq \alpha\tr(\Hb)\Hb$, plugging the above inequality into \eqref{eq:recursive_bound_MD} gives
\begin{align*}
\cM\circ\cT^{-1}\circ\Ab = \cM\circ\Db&\preceq\alpha\tr(\Ab)\sum_{t=0}^\infty( \gamma \alpha\tr(\Hb))^t\Hb\preceq\frac{\alpha \tr(\Ab)}{1- \gamma \alpha\tr(\Hb)}\cdot\Hb.
\end{align*}
This completes the proof.
\end{proof}

\begin{lemma}\label{lemma:T_inv}
Under Assumptions \ref{assump:second_moment}, and \ref{assump:bound_fourthmoment}, if the stepsize satisfies $\gamma< 1/(\alpha\tr(\Hb))$, then
\begin{align*}
\cM\circ\Sbb_t
\preceq 
\frac{\alpha\cdot\tr\big(\big[\cI-(\cI-\gamma\tilde \cT)^t\big]\circ\Bb_0\big)}{\gamma(1-\gamma\alpha\tr(\Hb))}\cdot \Hb.
\end{align*}
\end{lemma}
\begin{proof}
Note that $\Sbb_t$ takes the following form
\begin{align*}
\Sbb_t := \sum_{k=0}^{t-1}(\cI-\gamma\cT)^k\circ\Bb_0 = \gamma^{-1}\cT^{-1}\circ\big[\cI-(\cI-\gamma\cT)^t\big]\Bb_0.
\end{align*}
Note that by Lemma \ref{lemma:operators}, we have $\cI-\gamma\tilde\cT\le\cI-\gamma\cT$ so that $\cI-(\cI-\gamma\cT)^t\preceq \cI-(\cI-\gamma\tilde\cT)^t$. Therefore, further note that $\cT^{-1}$ is a PSD mapping, we have the following bound on $\Sbb_t$,
\begin{align*}
\Sbb_t\preceq \gamma^{-1}\cT^{-1}\circ\big[\cI-(\cI-\gamma\tilde \cT)^t\big]\circ\Bb_0.
\end{align*}
Then note that $[\cI-(\cI-\gamma\tilde \cT)^t\big]\circ\Bb_0$ is a PSD matrix, applying Lemma \ref{lemma:bound_M_Tinv_A}, we get
\begin{align*}
\cM\circ\Sbb_t \preceq\gamma^{-1}\cM\circ\cT^{-1}\circ\big[\cI-(\cI-\gamma\tilde \cT)^t\big]\circ\Bb_0 \preceq \frac{\alpha\cdot\tr\big(\big[\cI-(\cI-\gamma\tilde \cT)^t\big]\circ\Bb_0\big)}{\gamma(1-\gamma\alpha\tr(\Hb))}\cdot \Hb.
\end{align*}
This completes the proof.

\end{proof}

The following lemma shows that using this crude bound on $\cM\circ\Sbb_t$ we are able to get a tighter upper bound on $\Sbb_t$.
\begin{lemma}\label{lemma:upperbound_St}
Under Assumptions \ref{assump:second_moment} and \ref{assump:bound_fourthmoment}, let $\Bb_{a,b} = \Bb_a - (\Ib-\gamma\Hb)^{b-a}\Bb_a(\Ib-\gamma\Hb)^{b-a}$, if the stepsize satisfies $\gamma< 1/(\alpha\tr(\Hb))$, then for any $t\le N$, it holds that
\begin{align*}
\Sbb_t\preceq\sum_{k=0}^{t-1}(\Ib-\gamma\Hb)^k\bigg(\frac{\gamma\alpha\tr(\Bb_{0,N})}{1-\gamma \alpha \tr(\Hb)}\cdot\Hb+\Bb_0\bigg)(\Ib-\gamma\Hb)^k.
\end{align*}
\end{lemma}
\begin{proof}
Recall the recursive form of $\Sbb_t$ given in Lemma \ref{lemma:properties_St}, we have
\begin{align*}
\Sbb_t &= (\cI-\gamma\cT)\circ\Sbb_{t-1}+\Bb_0.
\end{align*}
Note that this is similar to the recursive form of $\Cb_t$ provided in \eqref{eq:update_Ct} but replacing $\gamma^2\bSigma$ with $\Bb_0$. Then we can use the similar proof of Lemma \ref{lemma:upper_bound_phit} to get the upper bound of $\Sbb_t$. In particular, note that we will run SGD with $N$ steps, then $\Sbb_N$ can be used as a uniform upper bound on $\Sbb_1,\dots,\Sbb_N$, we can upper bound $\Sbb_t$ by
\begin{align*}
\Sbb_t &\preceq
 (\cI - \gamma\tilde\cT)\circ\Sbb_{t-1} + \gamma^2\cM\circ\Sbb_N + \Bb_0\notag\\
 &\preceq(\cI - \gamma\tilde\cT)\circ\Sbb_{t-1} + \frac{\gamma\alpha\cdot\tr\big(\big[\cI-(\cI-\gamma\tilde \cT)^N\big]\circ\Bb_0\big)}{1-\gamma\alpha\tr(\Hb)}\cdot \Hb+\Bb_0\notag\\
 &= \sum_{k=0}^{t-1}(\cI-\gamma\tilde\cT)^k\circ\Bigg(\frac{\gamma\alpha\cdot\tr\big(\big[\cI-(\cI-\gamma\tilde \cT)^N\big]\circ\Bb_0\big)}{1-\gamma\alpha\tr(\Hb)}\cdot \Hb+\Bb_0\Bigg) \notag\\
& = \sum_{k=0}^{t-1}(\Ib-\gamma\Hb)^k\bigg(\frac{\gamma\alpha\tr\big(\Bb_0 - (\Ib-\gamma\Hb)^N\Bb_0(\Ib-\gamma\Hb)^N\big)}{1-\gamma \alpha \tr(\Hb)}\cdot\Hb+\Bb_0\bigg)(\Ib-\gamma\Hb)^k.
\end{align*}
where we use Lemma \ref{lemma:T_inv} in the second inequality, the first equality is by recursion, and the last equality is by the definition of $\tilde\cT$. 
\end{proof}

We now put these lemmas together and provide our upper bound on the bias error:

\begin{lemma}\label{lemma:bound_bias_final}
Under Assumptions \ref{assump:second_moment} and \ref{assump:bound_fourthmoment},
if the stepsize satisfies $\gamma < 1/(\alpha\tr(\Hb))$, it holds that 
\begin{align*}
\bias&\le \frac{1}{\gamma^2 N^2}\cdot\|\wb_0-\wb^*\|_{\Hb_{0:k^*}^{-1}}^2+\|\wb_0-\wb^*\|_{\Hb_{k^*:\infty}}^{2}\notag\\
&\quad +\frac{2\alpha\big(\|\wb_0-\wb^*\|_{\Ib_{0:k^*}}^2 + N\gamma\|\wb_0-\wb^*\|_{\Hb_{k^*:\infty}}^2\big)}{1-\gamma \alpha\tr(\Hb)}\cdot\bigg(\frac{k^*}{N} + N\gamma^2 \sum_{i> k^*}\lambda_i^2\bigg),
\end{align*}
where $k^* = \max \{k: \lambda_k \ge \gamma^{-1}/N\}$.
\end{lemma}
\begin{proof}
We can plug the upper bound of $\Sbb_t$ derived in Lemma \ref{lemma:upperbound_St} into \eqref{eq:bias-upperbound-0} and get
\begin{align*}
\bias &\le \frac{1}{\gamma N^2}\sum_{k=0}^{N-1}\bigg\la\Ib-(\Ib-\gamma\Hb)^{N},(\Ib-\gamma\Hb)^k\bigg(\frac{\gamma\alpha\tr(\Bb_{0,N})}{1-\gamma \alpha \tr(\Hb)}\cdot\Hb+\Bb_0\bigg)(\Ib-\gamma\Hb)^k\bigg\ra\notag\\
& = \frac{1}{\gamma N^2}\sum_{k=0}^{N-1}\bigg\la(\Ib-\gamma\Hb)^{2k}-(\Ib-\gamma\Hb)^{N+2k},\frac{\gamma\alpha\tr(\Bb_{0,N})}{1-\gamma \alpha \tr(\Hb)}\cdot\Hb+\Bb_0\bigg\ra.
\end{align*}
Note that
\begin{align*}
(\Ib-\gamma\Hb)^{2k}-(\Ib-\gamma\Hb)^{N+2k} &= (\Ib-\gamma\Hb)^k\big((\Ib-\gamma\Hb)^k - (\Ib-\gamma\Hb)^{N+k}\big)\notag\\
&\preceq(\Ib-\gamma\Hb)^k - (\Ib-\gamma\Hb)^{N+k}.
\end{align*}
We obtain
\begin{align*}
\bias \le \frac{1}{\gamma N^2}\sum_{k=0}^{N-1}\bigg\la(\Ib-\gamma\Hb)^k - (\Ib-\gamma\Hb)^{N+k},\frac{\gamma\alpha\tr(\Bb_{0,N})}{1-\gamma \alpha \tr(\Hb)}\cdot\Hb+\Bb_0\bigg\ra,
\end{align*}
Therefore, it suffices to upper bound the following two terms:
\begin{align*}
I_1 &= \frac{\alpha \tr(\Bb_{0,N})}{N^2(1-\gamma \alpha\tr(\Hb))}\sum_{k=0}^{N-1}\big\la(\Ib-\gamma\Hb)^{k} - (\Ib-\gamma\Hb)^{N+k},\Hb\big\ra\notag\\
I_2 &= \frac{1}{\gamma N^2}\sum_{k=0}^{N-1}\big\la(\Ib-\gamma\Hb)^{k} - (\Ib-\gamma\Hb)^{N+k},\Bb_0\big\ra.
\end{align*}
Regarding $I_1$, since $\Hb$ and $\Ib-\gamma\Hb$ can be diagonalized simultaneously, we have
\begin{align}\label{eq:bound_bias_I1}
I_1 &= \frac{\alpha\tr(\Bb_{0,N})}{N^2(1-\gamma \alpha\tr(\Hb))}\sum_{k=0}^{N-1}\sum_i \big[(1-\gamma\lambda_i)^k-(1-\gamma\lambda_i)^{N+k}\big]\lambda_i\notag\\
& = \frac{\alpha \tr(\Bb_{0,N})}{\gamma N^2(1-\gamma \alpha\tr(\Hb))}\sum_i \big[1-(1-\gamma\lambda_i)^{N}\big]^2\notag\\
&\le \frac{\alpha \tr(\Bb_{0,N})}{\gamma N^2(1-\gamma \alpha\tr(\Hb))}\sum_i \min\big\{1, \gamma^2 N^2\lambda_i^2\big\}\notag\\
&\le\frac{\alpha\tr(\Bb_{0,N})}{\gamma(1-\gamma \alpha\tr(\Hb))}\cdot\bigg(\frac{k^*}{N^2} + \gamma^2 \sum_{i> k^*}\lambda_i^2\bigg),
\end{align}
where $k^*$ is the index of the smallest eigenvalue of $\Hb$ satisfying $\lambda_{k^*}\ge \gamma^{-1}/N$.
Moreover, recall that $\tilde\Bb = \Bb_0 - (\Ib-\gamma\Hb)^N\Bb_0(\Ib-\gamma\Hb)^N\big)$ and $\Bb_0 = (\wb_0-\wb^*)\otimes (\wb_0-\wb^*)$, we have
\begin{align*}
\tr(\Bb_{0,N})& = \tr\big(\Bb_0 - (\Ib-\gamma\Hb)^N\Bb_0(\Ib-\gamma\Hb)^N\big)\big) = \sum_{i} \big(1 - (1-\gamma\lambda_i)^{2N})\cdot\big(\la\wb_0-\wb^*,\vb_i\ra\big)^2. 
\end{align*}
Note that 
\begin{align*}
\big(1 - (1-\gamma\lambda_i)^{2N})\le \min\{2, 2N\gamma\lambda_i\}, 
\end{align*}
thus it follows that,
\begin{align}\label{eq:bound_traceB0N}
\tr(\Bb_{0,N})\le 2\sum_{i} \min\{1,N\gamma\lambda_i\}\big(\la\wb_0-\wb^*,\vb_i\ra\big)^2 \le 2\big(\|\wb_0-\wb^*\|_{\Ib_{0:k^*}}^2 + N\gamma\|\wb_0-\wb^*\|_{\Hb_{k^*:\infty}}^2\big).
\end{align}
where $k^* = \max\{k: \lambda_k\ge \frac{1}{N\gamma}\}$.
Then plug this bound into \eqref{eq:bound_bias_I1}, we have
\begin{align}\label{eq:bound_bias_I1_final}
I_1
&\le\frac{2\alpha\big(\|\wb_0-\wb^*\|_{\Ib_{0:k^*}}^2 + N\gamma\|\wb_0-\wb^*\|_{\Hb_{k^*:\infty}}^2\big)}{N\gamma(1-\gamma \alpha\tr(\Hb))}\cdot\bigg(\frac{k^*}{N} + N\gamma^2 \sum_{i> k^*}\lambda_i^2\bigg),
\end{align}
In the sequel we will upper bound $I_2$. 
Let $\Hb = \Vb\bLambda\Vb^\top$ be the orthogonal decomposition of $\Hb$, where $\Vb = (\vb_1,\vb_2,\dots)$ and $\bLambda$ is a diagonal matrix with diagonal entries $\lambda_1,\lambda_2,\dots$. Then we have 
\begin{align*}
I_2 &= \frac{1}{\gamma N^2}\sum_{k=0}^{N-1}\big\la(\Ib-\gamma\bLambda)^{k} - (\Ib-\gamma\bLambda)^{N+k},\Vb^\top\Bb_0\Vb\big\ra.
\end{align*}
Note that $(\Ib-\gamma\bLambda)^{k} - (\Ib-\gamma\bLambda)^{N+k}$ is a diagonal matrix, thus the above inner product only operates on the diagonal entries of $\Vb^\top\Bb_0\Vb$. Note that $\Bb_0 = \betab_0\betab_0^\top$, it can be shown that the diagonal entries of $\Vb^\top\Bb_0\Vb$ are $\omega_1^2,\omega_2^2,\dots$, where $\omega_i = \vb_i^\top\betab_0 = \vb_i^\top(\wb_0-\wb^*)$.
\begin{align*}
I_2&= \frac{1}{\gamma N^2}\sum_{k=0}^{N-1}\big\la(\Ib-\gamma\Hb)^{k} - (\Ib-\gamma\Hb)^{N+k},\Bb_0\big\ra\notag\\
&= \frac{1}{\gamma N^2}\sum_{k=0}^{N-1}\sum_i\big[(1-\gamma\lambda_i)^k-(1-\gamma\lambda_i)^{N+k}\big]\omega_i^2\notag\\
& = \frac{1}{\gamma^2 N^2}\sum_i\frac{\omega_i^2}{\lambda_i}\big[1-(1-\gamma\lambda_i)^{N}\big]^2\notag\\
&\le \frac{1}{\gamma^2 N^2}\sum_i\frac{\omega_i^2}{\lambda_i}\min\big\{1, \gamma^2N^2\lambda_i^2\big\}\notag\qquad\\
& \le\frac{1}{\gamma^2 N^2}\cdot\sum_{i\le k^*}\frac{\omega_i^2}{\lambda_i}+\sum_{i> k^*}\lambda_i\omega_i^2\notag\\
& = \frac{1}{\gamma^2 N^2}\cdot\|\wb_0-\wb^*\|_{\Hb_{0:k^*}^{-1}}^2+\|\wb_0-\wb^*\|_{\Hb_{k^*:\infty}}^2,
\end{align*}
where the first inequality is by \eqref{eq:upperbound_1-gamma_exponent} and $k^*=\max\{k: \lambda_k\ge \gamma^{-1}/N\}$. Combining the upper bounds on $I_1$ and $I_2$ directly completes the proof.
\end{proof}


\subsection{Proof of Theorem \ref{thm:generalization_error}}
\begin{proof}
By Lemma \ref{lemma:bias_var_decomposition}, it suffices to substitute into the upper bounds on the bias and variance errors. In particular, by Young's inequality we have
\begin{align*}
\EE[L(\overline{\wb}_N)] - L(\wb^*) \le \Big(\sqrt{\text{bias}} + \sqrt{\text{variance}}\Big)^2\le 2\cdot\text{bias} + 2\cdot\text{variance}.
\end{align*}
Then we can directly substitute the bounds of $\text{variance}$ and $\text{bias}$ we proved in Lemmas \ref{lemma:upperbound_var} and \ref{lemma:bound_bias_final}. In particular, by Assumptions \ref{assump:bound_fourthmoment} we can directly get $R^2= \alpha\tr(\Hb)$. Therefore, it holds that 
\begin{align*}
&\EE[L(\overline{\wb}_N)] - L(\wb^*) \notag\\
&\le 2\bigg[\frac{\alpha \|\wb_0-\wb^*\|_2^2}{\gamma(1-\gamma \alpha\tr(\Hb))}\cdot\bigg(\frac{k^*}{N^2} + \gamma^2 \sum_{i> k^*}\lambda_i^2\bigg) + \frac{1}{\gamma^2 N^2}\cdot\|\wb_0-\wb^*\|_{\Hb_{0:k^*}^{-1}}^2+\|\wb_0-\wb^*\|_{\Hb_{k^*:\infty}}^{2}\notag\\
&\qquad +\frac{\sigma_z^2}{1-\gamma \alpha\tr(\Hb)}\bigg(\frac{k^*}{N}+\gamma^2N\cdot\sum_{i>k^*}\lambda_i^2\bigg)\bigg]\notag\\
&= 2\cdot \mathrm{EffectiveBias}+2\cdot \mathrm{EffectiveVar},
\end{align*}
where 
\begin{align*}
\mathrm{EffectiveBias} & = \frac{1}{\gamma^2 N^2}\cdot\|\wb_0-\wb^*\|_{\Hb_{0:k^*}^{-1}}^2+\|\wb_0-\wb^*\|_{\Hb_{k^*:\infty}}^{2} \\
 \mathrm{EffectiveVar} & = \bigg(\frac{ \sigma_z^2}{1-\gamma \alpha\tr(\Hb)}+\frac{\alpha \|\wb_0-\wb^*\|_2^2}{N\gamma (1-\gamma \alpha\tr(\Hb))}\bigg) \rbr{\frac{k^*}{N} + \gamma^2 N \cdot \sum_{i>k^*}\lambda_i^2  }.
\end{align*}
\end{proof}

\subsection{Proof of Corollary \ref{thm:simplied_theory}}
\begin{proof}
We will show that the corollary can be directly implied by Theorem~\ref{thm:generalization_error}. In terms of the effective bias term, it is clear that 
\begin{align*}
\mathrm{EffectiveBias} &\le\frac{1}{\gamma^2 N^2}\cdot\|\wb_0-\wb^*\|_{\Hb_{0:k^*}^{-1}}^2+\|\wb_0-\wb^*\|_{\Hb_{k^*:\infty}}^{2}\notag\\
&=\frac{1}{\gamma^2N^2}\cdot\lambda_{k^*}^{-1}\sum_{i\le k^*}\big(\vb_i^\top\wb_0-\vb_i^\top\wb^*\big)^2 + \lambda_{k^*+1}\sum_{i>k^*}\big(\vb_i^\top\wb_0-\vb_i^\top\wb^*\big)^2.
\end{align*}
where $\vb_i$ is the eigenvector of $\Hb$ corresponding to the eigenvalue $\lambda_i$.
Based on our definition of $k^*$, we have $\lambda_{k^*}^{-1}\le N\gamma$ and $\lambda_{k^*+1}\le 1/(N\gamma)$. Therefore, it follows that
\begin{align}\label{eq:0018}
\mathrm{EffectiveBias} \le \frac{1}{\gamma N}\cdot\sum_{i}\big(\vb_i^\top\wb_0-\vb_i^\top\wb^*\big)^2 = \frac{\|\wb_0-\wb^*\|_2^2}{\gamma N}.
\end{align}
Then regarding the effective variance, given the choice of stepsize that $\gamma = 1/(2\alpha\tr(\Hb))$, we have
\begin{align*}
\mathrm{EffectiveVar} &\le 2\bigg(\sigma^2+\frac{\alpha \|\wb_0-\wb^*\|_2^2}{N\gamma}\bigg) \rbr{\frac{k^*}{ N} + \gamma^2 N \cdot \sum_{i>k^*}\lambda_i^2  }\notag\\
& = 2\sigma^2\cdot\rbr{\frac{k^*}{N} + \gamma^2 N \cdot \sum_{i>k^*}\lambda_i^2  }+\frac{2\alpha \|\wb_0-\wb^*\|_2^2}{\gamma N}\cdot\rbr{\frac{k^*}{N} + \gamma^2 N \cdot \sum_{i>k^*}\lambda_i^2  }. 
\end{align*}
Based on the definition of $k^*$, we have $\lambda_i \leq 1/(N\gamma)$ for $i> k^*$, thus
\begin{align*}
 \gamma^2N\sum_{i>k^*}\lambda_i^2\le \gamma \sum_{i> k^*}\lambda_i. 
\end{align*}
Besides, we also have $k^*/N\le \gamma\sum_{i=1}^{k^*}\lambda_i$. Therefore, we have 
\begin{align*}
\mathrm{EffectiveVar} &\le 2\sigma^2\cdot\rbr{\frac{k^*}{N} + \gamma^2 N \cdot \sum_{i>k^*}\lambda_i^2  }+\frac{2\gamma\alpha  \|\wb_0-\wb^*\|_2^2}{\gamma N}\cdot\sum_{i}\lambda_i.
\end{align*}
According to our choice of stepsize that $\gamma = 1/(2\alpha\tr(\Hb))$, we can get
\begin{align*}
\frac{2\gamma\alpha  \|\wb_0-\wb^*\|_2^2}{\gamma N}\cdot\sum_{i}\lambda_i = \frac{\|\wb_0-\wb^*\|_2^2}{\gamma N}.
\end{align*}
This further implies that 
\begin{align}\label{eq:0019}
\mathrm{EffectiveVar} &\le2\sigma^2\cdot\rbr{\frac{k^*}{N} + \gamma^2 N \cdot \sum_{i>k^*}\lambda_i^2  }+\frac{\|\wb_0-\wb^*\|_2^2}{\gamma N}.
\end{align}
Combining \eqref{eq:0018} and \eqref{eq:0019}, we have
\begin{align*}
\EE[L(\overline{\wb}_N)] - L(\wb^*)&\le 2\cdot\mathrm{EffectiveBias} + 2\cdot\mathrm{EffectiveVar}\notag\\
&\le \frac{4\|\wb_0-\wb^*\|_2^2}{\gamma N} + 4\sigma^2\cdot\rbr{\frac{k^*}{N} + \gamma^2 N \cdot \sum_{i>k^*}\lambda_i^2  }.
\end{align*}
Further using the assumption that $\gamma = 1/(2\alpha\tr(\Hb))$ completes the proof.
\end{proof}

\subsection{Proof of Corollary \ref{thm:example_spectrum}}
\begin{proof}
For the bias error term, recall the definition of $k^*$,  we have 
\begin{align*}
 \mathrm{EffectiveBias} & \le \bigO{\frac{1}{ N^2} \cdot \norm{\wb_0 - \wb^*}^2_{\Hb^{-1}_{0:k^*}} + \norm{\wb_0 - \wb^*}^2_{\Hb_{k^*:\infty}} } \\
 &\le \bigO{\frac{1}{ N^2} \cdot \frac{1}{\lambda_{k^*}}\cdot \norm{\wb_0 - \wb^*}^2_2 + \lambda_{k^*}\cdot \norm{\wb_0 - \wb^*}^2_2 } \\
 &\le \bigO{\frac{1}{N}}.
\end{align*}
For the variance error term, it can be verified that all these examples satisfies $\sum_i\lambda_i< \infty$, thus we have 
\begin{align*}
 \mathrm{EffectiveVar} & = \bigO{ \frac{k^*}{N} 
+  N \sum_{i>k^*}\lambda_i^2 }.
\end{align*}
\begin{enumerate}
    \item By the definition of $k^*$ we have $k^* = s = N^r$, therefore
    \begin{align*}
 \mathrm{EffectiveVar} & = \bigO{ N^{-1} \cdot N^r
+  N \cdot N^{-q} } = \bigO{N^{r-1} + N^{1-q}}.
\end{align*}

\item By the definition of $k^*$ we have $k^* =\Theta\big(N^{1/(1+r)}\big)$, therefore
\begin{align*}
 \mathrm{EffectiveVar} & = \bigO{ N^{-1} \cdot N^{1/(1+r)}
+  N \cdot \rbr{ N^{1/(1+r)} }^{-1-2r} } = \bigO{N^{-r/(1+r)}}.
\end{align*}

\item By the definition of $k^*$ it can be shown that $k^*= \Omega\big(N/\log^\beta(N)\big)$ since otherwise 
\begin{align*}
\lambda_{k^*+1} = \omega\bigg(\frac{\log^\beta (N)}{N}\cdot \frac{1}{\big[\log(N)-\beta\log(\log(N))\big]^\beta}\bigg) = \omega(1/N),
\end{align*}
which contradicts to the fact that $\lambda_{k^*+1}=\bigO{1/N}$. 
Besides, we have
\begin{align*}
\sum_{i\ge k^*}\lambda_i^2 = \bigO{\int_{k^*}^\infty \frac{1}{x^2\log^{2\beta}(x+1)} \dd x  }.
\end{align*}
Then note that
\begin{align*}
\frac{1}{x^2\log^{2\beta}(x+1)}\le \frac{\log^{2\beta}(x+1)+2\beta x\log^{2\beta-1}(x)/(x+1)}{x^2\log^{4\beta}(x)}.
\end{align*}
This implies that 
\begin{align*}
\int_{k^*}^\infty \frac{1}{x^2\log^{2\beta}(x)} \dd x  &\le \int_{k^*}^\infty\frac{\log^{2\beta}(x+1)+2\beta x\log^{2\beta-1}(x)/(x+1)}{x^2\log^{4\beta}(x)}\dd x \notag\\
&= \frac{1}{k^*\log^{2\beta}(k^*+1)} \notag\\
&= \bigO{N^{-1}\log^{-\beta}(k^*)},
\end{align*}
where the last equality is due to the fact that $1/(k^*\log^{\beta}(k^*+1)) = \Theta(1/N)$.
As a result, we can get 
\begin{align*}
\mathrm{EffectiveVar} &= \bigO{k^*\cdot N^{-1} + N\sum_{i\ge k^*}\lambda_i^2} = \bigO{\log^{-\beta}(k^*)} = \bigO{\log^{-\beta}(N)},
\end{align*}
where the second equality is due to the fact that $k*/N = \bigO{\log^{-\beta}(k^*)}$ and the last equality is due to $k^*= \Omega\big(N/\log^\beta(N)\big)$.

\item By definition of $k^*$ we have $k^* = \bigTht{\log N}$, therefore
\begin{align*}
 \mathrm{EffectiveVar} & = \bigO{ N^{-1} \cdot \log N
+  N \cdot e^{-2\log N} } = \bigO{N^{-1}\log N}.
\end{align*}
\end{enumerate}
Summing up the bias error and variance error concludes the proof.
\end{proof}

\section{Proofs of the Lower Bounds}\label{sec:lower_bound}

\subsection{Lower Bound for Bias-Variance Decomposition}
We first introduce the following lemma to lower bound the excess risk when the noise is well-specified as in \eqref{eq:well}.
\begin{lemma}\label{lemma:lower_bound_decomp}
Suppose the model noise $\xi_t$ is well-specified, i.e., $\xi_t$ and $\xb_t$ are independent and $\EE [\xi_t] = 0$. 
Then 
\begin{align*}
    \EE [ L(\overline{\wb}_N) - L(\wb^*)] 
& \ge \frac{1}{2N^2}\cdot\sum_{t=0}^{N-1}\sum_{k=t}^{N-1}\Big\la(\Ib-\gamma\Hb)^{k-t}\Hb,\Bb_t \Big\ra \\
&\quad  + \frac{1}{2N^2}\cdot\sum_{t=0}^{N-1}\sum_{k=t}^{N-1}\Big\la(\Ib-\gamma\Hb)^{k-t}\Hb,\Cb_t\Big\ra.
\end{align*}
\end{lemma}
\begin{proof}
Let $\Pb_t = \Ib - \gamma \xb_t \xb_t^\top$, then the definitions of $\betab_t^{\bias}$ in \eqref{eq:variance_iterates} and $\betab_t^{\var}$ \eqref{eq:bias_iterates} imply
\begin{align*}
    \betab_t^{\bias} = \prod_{k=1}^t \Pb_k \betab_0, \qquad
    \betab_t^{\var} = \gamma \sum_{i=1}^t \prod_{j=i+1}^t \xi_i\Pb_j  \xb_i.
\end{align*}
Note that in the well specified case, the noise $\xi_t := y_t - \la\wb^*, \xb_t\ra$ is independent of the data $\xb_t$, and is of zero mean, hence
\begin{align*}
\EE[\betab_t^{\bias}\otimes \betab_t^{\var}] 
&= \gamma\EE\bigg[\prod_{k=1}^t \Pb_k \betab_0 \otimes \sum_{i=1}^t \prod_{j=i+1}^t \xi_i\Pb_j  \xb_i \bigg] \\
&= \gamma\sum_{i=1}^t \EE\big[\prod_{k=1}^t \Pb_k \betab_0 \otimes  \prod_{j=i+1}^t \Pb_j  \xb_i \big]\cdot \EE[\xi_i] 
= \boldsymbol{0}.
\end{align*}
This implies that 
\[ 
\EE[\bar{\betab}_t \otimes \bar{\betab}_t] = \EE[\bar{\betab}_t^\bias \otimes \bar{\betab}_t^\bias] + \EE[\bar{\betab}_t^\var \otimes \bar{\betab}_t^\var ], \]
and furthermore,
\begin{align}\label{eq:decomposition_error_wellspecified}
\EE [ L(\overline{\wb}_N) - L(\wb^*)] 
&= \half \la \Hb,\EE[\bar{\betab}_t \otimes \bar{\betab}_t]  \ra \notag\\
&= \half \la \Hb,\EE[\bar{\betab}_t^\bias \otimes \bar{\betab}_t^\bias]  \ra + \half \la \Hb,\EE[\bar{\betab}_t^\var \otimes \bar{\betab}_t^\var]  \ra.
\end{align}
Next, we lower bound each term on the R.H.S. of \eqref{eq:decomposition_error_wellspecified} separately.
By \eqref{eq:expansion_average_outproduct}, we have
\begin{align}
\EE[\bar\betab_{N}^{\bias}\otimes \bar\betab_{N}^{\bias}]\notag
& = \frac{1}{N^2}\cdot\bigg(\sum_{0\le k < t\le N-1}\EE[\betab_t^{\bias}\otimes \betab_k^{\bias}] + \sum_{0\le t \le k\le N-1}\EE[\betab_t^{\bias}\otimes \betab_k^{\bias}]\bigg).
\end{align}
Additionally, by \eqref{eq:bias_iterate_expectation} we can get
\begin{align*}
\bigg\la\Hb,\sum_{0\le k < t\le N-1}\EE[\betab_t^{\bias}\otimes \betab_k^{\bias}] \bigg\ra &= \bigg\la\Hb, \sum_{k=0}^{N-1}\sum_{t= k+1}^{N-1}  (\Ib-\gamma\Hb)^{t-k}\EE[\betab_k^{\bias}\otimes \betab_k^{\bias}]\bigg\ra\notag\\
& = \sum_{k=0}^{N-1}\sum_{t= k+1}^{N-1} \big\la (\Ib-\gamma\Hb)^{t-k}\Hb,\EE[\betab_k^{\bias}\otimes \betab_k^{\bias}]\big\ra
\ge 0,
\end{align*}
where the inequality is due to the fact that $(\Ib-\gamma\Hb)^{t-k}\Hb$ and $\EE[\betab_k^{\bias}\otimes \betab_k^{\bias}]$ are both PSD. Therefore, it follows that
\begin{align}\label{eq:lowerbound_bias_decomp}
\bias&:=\frac{1}{2}\langle \Hb, \EE[{\bar\betab}^{\bias}_{N} \otimes {\bar\betab}^{\bias}_{N}] \rangle\notag\\
&\ge \frac{1}{2N^2}\cdot\bigg\la\Hb,\sum_{0\le t \le k\le N-1}\EE[\betab_t^{\bias}\otimes \betab_k^{\bias}]\bigg\ra \notag\\
&= \frac{1}{2N^2}\cdot\sum_{t=0}^{N-1}\sum_{k=t}^{N-1}\Big\la\Hb,\EE[\betab_t^{\bias}\otimes \betab_t^{\bias}]\cdot(\Ib-\gamma\Hb)^{k-t}\Big\ra  \notag\\
& =\frac{1}{2N^2}\cdot\sum_{t=0}^{N-1}\sum_{k=t}^{N-1}\Big\la(\Ib-\gamma\Hb)^{k-t}\Hb,\EE[\betab_t^{\bias}\otimes \betab_t^{\bias}]\Big\ra,
\end{align}
where the last equality holds since $\Hb$ and $(\Ib-\gamma\Hb)^{k-t}$ commute.
Repeating the computation for the variance terms, we can similarly obtain
\begin{align}\label{eq:lowerbound_var_decomp}
\var&:=\frac{1}{2}\langle \Hb, \EE[{\bar\betab}^{\var}_{N} \otimes {\bar\betab}^{\var}_{N}] \rangle\notag\\
&\ge \frac{1}{2N^2}\cdot\sum_{t=0}^{N-1}\sum_{k=t}^{N-1}\Big\la(\Ib-\gamma\Hb)^{k-t}\Hb,\EE[\betab_t^{\var}\otimes \betab_t^{\var}]\Big\ra.
\end{align}
Plugging \eqref{eq:lowerbound_bias_decomp} and \eqref{eq:lowerbound_var_decomp} into \eqref{eq:decomposition_error_wellspecified} gives
\begin{align*}
    \EE [ L(\overline{\wb}_N) - L(\wb^*)] 
&= \half \la \Hb,\EE[\bar{\betab}_t \otimes \bar{\betab}_t]  \ra 
=\bias+\var\notag\\
& \ge \frac{1}{2N^2}\cdot\sum_{t=0}^{N-1}\sum_{k=t}^{N-1}\Big\la(\Ib-\gamma\Hb)^{k-t}\Hb,\EE[\betab_t^{\bias}\otimes \betab_t^{\bias}]\Big\ra \\
&\quad  + \frac{1}{2N^2}\cdot\sum_{t=0}^{N-1}\sum_{k=t}^{N-1}\Big\la(\Ib-\gamma\Hb)^{k-t}\Hb,\EE[\betab_t^{\var}\otimes \betab_t^{\var}]\Big\ra.
\end{align*}

\end{proof}

\subsection{Lower Bounding the Variance Error}

\begin{lemma}\label{lemma:lower_bound_phit}
Suppose Assumptions \ref{assump:second_moment} hold.
Suppose the noise is well-specified as in \eqref{eq:well}.
If the stepsize satisfies $\gamma < 1/\lambda_1$, it holds that
\begin{align*}
\Cb_t \succeq \frac{\gamma\sigma^2_{\mathrm{noise}}}{2 } \rbr{\Ib -  (\Ib-\gamma\Hb)^{2t} }. 
\end{align*}
\end{lemma}
\begin{proof}
Recall that $\cM - \tilde\cM$ is a PSD mapping by Lemma \ref{lemma:operators} and $\Cb_{t-1}$ is PSD, then from \eqref{eq:update_Ct} we have
\begin{align*}
\Cb_t 
&= (\cI - \gamma\cT) \circ \Cb_{t-1} + \gamma^2 \bSigma \\
&= (\cI -\gamma\tilde\cT) \circ \Cb_{t-1} + (\cM - \tilde\cM) \circ \Cb_{t-1} + \gamma^2 \bSigma \\
&\succeq (\cI - \gamma\tilde\cT ) \circ \Cb_{t-1} + \gamma^2 \sigma^2_{\textrm{noise}}\Hb \qquad (\text{since in the well-specified case $\bSigma=\sigma^2_{\textrm{noise}}\Hb$})\\
&\succeq \gamma^2\sigma^2_{\textrm{noise}}\cdot \sum_{k=0}^{t-1}(\cI - \gamma\tilde\cT)^k \circ \Hb \qquad (\text{solving the recursion})\\
&= \gamma^2\sigma^2_{\textrm{noise}}\cdot \sum_{k=0}^{t-1}(\Ib-\gamma\Hb)^{k} \Hb (\Ib-\gamma\Hb)^{k} \qquad (\text{by the property of $\cI-\gamma\tilde\cT$ in \eqref{eq:0005}})\\
&=  \gamma^2\sigma^2_{\textrm{noise}}\cdot  \rbr{\Ib -  (\Ib-\gamma\Hb)^{2t} } \cdot \rbr{ 2\gamma \Ib  - \gamma^2 \Hb }^{-1} \\
&\succeq\frac{\gamma\sigma^2_{\textrm{noise}}}{2}\cdot  \rbr{\Ib -  (\Ib-\gamma\Hb)^{2t} },
\end{align*}
where in the last inequality we use $2\gamma \Ib - \gamma^2 \Hb \preceq 2\gamma\Ib$.
This completes the proof.
\end{proof}

\begin{lemma}\label{lemma:lowerbound_var}
Suppose Assumptions \ref{assump:second_moment} hold. Suppose the noise is well-specified as in \eqref{eq:well} and
$N\geq 500$.
Denote 
\[\var = \frac{1}{2N^2}\cdot\sum_{t=0}^{N-1}\sum_{k=t}^{N-1}\Big\la(\Ib-\gamma\Hb)^{k-t}\Hb,\Cb_t \Big\ra.\]
If the stepsize satisfies $\gamma < 1/\lambda_1$, then
\begin{equation*}
    \var \ge \frac{ \sigma^2_{\mathrm{noise}} }{50} \rbr{\frac{k^*}{N} + N\gamma^2 \cdot \sum_{i>k^*}\lambda_i^2  },
\end{equation*}
where $k^* = \max \{k: \lambda_k \ge \frac{1}{ N \gamma}\}$.
\end{lemma}
\begin{proof}
We can lower bound the variance error as follows
\begin{align*}
   \var 
    &= \frac{1}{2 N^2}\sum_{t=0}^{N-1}\sum_{k=t}^{N-1}\big\la(\Ib-\gamma\Hb)^{k-t}\Hb,\Cb_t\big\ra \\
    &= \frac{1}{2\gamma  N^2} \sum_{t=0}^{N-1} \big\la \Ib - (\Ib - \gamma\Hb)^{N-t} ,\Cb_t \big\ra \\
    &\ge \frac{ \sigma_{\textrm{noise}}^2 }{4 N^2} \sum_{t=0}^{N-1} \big\la \Ib - (\Ib - \gamma\Hb)^{N-t} , \Ib - (\Ib - \gamma \Hb)^{2t} \big \ra \qquad (\text{use Lemma \ref{lemma:lower_bound_phit}}) \\
    &= \frac{ \sigma_{\textrm{noise}}^2 }{4 N^2}\sum_{i} \sum_{t=0}^{N-1} \rbr{ 1 - (1 - \gamma\lambda_i)^{N-t} }\rbr{ 1 - (1 - \gamma \lambda_i)^{2t} }\notag\\
    &\ge \frac{ \sigma_{\textrm{noise}}^2 }{4 N^2}\sum_{i} \sum_{t=0}^{N-1} \rbr{ 1 - (1 - \gamma\lambda_i)^{N-t-1} }\rbr{ 1 - (1 - \gamma \lambda_i)^{t} },
\end{align*}
where $\{\lambda_i\}_{i\ge1}$ are the eigenvalues of $\Hb$ and are sorted in decreasing order.
Define 
\[ f(x):= \sum_{t=0}^{N-1} \rbr{ 1 - (1 - x)^{N-t-1} }\rbr{ 1 - (1 - x)^{t} }, \qquad 0< x <1,\]
then 
\[
\var \ge \frac{ \sigma_{\textrm{noise}}^2 }{4 N^2} \sum_{i\ge1} f(\gamma \lambda_i).
\]
Clearly $f(x)$ is increasing for $0<x<1$.
Moreover:
\begin{align*}
f(x) &= \sum_{t=0}^{N-1} \rbr{ 1-(1-x)^{N-1-t}-(1-x)^t + (1-x)^{N-1}}\notag\\
& = N - 2\frac{1-(1-x)^N}{x}+N(1-x)^{N-1}.
\end{align*}
Next we lower bound $f(x)$ within the range $ \frac{1}{N} < x < 1$ and $0 < x<\frac{1}{N}$, respectively.

First consider $\frac{1}{N} \le x < 1$. Notice that $f(x)$ is increasing and $\rbr{1-\frac{1}{N}}^{N} \ge \rbr{1-\frac{1}{500}}^{500}> 1.1/3$ if $N\ge 500$, thus for $\frac{1}{N} \le x < 1$, we have
\begin{align*}
f(x)\ge N - 2N + 3N\cdot (1-1/N)^N\ge 0.1N.
\end{align*}
On the other hand, note that we have the fourth-order derivative of $f(x)$ is positive when $x\in(0,1/N)$, thus for $0\le x\le 1/N$, we can perform third-order Taylor expansion on $f(x)$ at $x=0$, which gives
\begin{align*}
f(x) &\ge \frac{N(N-1)(N-2)x^2}{6}-\frac{N(N-1)(N-2)(N-3)x^3}{12}\notag\\
&\ge \frac{N(N-1)(N-2)x^2}{12}\qquad \text{(since $x\le 1/N$)}\notag\\
&\ge \frac{2N^3x^2}{25}.\qquad \text{(since $N\ge 500$)}
\end{align*}
In sum, 
\[
f(x) \ge 
\begin{cases}
\frac{N}{10}, &\frac{1}{N} \le x < 1,\\
\frac{2N^3}{25} x^2, &0 <  x < \frac{1}{N}.
\end{cases}
\]
Set $k^* = \max \{k: \lambda_k \ge \frac{1}{N\gamma }\}$,
then
\begin{align*}
    \var 
    &\ge \frac{ \sigma_{\textrm{noise}}^2 }{4 N^2} \sum_{i}  f(\gamma\lambda_i) \\
    &\ge \frac{ \sigma_{\textrm{noise}}^2 }{4 N^2} \rbr{\frac{ N k^*}{10} + \frac{ 2N^3}{25}\gamma^2 \cdot \sum_{i>k^*}\lambda_i^2  } \\
    &\ge \frac{ \sigma_{\textrm{noise}}^2 }{50} \rbr{\frac{k^*}{N} + N\gamma^2 \cdot \sum_{i>k^*}\lambda_i^2  }.
\end{align*}
This completes the proof.
\end{proof}


\subsection{Lower Bounding the Bias Error}
Recall that we have the following lower bound on the bias error
\begin{align*}
    \bias
\ge \frac{1}{2N^2}\sum_{t=0}^{N-1}\sum_{k=t}^{N-1}\big\la(\Ib-\gamma\Hb)^{k-t}\Hb,\Bb_t\big\ra,
\end{align*}
from which we notice that
\begin{align}
\bias
&\ge \frac{1}{2N^2}\sum_{t=0}^{N-1}\sum_{k=t}^{N-1}\big\la(\Ib-\gamma\Hb)^{k-t}\Hb,\Bb_t\big\ra 
= \frac{1}{2\gamma N^2}\sum_{t=0}^{N-1}\big\la\Ib - (\Ib-\gamma\Hb)^{N-t}, \Bb_t\big\ra \notag \\
&\ge \frac{1}{2\gamma N^2}\sum_{t=0}^{N/2}\big\la\Ib - (\Ib-\gamma\Hb)^{N-t}, \Bb_t\big\ra \notag \\
&\ge \frac{1}{2\gamma N^2}\big\la\Ib - (\Ib-\gamma\Hb)^{N/2}, \sum_{t=0}^{N/2}\Bb_t\big\ra. \label{eq:bias-lowerbound}
\end{align}
Let $\Sbb_n := \sum_{t=0}^{n-1} \Bb_t$.
Then the reminding challenge is to lower bound $\Sbb_{N/2+1} = \sum_{t=0}^{N/2}\Bb_t$.
Similarly to the idea of proving the upper bound, we first establish a crude lower bound on $\Sbb_n$ then improve it to a fine lower bound.

\begin{lemma}\label{lemma:lowerbound_St}
Suppose Assumptions \ref{assump:second_moment} and \ref{assumption:lowerbound_fourthmoment} hold.
If the stepsize satisfies $\gamma< 1/\lambda_1$, then for any $n \ge 2$, it holds that
\begin{align*}
\Sbb_n \succeq \frac{\beta}{4}\tr\rbr{  \rbr{\Ib - (\Ib - \gamma \Hb)^{n/2}}  \Bb_0 } \cdot \rbr{\Ib - (\Ib - \gamma \Hb)^{n/2}} + \sum_{t=0}^{n-1} (\Ib - \gamma \Hb)^t \cdot \Bb_0 \cdot (\Ib - \gamma \Hb)^t.
\end{align*}
\end{lemma}
\begin{proof}
We first build a crude bound for $\Sbb_n$.
Recall that $\tilde{\cT} - \cT$ is a PSD mapping by Lemma \ref{lemma:operators}, then
\begin{align*}
    \Sbb_n 
    = \sum_{t=0}^{n-1} \Bb_t = \sum_{t=0}^{n-1} (\cI - \gamma \cT)^t \circ \Bb_0 
    \succeq \sum_{t=0}^{n-1} (\cI - \gamma \tilde{\cT})^t \circ \Bb_0 
    = \sum_{t=0}^{n-1} (\Ib - \gamma \Hb)^t \cdot \Bb_0 \cdot (\Ib - \gamma \Hb)^t.
\end{align*}
Now we apply Assumption \ref{assumption:lowerbound_fourthmoment} with the above crude bound to obtain that
\begin{align*}
    (\cM - \tilde{\cM}) \circ \Sbb_n
    &\succeq \beta \tr\rbr{\Hb \Sbb_n } \Hb \\
    &\succeq \beta \tr\rbr{  \sum_{t=0}^{n-1} (\Ib - \gamma \Hb)^{2t} \Hb \cdot \Bb_0 }\Hb \\
    &\succeq \beta \tr\rbr{  \sum_{t=0}^{n-1} (\Ib - 2\gamma \Hb)^{t} \Hb \cdot \Bb_0 }\Hb \\
    &= \frac{\beta}{2\gamma} \tr\rbr{  \rbr{\Ib - (\Ib - 2\gamma \Hb)^{n}}  \Bb_0 }\Hb \\
    &\succeq \frac{\beta}{2\gamma} \tr\rbr{  \rbr{\Ib - (\Ib - \gamma \Hb)^{n}} \Bb_0 }\Hb.
\end{align*}
Next we use the above inequality to build a fine lower bound for $\Sbb_n$:
\begin{align*}
    \Sbb_n  &= (\cI - \gamma \cT) \circ \Sbb_{n-1} + \Bb_0 
    = (\cI - \gamma \tilde{\cT} ) \circ \Sbb_{n-1} + \gamma^2 (\cM - \tilde{\cM}) \circ \Sbb_{n-1} + \Bb_0 \\
    &\succeq (\cI - \gamma \tilde{\cT} ) \circ \Sbb_{n-1} + \frac{\beta \gamma}{2} \tr\rbr{  \rbr{\Ib - (\Ib - \gamma \Hb)^{n-1}}  \Bb_0 }\Hb + \Bb_0.
\end{align*}
Solving the recursion we obtain 
\begin{align*}
    \Sbb_n 
    &\succeq \sum_{t=0}^{n-1} (\cI - \gamma \tilde{\cT})^{t} \circ \cbr{ \frac{\beta \gamma}{2} \tr\rbr{  \rbr{\Ib - (\Ib - \gamma \Hb)^{n-1-t}}  \Bb_0 }\Hb + \Bb_0  } \\
    &= \frac{\beta \gamma}{2} \sum_{t=0}^{n-1}  \tr\rbr{  \rbr{\Ib - (\Ib - \gamma \Hb)^{n-1-t}}  \Bb_0 } \cdot (\Ib - \gamma \Hb)^{2t} \Hb \\
    &\qquad + \sum_{t=0}^{n-1} (\Ib - \gamma \Hb)^t \cdot \Bb_0 \cdot (\Ib - \gamma \Hb)^t.
\end{align*}
For the first term, noticing the following:
\begin{align*}
    &\ \sum_{t=0}^{n-1}  \tr\rbr{  \rbr{\Ib - (\Ib - \gamma \Hb)^{n-1-t}}  \Bb_0 } \cdot (\Ib - \gamma \Hb)^{2t} \Hb \\
    &\succeq \sum_{t=0}^{n-1}  \tr\rbr{  \rbr{\Ib - (\Ib - \gamma \Hb)^{n-1-t}}  \Bb_0 } \cdot (\Ib - 2\gamma \Hb)^{t} \Hb \\
    &\succeq \sum_{t=0}^{n/2-1}  \tr\rbr{  \rbr{\Ib - (\Ib - \gamma \Hb)^{n-1-t}}  \Bb_0 } \cdot (\Ib - 2\gamma \Hb)^{t} \Hb \\
    &\succeq \tr\rbr{  \rbr{\Ib - (\Ib - \gamma \Hb)^{n/2}}  \Bb_0 } \cdot \sum_{t=0}^{n/2-1}   (\Ib - 2\gamma \Hb)^{t} \Hb \\
    &= \frac{1}{2\gamma} \tr\rbr{  \rbr{\Ib - (\Ib - \gamma \Hb)^{n/2}}  \Bb_0 } \cdot \rbr{\Ib - (\Ib - 2\gamma \Hb)^{n/2}} \\
    &\succeq \frac{1}{2\gamma} \tr\rbr{  \rbr{\Ib - (\Ib - \gamma \Hb)^{n/2}}  \Bb_0 } \cdot \rbr{\Ib - (\Ib - \gamma \Hb)^{n/2}},
\end{align*}
inserting which back to the lower bound for $\Sbb_n$, we complete the proof.

\end{proof}

\begin{lemma}\label{lemma:lowerbound_bias_final}
Suppose Assumptions \ref{assump:second_moment} and \ref{assumption:lowerbound_fourthmoment} hold and $N\ge 2$.
If the stepsize satisfies $\gamma< 1/\gamma_1$, then
\begin{align*}
\bias
&\ge\frac{1}{100 \gamma^2N^2}\cdot\|\wb_0-\wb^*\|^2_{\Hb_{0:k^*}^{-1}} + \frac{1}{100}\cdot \|\wb_0-\wb^*\|^2_{\Hb_{k^*:\infty}}\notag\\
&\qquad+\frac{\beta \rbr{ \|\wb_0-\wb^*\|^2_{\Ib_{0:k^*} } + \gamma N  \|\wb_0-\wb^*\|^2_{\Hb_{k^*:\infty}}   }}{1000\gamma N^2 }\cdot \rbr{ k^* + \gamma^2 N^2 \sum_{i>k^*} \lambda_i^2},
\end{align*}
where $k^* = \max \{k: \lambda_k \ge \frac{1}{N \gamma}\}$.
\end{lemma}
\begin{proof}
According to \eqref{eq:bias-lowerbound} and Lemma \ref{lemma:lowerbound_St}, we have that
\begin{align*}
\bias 
&\ge  \frac{1}{2\gamma N^2}\big\la\Ib - (\Ib-\gamma\Hb)^{N/2}, \Sbb_{N/2+1}\big\ra 
\ge  \frac{1}{2\gamma N^2} \big\la\Ib - (\Ib-\gamma\Hb)^{N/2}, \Sbb_{N/2}\big\ra  \\
&\ge \underbrace{\frac{\beta}{8\gamma N^2}\tr\rbr{  \rbr{\Ib - (\Ib - \gamma \Hb)^{N/4}}  \Bb_0 } \cdot  \big\la\Ib - (\Ib-\gamma\Hb)^{N/2}, \Ib - (\Ib - \gamma \Hb)^{N/4} \ra }_{I_1} \\
&\quad + \underbrace{\frac{1}{2\gamma N^2} \big\la\Ib - (\Ib-\gamma\Hb)^{N/2}, \sum_{t=0}^{N/2-1} (\Ib - \gamma \Hb)^t \cdot \Bb_0 \cdot (\Ib - \gamma \Hb)^t \big\ra }_{I_2}.
\end{align*}
The first term is lower bounded by 
\begin{align*}
    I_1
    &\ge {\frac{\beta}{8\gamma N^2}\tr\rbr{  \rbr{\Ib - (\Ib - \gamma \Hb)^{N/4}}  \Bb_0 } \cdot  \tr\rbr{ \rbr{ \Ib - (\Ib - \gamma \Hb)^{N/4} }^2 } } \\
    &= {\frac{\beta}{8\gamma N^2} \rbr{ \sum_{i}\rbr{1- (1-\gamma \lambda_i)^{N/4}} \omega_i^2 } \cdot  \rbr{ \sum_{i}\rbr{1 - (1-\gamma \lambda_i)^{N/4}  }^2 } },
\end{align*}
where $\omega_i = \vb_i^\top(\wb_0-\wb^*)$ for $\vb_1,\dots,\vb_d$ being the eigenvectors of $\Hb$;
and the second term is lower bounded by
\begin{align*}
    I_2
    & = {\frac{1}{2\gamma N^2} \la \sum_{t=0}^{N/2-1} (\Ib - \gamma \Hb)^{2t} \rbr{ \Ib - (\Ib-\gamma\Hb)^{N/2} }, \Bb_0 \ra }\\
    & \ge {\frac{1}{2\gamma N^2} \la \sum_{t=0}^{N/2-1} (\Ib - 2\gamma \Hb)^{t} \rbr{ \Ib - (\Ib-\gamma\Hb)^{N/2} }, \Bb_0 \ra }\\
    &\ge {\frac{1}{4 \gamma^2 N^2} \la \rbr{ \Ib - (\Ib-\gamma\Hb)^{N/2} }^2 \Hb^{-1}, \Bb_0 \ra } \\
    &\ge {\frac{1}{4 \gamma^2 N^2} \la \rbr{ \Ib - (\Ib-\gamma\Hb)^{N/4} }^2 \Hb^{-1}, \Bb_0 \ra } \\
    &=  \frac{1}{4 \gamma^2 N^2} \sum_{i} \rbr{1 - (1-\gamma \lambda_i)^{N/4}  }^2\lambda_i^{-1} \omega_i^2.
\end{align*}

To further lower bound the two terms, noticing the following inequality:
\begin{align*}
    1-(1-\gamma \lambda_i)^{\frac{N}{4}} \ge 
    \begin{cases}
    1-(1-\frac{1}{N})^{\frac{N}{4}} \ge 1-e^{-\frac{1}{4}} \ge \frac{1}{5}, & \lambda_i \ge \frac{1}{\gamma N}, \\
    \frac{N}{4}\cdot \gamma \lambda_i -  \frac{N(N-4)}{32}\cdot\gamma^2\lambda_i^2 \ge \frac{N}{5} \cdot \gamma \lambda_i, & \lambda_i < \frac{1}{\gamma N}.
    \end{cases}
\end{align*}
Plugging this into the bounds for $I_1$ and $I_2$, and setting $k^* := \max \{ k : \lambda_k \ge 1/(\gamma N)\}$, we then obtain that
\begin{align*}
I_1 
&\ge \frac{\beta}{8 \gamma N^2} \cdot \rbr{ \frac{1}{5} \cdot \sum_{i\le k^*}  \omega_i^2 + \frac{\gamma N}{5}\sum_{i > k^*}\lambda_i \omega_i^2 } \cdot \rbr{ \frac{1}{25} \cdot k^* + \frac{\gamma^2 N^2}{25}\cdot \sum_{i > k^*}\lambda_i^2 }  \\
& = \frac{\beta}{1000\gamma N^2}\cdot \rbr{ \nbr{\wb_0 - \wb^*}^2_{\Ib_{0:k^*}} + \gamma N \nbr{\wb_0 - \wb^*}^2_{\Hb_{k^*:\infty}} } \cdot \rbr{k^* + \gamma^2 N^2 \sum_{i > k^*}\lambda_i^2 },
\end{align*}
and that 
\begin{align*}
I_2 
&\ge \frac{1}{4\gamma^2 N^2} \rbr{ \frac{1}{25} \cdot \sum_{i \le k^*} \lambda_i^{-1} \omega_i^2 + \frac{\gamma^2 N^2}{25} \cdot \sum_{i > k^*} \lambda_i \omega_i^2 } \\
&= \frac{1}{100 \gamma^2 N^2} \rbr{ \nbr{\wb_0 - \wb^*}^2_{\Hb^{-1}_{0:k^*}} + \gamma^2 N^2 \nbr{\wb_0 - \wb^*}^2_{\Hb_{k^*:\infty}} }.
\end{align*}
Summing up the two terms completes the proof.

\end{proof}

\subsection{Proof of Theorem \ref{thm:generalization_error_lowerbound}}\label{append}
\begin{proof}
Plugging the bounds of the bias error and variance error in Lemmas \ref{lemma:lowerbound_bias_final} and \ref{lemma:lowerbound_var} into Lemma \ref{lemma:lower_bound_decomp} immediately completes the proof.
\end{proof}

\section{Proofs for Tail-Averaging}\label{append-sec:tail-average}
In this section, we provide the proofs for SGD with tail-averaging. Recall that in tail-averaging, we take average from the $s$-th iterate, i.e., the output of the tail-average SGD is
\begin{align*}
\overline \wb_{s:s+N} = \frac{1}{N}\sum_{t=s}^{s+N-1}\wb_t.
\end{align*}

\subsection{Upper Bounds for Tail-Averaging}

The following two lemmas are straightforward extensions of Lemmas~\ref{lemma:bias_var_decomposition} and  \ref{lemma:bias_var_decomposition_bound}.

\begin{lemma}[Variant of Lemma \ref{lemma:bias_var_decomposition}]\label{lemma:tail_bias_var_decomposition}
\begin{align*}
\EE [L(\overline{\wb}_{s:s+N})] - L(\wb^*) = \frac{1}{2}\la\Hb,\EE[\bar\betab_{s:s+N}\otimes \bar\betab_{s:s+N}]\ra\le \rbr{ \sqrt{\bias} + \sqrt{\var} }^2,
\end{align*}
where 
\[
\bias := \frac{1}{2} \langle \Hb, \EE[{\bar\betab}^{\bias}_{s:s+N} \otimes {\bar\betab}^{\bias}_{s:s+N}] \rangle, \qquad 
\var := \frac{1}{2} \langle \Hb, \EE[{\bar\betab}^{\var}_{s:s+N} \otimes {\bar\betab}^{\var}_{s:s+N}] \rangle.
\]
\end{lemma}

\begin{lemma}[Variant of Lemma \ref{lemma:bias_var_decomposition_bound}]\label{lemma:tail_bias_var_decomposition_bound}
Recall iterates \eqref{eq:update_Bt} and \eqref{eq:update_Ct}. 
If the stepsize satisfies $\gamma < 1/\lambda_1$, the bias error and variance error are upper bounded respectively as follows:
\begin{gather*}
    \bias :=  \half \langle \Hb, \EE[{\bar\betab}^{\bias}_{s:s+N} \otimes {\bar\betab}^{\bias}_{s:s+N}] \rangle 
    \le \frac{1}{N^2}\sum_{t=0}^{N-1}\sum_{k=t}^{N-1}\big\la (\Ib-\gamma\Hb)^{k-t}\Hb,\Bb_{s+t}\big\ra, \\
    \var := \half \langle \Hb, \EE[{\bar\betab}^{\var}_{s:s+N} \otimes {\bar\betab}^{\var}_{s:s+N}] 
    \le \frac{1}{N^2}\sum_{t=0}^{N-1}\sum_{k=t}^{N-1}\big\la (\Ib-\gamma\Hb)^{k-t}\Hb,\Cb_{s+t}\big\ra.
\end{gather*}
\end{lemma}
\begin{proof}
By replacing $\Bb_0$ and $\Cb_0$ by $\Bb_s$ and $\Cb_s$ in the proof of Lemma \ref{lemma:bias_var_decomposition_bound}, and repeating the remaining arguments, we can easily complete the proof.
\end{proof}

\subsubsection{Bounding the Variance Error}

\begin{lemma}[Variant of Lemma \ref{lemma:upperbound_var}]\label{lemma:tail_upperbound_var}
Under Assumptions \ref{assump:second_moment}, \ref{assump:noise} and \ref{assump:R2}, if the stepsize satisfies $\gamma < 1/R^2$, then it holds that
\begin{equation*}
    \var \le \frac{ \sigma^2}{ 1-\gamma R^2} \cdot\bigg(\frac{k^*}{N} + \gamma\cdot \sum_{k^*< i\le k^\dagger}\lambda_i + \gamma^2(s+N)\cdot\sum_{i>k^\dagger}\lambda_i^2\bigg),
\end{equation*}
where $k^* = \min\{k: \lambda_i<\frac{1}{\gamma N}\}$ and $k^\dagger = \min\{k: \lambda_i<\frac{1}{\gamma(s+N)}\}$.
\end{lemma}
\begin{proof}
By Lemma \ref{lemma:tail_bias_var_decomposition_bound},
we can bound the variance error as follows
\begin{align*}
   \var 
    &\le \frac{1}{N^2}\sum_{t=0}^{N-1}\sum_{k=t}^{N-1}\big\la(\Ib-\gamma\Hb)^{k-t}\Hb,\Cb_{s+t}\big\ra \\
    &= \frac{1}{\gamma N^2} \sum_{t=0}^{N-1} \big\la \Ib - (\Ib - \gamma\Hb)^{N-t} ,\Cb_{s+t} \big\ra \\
    &\le \frac{ \sigma^2}{ N^2 (1-\gamma R^2)} \sum_{t=0}^{N-1} \big\la \Ib - (\Ib - \gamma\Hb)^{N-t} , \rbr{\Ib - (\Ib - \gamma \Hb)^{s+t}}  \big \ra \\
    &= \frac{ \sigma^2}{N^2 (1-\gamma R^2)} \sum_{i} \sum_{t=0}^{N-1} \rbr{ 1 - (1 - \gamma\lambda_i)^{N-t} }\rbr{ 1 - (1 - \gamma \lambda_i)^{s+t} } \notag\\
    & \le \frac{ \sigma^2}{N^2 (1-\gamma R^2)} \sum_{i} \sum_{t=0}^{N-1} \rbr{ 1 - (1 - \gamma\lambda_i)^N }\rbr{ 1 - (1 - \gamma \lambda_i)^{s+N} }\notag\\
    & = \frac{ \sigma^2}{N (1-\gamma R^2)}\sum_{i}\rbr{ 1 - (1 - \gamma \lambda_i)^N }\rbr{ 1 - (1 - \gamma \lambda_i)^{s+N} },
\end{align*}
where the second inequality is due to Lemma \ref{lemma:upper_bound_phit}, $\{\lambda_i\}_{i\geq 1}$ are the eigenvalues of $\Hb$ and are sorted in decreasing order.
Now we will move to upper bound the quantity $\rbr{1-(1-\gamma\lambda_i)^{N}} \rbr{1-(1-\gamma\lambda_i)^{s+N}}$, which will be separately discussed according to the following three cases: (1) $\gamma\lambda_i\ge1/N$, (2) $1/(s+N)\le\gamma\lambda_i<1/N$, and (3) $\gamma\lambda<1/(s+N)$.
In case (1), we can crudely bound this quantity as follows,
\begin{align*}
\rbr{1-(1-\gamma\lambda_i)^{N}} \rbr{1-(1-\gamma\lambda_i)^{s+N}} \le 1.
\end{align*}
In case (2), we can use $(1-\gamma\lambda_i)^N\ge 1-\gamma N\lambda_i$ and get
\begin{align*}
\rbr{1-(1-\gamma\lambda_i)^{N}} \rbr{1-(1-\gamma\lambda_i)^{s+N}}\le \gamma N\lambda_i\cdot1 = \gamma N\lambda_i.
\end{align*}
In case (3), we can use $(1-\gamma\lambda_i)^N\ge 1-\gamma N\lambda_i$ and  $(1-\gamma\lambda_i)^{s+N}\ge 1-\gamma (s+N)\lambda_i$, and get
\begin{align*}
\rbr{1-(1-\gamma\lambda_i)^{N}} \rbr{1-(1-\gamma\lambda_i)^{s+N}}\le \gamma N\lambda_i\cdot \gamma (s+N)\lambda_i = \gamma^2 N(s+N)\lambda_i^2.
\end{align*}
Therefore, set $k^* = \min\{k: \lambda_i<\frac{1}{N\gamma}\}$ and $k^\dagger = \min\{k: \lambda_i<\frac{1}{(s+N)\gamma}\}$, we have
\begin{align*}
\var &\le \frac{ \sigma^2}{ N (1-\gamma R^2)} \cdot\bigg(k^* + \gamma N\sum_{k^*< i\le  k^\dagger}\lambda_i + \gamma^2N(s+N)\sum_{i>k^\dagger}\lambda_i^2\bigg)\\
&=\frac{ \sigma^2}{ 1-\gamma R^2} \cdot\bigg(\frac{k^*}{N} + \gamma\cdot \sum_{k^*< i\le k^\dagger}\lambda_i + \gamma^2(s+N)\cdot\sum_{i>k^\dagger}\lambda_i^2\bigg).
\end{align*}
This completes the proof.

\end{proof}

\subsubsection{Bounding the Bias Error}
Similarly to \eqref{eq:bias-upperbound-0} and using Lemma \ref{lemma:tail_bias_var_decomposition_bound}, we have the following upper bound for the bias error:
\begin{align}
\bias &\le \frac{1}{N^2}\sum_{t=0}^{N-1}\sum_{k=t}^{N-1}\big\la (\Ib-\gamma\Hb)^{k-t}\Hb,\Bb_{s+t}\big\ra\notag\\
&= \frac{1}{\gamma N^2}\sum_{t=0}^{N-1}\big\la \Ib - (\Ib-\gamma\Hb)^{N-t},\Bb_{s+t}\big\ra\notag\\
&\le \frac{1}{\gamma N^2}\left\la\Ib - (\Ib-\gamma\Hb)^N, \sum_{t=0}^{N-1}\Bb_{s+t}\right\ra. \label{eq:tail-bias-upperbound-0}
\end{align}
Let $\Sbb_{s:s+t} = \sum_{k=s}^{s+t-1}\Bb_k$, then we only need to establish an upper bound for $\Sbb_{s:s+N}$.
\begin{lemma}[Variant of Lemma \ref{lemma:upperbound_St}]\label{lemma:tail_bound_St}
Let $\Sbb_{s:s+t} = \sum_{k=s}^{s+t-1}\Bb_k$ for any $t\ge s$ and $\Bb_{a,b} = \Bb_a - (\Ib-\gamma\Hb)^{b-a}\Bb_a(\Ib-\gamma\Hb)^{b-a}$. Under Assumptions \ref{assump:second_moment} and \ref{assump:bound_fourthmoment}, if the stepsize satisfies $\gamma < 1/\big(\alpha\tr(\Hb)\big)$, it holds that
\begin{align*}
\Sbb_{s:s+N}\preceq \sum_{k=0}^{N-1}(\Ib-\gamma\Hb)^{k+s}\Bb_0(\Ib-\gamma\Hb)^{k+s} + \frac{\gamma\alpha\tr(\Bb_{s,s+N}+\Bb_{0,s})}{1-\gamma\alpha\tr(\Hb)}\sum_{k=0}^{N-1}(\Ib-\gamma\Hb)^{2k}\Hb.
\end{align*}
\end{lemma}
\begin{proof}
Based on the definition of $\Sbb_{s:s+t}$, we have
\begin{align*}
\Sbb_{s:s+t} = \sum_{k=s}^{s+t-1}\Bb_k = \sum_{k=0}^{t-1}(\cI-\gamma\cT)^k\circ\Bb_s = (\cI-\gamma\cT)\circ\Sbb_{s:s+t-1} + \Bb_s.
\end{align*}
Therefore, following the similar proof technique of Lemma \ref{lemma:upperbound_St},  
we can get
\begin{align}\label{eq:upperbound_SsN}
\Sbb_{s:s+N}\preceq \underbrace{\sum_{k=0}^{N-1}(\Ib-\gamma\Hb)^k\Bb_s(\Ib-\gamma\Hb)^k}_{I_1} + \underbrace{\frac{\gamma \alpha \tr(\Bb_{s,s+N})}{1-\gamma\alpha\tr(\Hb)}\sum_{k=0}^{N-1}(\Ib-\gamma\Hb)^{2k}\Hb}_{I_2}.
\end{align}
Now we will upper bound $I_1$, which requires a carefully characterization on $\Bb_s$. Particularly, the update form of $\Bb_k$ in \eqref{eq:bias_iterates} implies
\begin{align*}
\Bb_k = (\cI - \gamma\cT)\circ\Bb_{k-1} \preceq (\cI - \gamma\tilde\cT)\circ\Bb_{k-1} + \gamma^2\cM\circ\Bb_{k-1}.
\end{align*}
By Assumption \ref{assump:bound_fourthmoment}, we have $\cM\circ\Bb_k \preceq \alpha\tr(\Hb\Bb_k)\cdot\Hb$. Thus,
\begin{align}\label{eq:upperbound_Bk}
\Bb_k &\preceq (\cI - \gamma\tilde\cT)\circ\Bb_{k-1} + \gamma^2\cM\circ\Bb_{k-1}\notag\\
&\preceq (\cI - \gamma\tilde\cT)\circ\Bb_{k-1} + \alpha\gamma^2 \tr(\Hb\Bb_{k-1})\cdot\Hb\notag\\
&= (\cI - \gamma\tilde\cT)^k\circ\Bb_0 + \alpha \gamma^2\sum_{t=0}^{k-1}\tr(\Hb\Bb_t)\cdot(\cI-\gamma\tilde \cT)^{k-1-t}\circ\Hb\notag\\
&\preceq (\cI - \gamma\tilde\cT)^k\circ\Bb_0 + \alpha \gamma^2\sum_{t=0}^{k-1}\tr(\Hb\Bb_t)\cdot\Hb
\end{align}
where in the third inequality we use the fact that $\cI-\gamma\tilde \cT$ is a PSD mapping and the last inequality is due to $(\cI-\gamma\tilde \cT)^{k-1-t}\Hb = (\Ib-\gamma\Hb)^{2(k-1-t)}\Hb\preceq \Hb$.
Next we will upper bound $\sum_{t=0}^{k-1}\tr(\Hb\Bb_t)$. Recall the definition of $\betab_k^{\bias}$ and its update rule, we have 
\begin{align*}
\EE[\|\betab_k^{\bias}\|_2^2|\betab_{k-1}^{\bias}] &= \EE[\|(\Ib-\gamma\xb_k\xb_k^\top)\betab_{k-1}^{\bias}\|_2^2|\betab_{k-1}^{\bias}]\notag\\
& = \|\betab_{k-1}^{\bias}\|_2^2  - 2\gamma \EE[\la\xb_k\xb_k^\top,\betab_{k-1}^{\bias}\otimes\betab_{k-1}^{\bias} \ra|\betab_{k-1}^{\bias}] + \gamma^2 \EE[\la\xb_k\xb_k^\top\xb_k\xb_k^\top,\betab_{k-1}^{\bias}\otimes\betab_{k-1}^{\bias}|\betab_{k-1}^{\bias}]\notag\\
& = \|\betab_{k-1}^{\bias}\|_2^2  - 2\gamma \la\Hb,\betab_{k-1}^{\bias}\otimes\betab_{k-1}^{\bias} \ra + \gamma^2 \la\cM\circ\Ib,\betab_{k-1}^{\bias}\otimes\betab_{k-1}^{\bias} \ra\notag\\
&\le \|\betab_{k-1}^{\bias}\|_2^2  - \big(2\gamma-\gamma^2\alpha\tr(\Hb)\big)\cdot \la\Hb,\betab_{k-1}^{\bias}\otimes\betab_{k-1}^{\bias} \ra,
\end{align*}
where the inequality is due to the fact that $\cM\circ\Ib\preceq\alpha\tr(\Hb)\Hb$. Note that $\Bb_k = \EE[\betab_k^{\bias}\otimes\betab_k^{\bias}]$, taking total expectation further gives
\begin{align*}
\tr(\Bb_k)\le \tr(\Bb_{k-1}) - \big(2\gamma-\gamma^2\alpha\tr(\Hb)\big)\cdot\tr(\Hb\Bb_{k-1}),
\end{align*}
which implies that
\begin{align}\label{eq:upperbound_sum_tr_HB}
\sum_{t=0}^{k-1}\tr(\Hb\Bb_t)\le \frac{\tr(\Bb_0) - \tr(\Bb_k)}{2\gamma-\gamma^2\alpha\tr(\Hb)}.
\end{align}
Substituting \eqref{eq:upperbound_sum_tr_HB} into \eqref{eq:upperbound_Bk} gives
\begin{align*}
\Bb_k &\preceq (\cI - \gamma\tilde\cT)^k\circ\Bb_0 + \alpha \gamma^2\sum_{t=0}^{k-1}\tr(\Hb\Bb_t)\cdot\Hb\notag\\
&\preceq (\cI - \gamma\tilde\cT)^k\circ\Bb_0 + \frac{\gamma\alpha\tr(\Bb_0-\Bb_k)}{2-\gamma\alpha\tr(\Hb)}\cdot\Hb.
\end{align*}
Therefore, we further have
\begin{align}\label{eq:upperbound_SsN_term1} I_1 \preceq \sum_{k=0}^{N-1}(\Ib-\gamma\Hb)^{k+s}\Bb_0(\Ib-\gamma\Hb)^{k+s} + \frac{\gamma\alpha\tr(\Bb_0-\Bb_s)}{2-\gamma\alpha\tr(\Hb)}\sum_{k=0}^{N-1}(\Ib-\gamma\Hb)^{2k}\Hb.
\end{align}
Further note that $\Bb_s = (\cI-\gamma\cT)^s\Bb_0$ and $\cT\succeq \tilde\cT$, we have
\begin{align*}
\tr(\Bb_0-\Bb_s) &= \tr\big(\Bb_0-(\cI-\gamma\cT)^s\Bb_0\big)\notag\\
&\le \tr\big(\Bb_0-(\cI-\gamma\tilde\cT)^s\Bb_0\big)\notag\\
&\le \tr\big(\Bb_0-(\Ib-\gamma\Hb)^s\Bb_0(\Ib-\gamma\Hb)^s\big)\notag\\
&=\tr(\Bb_{0,s}).
\end{align*}
Now, we can substitute the above inequality and \eqref{eq:upperbound_SsN_term1} into \eqref{eq:upperbound_SsN} and obtain the following upper bound on $\Sbb_{s:s+N}$,
\begin{align*}
\Sbb_{s:s+N}\preceq I_1+I_2\preceq \sum_{k=0}^{N-1}(\Ib-\gamma\Hb)^{k+s}\Bb_0(\Ib-\gamma\Hb)^{k+s} + \frac{\gamma\alpha\tr(\Bb_{s,s+N}+\Bb_{0,s})}{1-\gamma\alpha\tr(\Hb)}\sum_{k=0}^{N-1}(\Ib-\gamma\Hb)^{2k}\Hb,
\end{align*}
where we use the fact that $0\le1-\gamma\alpha\tr(\Hb)\le 2-\gamma\alpha\tr(\Hb)$. This completes the proof.
\end{proof}

\begin{lemma}[Variant of Lemma \ref{lemma:bound_bias_final}]\label{lemma:bound_bias_final_tail} 
Under Assumptions \ref{assump:second_moment} and \ref{assump:bound_fourthmoment}, if the stepsize satisfies $\gamma < 1/(\alpha\tr(\Hb))$, it holds that
\begin{align*}
\bias &\le \frac{1}{\gamma^2N^2}\cdot\big\|(\Ib-\gamma\Hb)^s(\wb_0-\wb^*)\big\|_{\Hb_{0:k^*}^{-1}}^2 + \big\|(\Ib-\gamma\Hb)^s(\wb_0-\wb^*)\big\|_{\Hb_{k^*:\infty}}^2\notag\\
&\qquad + \frac{4\alpha \big(\|\wb_0-\wb^*\|_{\Ib_{0:k^*}}+(s+N)\gamma\|\wb_0-\wb^*\|_{\Hb_{k^*:\infty}}^2\big)}{\gamma(1-\gamma \alpha\tr(\Hb))}\cdot\bigg(\frac{k^*}{N^2} + \gamma^2 \sum_{i> k^*}\lambda_i^2\bigg),
\end{align*}
where $k^* = \max \{k: \lambda_k \ge \frac{1}{\gamma N}\}$.
\end{lemma}
\begin{proof}
Substituting the upper bound of $\Sbb_{s:s+N}$ into \eqref{eq:tail-bias-upperbound-0}, we can get
\begin{align}\label{eq:tailaverage_biasupperbound_I1_temp}
\bias&\le \underbrace{\frac{\alpha\tr(\Bb_{s,s+N}+\Bb_{0,s})}{ N^2(1-\gamma\alpha\tr(\Hb))}\sum_{k=0}^{N-1}\left\la\Ib - (\Ib-\gamma\Hb)^N,(\Ib-\gamma\Hb)^{2k}\Hb\right\ra}_{I_1}\notag\\
&\qquad +\underbrace{\frac{1}{\gamma N^2}\sum_{k=0}^{N-1}\left\la\Ib - (\Ib-\gamma\Hb)^N, (\Ib-\gamma\Hb)^{k+s}\Bb_0(\Ib-\gamma\Hb)^{k+s}\right\ra}_{I_2}.
\end{align}
By  \eqref{eq:bound_bias_I1}, we can get the following bound on $I_1$,
\begin{align}\label{eq:bound_bias_I1_tail}
I_1\le \frac{\alpha \tr(\Bb_{s,s+N} + \Bb_{0,s})}{\gamma(1-\gamma \alpha\tr(\Hb))}\cdot\bigg(\frac{k^*}{N^2} + \gamma^2 \sum_{i> k^*}\lambda_i^2\bigg).
\end{align}
Then following the same procedure in \eqref{eq:bound_traceB0N}, we have
\begin{align*}
\tr(\Bb_{s,s+N} + \Bb_{0,s}) &\le 2\tr(\Bb_{0,s+N}) \le 4\big(\|\wb_0-\wb^*\|_{\Ib_{0:k^*}}+(s+N)\gamma\|\wb_0-\wb^*\|_{\Hb_{k^*:\infty}}^2\big)
\end{align*}
where $k^* = \max\big\{k: \lambda_k\ge\frac{1}{\gamma N} \big\}$ (in fact $k^*$ can be arbitrary choosen). Plugging this into \eqref{eq:tailaverage_biasupperbound_I1_temp} gives
\begin{align*}
I_1\le \frac{4\alpha \big(\|\wb_0-\wb^*\|_{\Ib_{0:k^*}}+(s+N)\gamma\|\wb_0-\wb^*\|_{\Hb_{k^*:\infty}}^2\big)}{\gamma(1-\gamma \alpha\tr(\Hb))}\cdot\bigg(\frac{k^*}{N^2} + \gamma^2 \sum_{i> k^*}\lambda_i^2\bigg).
\end{align*}

Additionally, we have the following upper bound on $I_2$,
\begin{align*}
I_2 &= \frac{1}{\gamma N^2}\sum_{k=0}^{N-1}\big\la(\Ib-\gamma\Hb)^{2(k+s)}\big(\Ib-(\Ib-\gamma\Hb)^N\big),\Bb_0\big\ra\notag\\
&\le \frac{1}{\gamma N^2}\sum_{k=0}^{N-1}\big\la(\Ib-\gamma\Hb)^{k+2s}-(\Ib-\gamma\Hb)^{N+k+2s},\Bb_0\big\ra.
\end{align*}
Similar to the proof of Lemma \ref{lemma:bound_bias_final}, let $\vb_1,\vb_2,\dots$ be the eigenvectors of $\Hb$ corresponding to its eigenvalues $\lambda_1,\lambda_2,\dots$ and $\omega_i = \vb_i^\top(\Ib-\gamma\Hb)^s(\wb_0-\wb^*)$, we have
\begin{align}\label{eq:bound_bias_I2_tail}
I_2&\le \frac{1}{\gamma N^2}\sum_{k=0}^{N-1}\big\la(\Ib-\gamma\Hb)^{k}-(\Ib-\gamma\Hb)^{N+k},(\Ib-\gamma\Hb)^{2s}\Bb_0\big\ra\notag\\
& = \frac{1}{\gamma N^2}\sum_{k=0}^{N-1}\sum_i\big[(1-\gamma\lambda_i)^{k}-(1-\gamma\lambda_i)^{N+k}\big]\omega_i^2\notag\\
& = \frac{1}{\gamma^2 N^2}\sum_i\frac{\omega_i^2}{\lambda_i}\big[1-(1-\gamma\lambda_i)^N\big]^2\notag\\
&\le \frac{1}{\gamma^2 N^2}\sum_{i}\frac{\omega_i^2}{\lambda_i}\cdot\min\{1, \gamma^2N^2\lambda_i^2\}\notag\\
&\le \frac{1}{\gamma^2N^2}\cdot\sum_{i\le k^*}\frac{\omega_i^2}{\lambda_i} + \sum_{i>k^*}\lambda_i\omega_i^2\notag\\
& = \frac{1}{\gamma^2N^2}\cdot\big\|(\Ib-\gamma\Hb)^s(\wb_0-\wb^*)\big\|_{\Hb_{0:k^*}^{-1}}^2 + \big\|(\Ib-\gamma\Hb)^s(\wb_0-\wb^*)\big\|_{\Hb_{k^*:\infty}}^2,
\end{align}
where $k^* = \max \{k: \lambda_k \ge \frac{1}{\gamma N}\}$.
Combining \eqref{eq:bound_bias_I1_tail} and \eqref{eq:bound_bias_I2_tail} immediately completes the proof.

\end{proof}

\subsubsection{Proof of Theorem \ref{thm:generalization_error_tail}}
\begin{proof}
By Lemma \ref{lemma:tail_bias_var_decomposition_bound}, it suffices to substitute into the upper bounds on the bias and variance errors. In particular, by Young's inequality we have
\begin{align*}
\EE[L(\overline{\wb}_N)] - L(\wb^*) \le \Big(\sqrt{\text{bias}} + \sqrt{\text{variance}}\Big)^2\le 2\cdot\text{bias} + 2\cdot\text{variance}.
\end{align*}
Then we can directly substitute the bounds of $\text{variance}$ and $\text{bias}$ we proved in Lemmas \ref{lemma:tail_upperbound_var} and \ref{lemma:bound_bias_final_tail}. In particular, by Assumptions \ref{assump:bound_fourthmoment} we can directly get $R^2= \alpha\tr(\Hb)$. Therefore, it holds that 
\begin{align*}
&\EE[L(\overline{\wb}_N)] - L(\wb^*) \notag\\
&\le 2\bigg[ \frac{1}{\gamma^2N^2}\cdot\big\|(\Ib-\gamma\Hb)^s(\wb_0-\wb^*)\big\|_{\Hb_{0:k^*}^{-1}}^2 + \big\|(\Ib-\gamma\Hb)^s(\wb_0-\wb^*)\big\|_{\Hb_{k^*:\infty}}^2\notag\\
&\qquad + \frac{2\alpha \|\wb_0-\wb^*\|_2^2}{\gamma(1-\gamma \alpha\tr(\Hb))}\cdot\bigg(\frac{k^*}{N^2} + \gamma^2 \sum_{i> k^*}\lambda_i^2\bigg)  \\
&\qquad + \frac{ \sigma^2}{ 1-\gamma \alpha \tr (\Hb)} \cdot\bigg(\frac{k^*}{N} + \gamma\cdot \sum_{k^*< i\le k^\dagger}\lambda_i + \gamma^2(s+N)\cdot\sum_{i>k^\dagger}\lambda_i^2\bigg)\bigg] \notag
\\
&= 2\cdot \mathrm{EffectiveBias}+2\cdot \mathrm{EffectiveVar},
\end{align*}
where 
\begin{align*}
\mathrm{EffectiveBias} & = \frac{1}{\gamma^2N^2}\cdot\big\|(\Ib-\gamma\Hb)^s(\wb_0-\wb^*)\big\|_{\Hb_{0:k^*}^{-1}}^2 + \big\|(\Ib-\gamma\Hb)^s(\wb_0-\wb^*)\big\|_{\Hb_{k^*:\infty}}^2 \\
 \mathrm{EffectiveVar} & = \frac{ \sigma^2}{ 1-\gamma \alpha \tr (\Hb)} \cdot\bigg(\frac{k^*}{N} + \gamma\cdot \sum_{k^*< i\le k^\dagger}\lambda_i + \gamma^2(s+N)\cdot\sum_{i>k^\dagger}\lambda_i^2\bigg) \\
 &\qquad + \frac{4\alpha \big(\|\wb_0-\wb^*\|_{\Ib_{0:k^*}}+(s+N)\gamma\|\wb_0-\wb^*\|_{\Hb_{k^*:\infty}}^2\big)}{N\gamma(1-\gamma \alpha\tr(\Hb))}\cdot\bigg(\frac{k^*}{N} + \gamma^2N \sum_{i> k^*}\lambda_i^2\bigg).
\end{align*}
\end{proof}

\subsection{Lower Bounds for Tail-Averaging}
In this part we assume the noise is well-specified as in \eqref{eq:well}, and consider the SGD with tail-averaging
\begin{align*}
    \overline \wb_{s:s+N} = \frac{1}{N}\sum_{t=s}^{s+N}\wb_t.
    \end{align*}

The following lemma is a variant of Lemma \ref{lemma:lower_bound_decomp}, and lowers bound the excess risk.
\begin{lemma}[Variant of Lemma \ref{lemma:lower_bound_decomp}]\label{lemma:tail_lower_bound_decomp}
Suppose the model noise $\xi_t$ is well-specified, i.e., $\xi_t$ and $\xb_t$ are independent and $\EE [\xi_t] = 0$. 
Then 
\begin{align*}
    \EE [ L(\overline{\wb}_{s:s+N}) - L(\wb^*)] 
& \ge \frac{1}{2N^2}\cdot\sum_{t=0}^{N-1}\sum_{k=t}^{N-1}\Big\la(\Ib-\gamma\Hb)^{k-t}\Hb,\Bb_{s+t} \Big\ra \\
&\quad  + \frac{1}{2N^2}\cdot\sum_{t=0}^{N-1}\sum_{k=t}^{N-1}\Big\la(\Ib-\gamma\Hb)^{k-t}\Hb,\Cb_{s+t} \Big\ra.
\end{align*}
\end{lemma}

We then present the lower bound for the variance error.

\begin{lemma}[Variant of Lemma \ref{lemma:lowerbound_var}]\label{lemma:tail_lowerbound_var}
    Suppose Assumptions \ref{assump:second_moment} hold. Suppose the noise is well-specified (as in \eqref{eq:well}).
    Suppose $N\geq500$.
    Denote 
    \[\var = \frac{1}{2N^2}\cdot\sum_{t=0}^{N-1}\sum_{k=t}^{N-1}\Big\la(\Ib-\gamma\Hb)^{k-t}\Hb,\Cb_{s+t} \Big\ra.\]
    If the stepsize satisfies $\gamma < 1/\lambda_1$, then
    \begin{equation*}
        \var \ge \frac{ \sigma_{\textrm{noise}}^2 }{600} \rbr{\frac{k^*}{N} + \gamma\cdot\sum_{k^* < i \le k^{\dagger}}\lambda_i  + (s+N)\gamma^2 \cdot \sum_{i>k^{\dagger}}\lambda_i^2  },
    \end{equation*}
    where $k^* = \max \{k: \lambda_k \ge \frac{1}{ N \gamma}\}$ and $k^{\dagger} = \max \{k: \lambda_k \ge \frac{1}{(s+N)\gamma }\}$.
    \end{lemma}
    \begin{proof}
    We can lower bound the variance error as follows
    \begin{align*}
       \var 
        &= \frac{1}{2 N^2}\sum_{t=0}^{N-1}\sum_{k=t}^{N-1}\big\la(\Ib-\gamma\Hb)^{k-t}\Hb,\Cb_{s+t}\big\ra \\
        &= \frac{1}{2\gamma  N^2} \sum_{t=0}^{N-1} \big\la \Ib - (\Ib - \gamma\Hb)^{N-t} ,\Cb_{s+t} \big\ra \\
        &\ge \frac{ \sigma_{\textrm{noise}}^2 }{4 N^2} \sum_{t=0}^{N-1} \big\la \Ib - (\Ib - \gamma\Hb)^{N-t} , \Ib - (\Ib - \gamma \Hb)^{2(s+t)} \big \ra \qquad (\text{use Lemma \ref{lemma:lower_bound_phit}}) \\
        &= \frac{ \sigma_{\textrm{noise}}^2 }{4 N^2}\sum_{i} \sum_{t=0}^{N-1} \rbr{ 1 - (1 - \gamma\lambda_i)^{N-t} }\rbr{ 1 - (1 - \gamma \lambda_i)^{2(s+t)} }\notag\\
        &\ge \frac{ \sigma_{\textrm{noise}}^2 }{4 N^2}\sum_{i} \sum_{t=0}^{N-1} \rbr{ 1 - (1 - \gamma\lambda_i)^{N-t-1} }\rbr{ 1 - (1 - \gamma \lambda_i)^{s+t} },
    \end{align*}
    where $\{\lambda_i\}_{i\ge1}$ are the eigenvalues of $\Hb$ and are sorted in decreasing order.
    Define 
    \[ f(x):= \sum_{t=0}^{N-1} \rbr{ 1 - (1 - x)^{N-t-1} }\rbr{ 1 - (1 - x)^{s+t} }, \qquad 0< x <1,\]
    then 
    \[
    \var \ge \frac{ \sigma_{\textrm{noise}}^2 }{4 N^2} \sum_{i} f(\gamma \lambda_i).
    \]

    We have the following lower bound for $f(x)$.
    \begin{align*}
        f(x)
        &= \sum_{t=0}^{N-1} \rbr{ 1 - (1 - x)^{N-t-1} }\rbr{ 1 - (1 - x)^{s+t} }\\
        &\ge \sum_{t=\frac{N}{4}}^{\frac{3N}{4}-1} \rbr{ 1 - (1 - x)^{N-t-1} }\rbr{ 1 - (1 - x)^{s+t} }\\
        &\ge \frac{N}{2} \rbr{ 1 - (1 - x)^{\frac{N}{4}} }\rbr{ 1 - (1 - x)^{s+\frac{N}{4}} }
    \end{align*}
    We then bound $f(x)$ by the range of $x$.
    \begin{enumerate}
        \item For $x>1/N$, we have that 
        \begin{align*}
            f(x) 
            &\ge \frac{N}{2} \rbr{ 1 - (1 - x)^{\frac{N}{4}} }\rbr{ 1 - (1 - x)^{\frac{N}{4}} } \\
            &\ge \frac{N}{2} \rbr{ 1 - \rbr{1 - \frac{1}{N} }^{\frac{N}{4}} }\rbr{ 1 - \rbr{1 - \frac{1}{N} }^{\frac{N}{4}} } \\
            &\ge \frac{N}{2} \rbr{1-\frac{1}{e^{1/4}}}\rbr{1-\frac{1}{e^{1/4}}}
            \ge \frac{N}{50}.
        \end{align*}
        
        \item For $1/N > x > 1/(s+N)$, we have that 
        \begin{align*}
            f(x) 
            &\ge \frac{N}{2} \rbr{ 1 - (1 - x)^{\frac{N}{4}} }\rbr{ 1 - (1 - x)^{\frac{s+N}{4}} } \\
            &\ge \frac{N}{2}\rbr{ 1 - (1 - x)^{\frac{N}{4}} } \rbr{ 1 - \rbr{1-\frac{1}{s+N}}^{\frac{s+N}{4}} } \\
            &\ge \frac{N}{2} \rbr{ 1 - \rbr{1-\frac{N}{8} x} } \rbr{1-\frac{1}{e^{1/4}}} \ge \frac{N^2 x}{100}.
        \end{align*}
        
        \item For $x < 1/(s+N) < 1/N$, we have that  
        \begin{align*}
            f(x) 
            &\ge \frac{N}{2} \rbr{ 1 - (1 - x)^{\frac{N}{4}} }\rbr{ 1 - (1 - x)^{s+\frac{N}{4}} } \\
            &\ge \frac{N}{2} \rbr{ 1 - \rbr{1 - \frac{N}{8}x } }\rbr{ 1 - \rbr{1 - \frac{s+N/4}{2} x} } \\
            &\ge \frac{(s+N)N^2}{128}x^2.
        \end{align*}
    \end{enumerate}
    In sum, we have that 
  \[
    f(x) \ge 
    \begin{cases}
    \frac{N}{50}, &\frac{1}{N} \le x < 1,\\
    \frac{N^2}{100}x, & \frac{1}{s+N} \le x < \frac{1}{N}, \\
    \frac{(s+N)N^2}{128} x^2, &0 <  x <  \frac{1}{s+N}.
    \end{cases}
    \]
      Set $k^* = \max \{k: \lambda_k \ge \frac{1}{N\gamma }\}$ and $k^{\dagger} = \max \{k: \lambda_k \ge \frac{1}{(s+N)\gamma }\}$,
    then
    \begin{align*}
        \var 
        &\ge \frac{ \sigma_{\textrm{noise}}^2 }{4 N^2} \sum_{i}  f(\gamma\lambda_i) \\
        &\ge \frac{ \sigma_{\textrm{noise}}^2 }{4 N^2} \rbr{\frac{ N k^*}{50} + \frac{N^2}{100}\gamma\cdot \sum_{k^* < i \le k^{\dagger}}\lambda_i + \frac{ (s+N) N^2}{128}\gamma^2 \cdot \sum_{i>k^{\dagger} }\lambda_i^2  } \\
        &\ge \frac{ \sigma_{\textrm{noise}}^2 }{600} \rbr{\frac{k^*}{N} + \gamma\cdot\sum_{k^* < i \le k^{\dagger}}\lambda_i  + (s+N)\gamma^2 \cdot \sum_{i>k^{\dagger}}\lambda_i^2  }.
    \end{align*}
    This completes the proof.

    \end{proof}

Next we discuss the lower bound for the bias error. Similarly to \eqref{eq:bias-lowerbound} and using Lemma \ref{lemma:tail_lower_bound_decomp}, we have that 
\begin{align}
\bias
&\ge \frac{1}{2N^2}\sum_{t=0}^{N-1}\sum_{k=t}^{N-1}\big\la(\Ib-\gamma\Hb)^{k-t}\Hb,\Bb_{s+t}\big\ra 
= \frac{1}{2\gamma N^2}\sum_{t=0}^{N-1}\big\la\Ib - (\Ib-\gamma\Hb)^{N-t}, \Bb_{s+t}\big\ra \notag \\
&\ge \frac{1}{2\gamma N^2}\sum_{t=0}^{N/2}\big\la\Ib - (\Ib-\gamma\Hb)^{N-t}, \Bb_{s+t}\big\ra \notag \\
&\ge \frac{1}{2\gamma N^2}\big\la\Ib - (\Ib-\gamma\Hb)^{N/2}, \sum_{t=0}^{N/2}\Bb_{s+t}\big\ra. \label{eq:tail-bias-lowerbound}
\end{align}
Let $\Sbb_{s:s+n} := \sum_{t=0}^{n-1}\Bb_{s+t} = \sum_{t=0}^{n-1} (\cI-\gamma \cT)^t \circ \Bb_s$. We remain to build lower bound for $\Sbb_{s: s+N/2+1}$.
Comparing the definitions of $\Sbb_{s:s+n}$ with $\Sbb_n$, the only difference is that $\Bb_0$ is replaced by $\Bb_s$. Therefore we directly have the following lemma.

\begin{lemma}[Variant of Lemma \ref{lemma:lowerbound_St}]\label{lemma:tail_lowerbound_St}
Suppose Assumptions \ref{assump:second_moment} and \ref{assumption:lowerbound_fourthmoment} hold.
If the stepsize satisfies $\gamma < 1/\lambda_1$, then for any $n \ge 2$, it holds that
\begin{align*}
\Sbb_{s:s+n} \succeq \frac{\beta}{4}\tr\rbr{  \rbr{\Ib - (\Ib - \gamma \Hb)^{n/2}}  \Bb_s } \cdot \rbr{\Ib - (\Ib - \gamma \Hb)^{n/2}} + \sum_{t=0}^{n-1} (\Ib - \gamma \Hb)^t \cdot \Bb_s \cdot (\Ib - \gamma \Hb)^t.
\end{align*}
\end{lemma}


\begin{lemma}[Variant of Lemma \ref{lemma:lowerbound_bias_final}]\label{lemma:tail_lowerbound_bias_final}
Suppose Assumptions \ref{assump:second_moment} and \ref{assumption:lowerbound_fourthmoment} hold and $N\ge 2$. Denote
\begin{align*}
\bias= \frac{1}{2N^2}\sum_{t=0}^{N-1}\sum_{k=t}^{N-1}\big\la(\Ib-\gamma\Hb)^{k-t}\Hb,\Bb_{s+t}\big\ra,
\end{align*}
then if the stepsize satisfies $\gamma < 1/\gamma_1$, it holds that
\begin{align*}
\bias
&\ge
\frac{1}{100 \gamma^2 N^2} \rbr{ \nbr{(\Ib-\gamma \Hb)^s(\wb_0 - \wb^*)}^2_{\Hb^{-1}_{0:k^*}} + \gamma^2 N^2 \nbr{(\Ib-\gamma \Hb)^s(\wb_0 - \wb^*)}^2_{\Hb_{k^*:\infty}} } \\
&\qquad +\frac{\beta \nbr{\wb_0 - \wb^*}^2_{\Hb_{k^\dagger:\infty}} }{10000 N} \cdot \rbr{k^* + \gamma^2 N^2 \sum_{i > k^*}\lambda_i^2 },
\end{align*}
where $k^* = \max \{k: \lambda_k \ge \frac{1}{ N \gamma}\}$ and $k^{\dagger} = \max \{k: \lambda_k \ge \frac{1}{(s+N)\gamma }\}$.
\end{lemma}
\begin{proof}
According to \eqref{eq:tail-bias-lowerbound} and Lemma \ref{lemma:tail_lowerbound_St}, we have that
\begin{align*}
\bias 
&\ge  \frac{1}{2\gamma N^2}\big\la\Ib - (\Ib-\gamma\Hb)^{N/2}, \Sbb_{s:s+N/2+1}\big\ra 
\ge  \frac{1}{2\gamma N^2} \big\la\Ib - (\Ib-\gamma\Hb)^{N/2}, \Sbb_{s:s+N/2}\big\ra  \\
&\ge \underbrace{\frac{\beta}{8\gamma N^2}\tr\rbr{  \rbr{\Ib - (\Ib - \gamma \Hb)^{N/4}}  \Bb_s } \cdot  \big\la\Ib - (\Ib-\gamma\Hb)^{N/2}, \Ib - (\Ib - \gamma \Hb)^{N/4} \ra }_{I_1} \\
&\quad + \underbrace{\frac{1}{2\gamma N^2} \big\la\Ib - (\Ib-\gamma\Hb)^{N/2}, \sum_{t=0}^{N/2-1} (\Ib - \gamma \Hb)^t \cdot \Bb_s \cdot (\Ib - \gamma \Hb)^t \big\ra }_{I_2}.
\end{align*}
Also noticing a lower bound for $\Bb_s$:
\begin{equation*}
    \Bb_s = (\cI - \gamma \cT)^s \circ \Bb_0 \ge (\cI - \gamma \tilde{\cT})^s \circ \Bb_0 = (\Ib - \gamma \Hb)^s \cdot \Bb_0 \cdot (\Ib - \gamma \Hb)^s.
\end{equation*}
Then the first term is lower bounded by 
\begin{align*}
    I_1
    &\ge {\frac{\beta}{8\gamma N^2}\tr\rbr{  \rbr{\Ib - (\Ib - \gamma \Hb)^{N/4}} (\Ib - \gamma \Hb)^{2s}  \Bb_0 } \cdot  \tr\rbr{ \rbr{ \Ib - (\Ib - \gamma \Hb)^{N/4} }^2 } } \\
    &= {\frac{\beta}{8\gamma N^2} \rbr{ \sum_{i}\rbr{1- (1-\gamma \lambda_i)^{N/4}} (1 - \gamma \lambda_i)^{2s}\omega_i^2 } \cdot  \rbr{ \sum_{i}\rbr{1 - (1-\gamma \lambda_i)^{N/4}  }^2 } },
\end{align*}
where $\omega_i = \vb_i^\top(\wb_0-\wb^*)$ for $\vb_1,\dots,\vb_d$ being the eigenvectors of $\Hb$;
and the second term is lower bounded by
\begin{align*}
    I_2
    & = {\frac{1}{2\gamma N^2} \la \sum_{t=0}^{N/2-1} (\Ib - \gamma \Hb)^{2t} \rbr{ \Ib - (\Ib-\gamma\Hb)^{N/2} }, \Bb_s \ra }\\
    & \ge {\frac{1}{2\gamma N^2} \la \sum_{t=0}^{N/2-1} (\Ib - 2\gamma \Hb)^{t} \rbr{ \Ib - (\Ib-\gamma\Hb)^{N/2} }, \Bb_s \ra }\\
    &\ge {\frac{1}{4 \gamma^2 N^2} \la \rbr{ \Ib - (\Ib-\gamma\Hb)^{N/2} }^2 \Hb^{-1}, \Bb_s \ra } \\
    &\ge {\frac{1}{4 \gamma^2 N^2} \la \rbr{ \Ib - (\Ib-\gamma\Hb)^{N/4} }^2 \Hb^{-1}, (\Ib - \gamma \Hb)^s \Bb_0(\Ib - \gamma \Hb)^s  \ra } \\
    &=  \frac{1}{4 \gamma^2 N^2} \sum_{i} \rbr{1 - (1-\gamma \lambda_i)^{N/4}  }^2\lambda_i^{-1} \rbr{(1-\gamma \lambda_i)^s \omega_i}^2.
\end{align*}

To further lower bound the two terms, noticing the following inequalities:
\begin{align*}
    1-(1-\gamma \lambda_i)^{\frac{N}{4}} \ge 
    \begin{cases}
    1-(1-\frac{1}{N})^{\frac{N}{4}} \ge 1-e^{-\frac{1}{4}} \ge \frac{1}{5}, & \lambda_i \ge \frac{1}{\gamma N}, \\
    \frac{N}{4}\cdot \gamma \lambda_i -  \frac{N(N-4)}{32}\cdot\gamma^2\lambda_i^2 \ge \frac{N}{5} \cdot \gamma \lambda_i, & \lambda_i < \frac{1}{\gamma N},
    \end{cases}
\end{align*}
and
\begin{align*}
    (1-\gamma \lambda_i)^{2s} \ge
    \begin{cases}
    0, & \lambda_i \ge \frac{1}{\gamma s}, \\
    (1-\frac{1}{s})^{2s} \ge e^{-2} \ge \frac{1}{10}, & \lambda_i < \frac{1}{\gamma s}.
    \end{cases}
\end{align*}
Plugging these into the bounds for $I_1$ and $I_2$, and setting $k^* := \max \{ k : \lambda_k \ge 1/(\gamma N)\}$ and $k^\dagger := \max \{ k : \lambda_k \ge 1/(\gamma(s+ N))\}$, we then obtain that
\begin{align*}
I_1 
&\ge \frac{\beta}{8 \gamma N^2} \cdot \rbr{  \frac{\gamma N}{50}\sum_{i > k^\dagger}\lambda_i \omega_i^2 } \cdot \rbr{ \frac{1}{25} \cdot k^* + \frac{\gamma^2 N^2}{25}\cdot \sum_{i > k^*}\lambda_i^2 }  \\
& = \frac{\beta \nbr{\wb_0 - \wb^*}^2_{\Hb_{k^\dagger:\infty}} }{10000 N} \cdot \rbr{k^* + \gamma^2 N^2 \sum_{i > k^*}\lambda_i^2 },
\end{align*}
and that 
\begin{align*}
I_2 
&\ge \frac{1}{4\gamma^2 N^2} \rbr{ \frac{1}{25} \cdot \sum_{i \le k^*} \lambda_i^{-1} \rbr{(1-\gamma \lambda_i)^s \omega_i}^2 + \frac{\gamma^2 N^2}{25} \cdot \sum_{i > K^*} \lambda_i \rbr{(1-\gamma \lambda_i)^s \omega_i}^2 } \\
&= \frac{1}{100 \gamma^2 N^2} \rbr{ \nbr{(\Ib-\gamma \Hb)^s(\wb_0 - \wb^*)}^2_{\Hb^{-1}_{0:k^*}} + \gamma^2 N^2 \nbr{(\Ib-\gamma \Hb)^s(\wb_0 - \wb^*)}^2_{\Hb_{k^*:\infty}} }.
\end{align*}
Summing up the two terms completes the proof.

\end{proof}

\subsubsection{Proof of Theorem \ref{thm:generalization_error_tail_lowerbound}}
\begin{proof}
Plugging the bounds of the bias error and variance error in Lemmas \ref{lemma:tail_lowerbound_bias_final} and \ref{lemma:tail_lowerbound_var} into Lemma \ref{lemma:tail_lower_bound_decomp} immediately completes the proof.
\end{proof}

\end{document}